\documentclass[twoside,11pt]{article}

\usepackage{jmlr2e}
\usepackage{graphicx,subfigure}
\usepackage{epsfig} 
\usepackage{mathptmx} 
\usepackage{times} 
\usepackage{amsmath} 
\usepackage{amssymb}  
\usepackage{dsfont}
\usepackage{verbatim}
\newcommand{\B}[1]{{\bf #1}}
\renewcommand{\P}{{\mathbb P}}
\newcommand{\E}{{\mathbb E}}
\newcommand{\R}{{\mathbb R}}

\newcommand{\cH}{{\mathcal{H}}}

\newcommand{\Diag}{\mbox{diag}}
\newcommand{\norm}[1]{\Vert #1\Vert}

\newcommand{\Var}{\mbox{Var}}
\newcommand{\Trace}{\mbox{tr}}

\newcommand{\cU}{\mathcal{U}}
\newcommand{\cW}{\mathcal{W}}

\newcommand{\half}{\frac{1}{2}}
\newcommand{\one}{\mathbf{1}}
\newcommand{\Lam}{\Lambda}
\newcommand{\lam}{\lambda}
\newcommand{\sig}{\sigma}
\newcommand{\Sig}{\Sigma}
\newcommand{\sigt}{\sig^2}
\newcommand{\gam}{\gamma}
\newcommand{\btheta}{{\bar\theta}}
\newcommand{\vE}{{\E_v}}
\newcommand{\vVar}{{\Var_v}}

\ShortHeadings{Hyperparameter Group Lasso}{Aravkin and Burke and Chiuso and Pillonetto}
\firstpageno{1}

\begin{document}

\title{Convex vs nonconvex approaches for sparse estimation:\\ GLasso, Multiple Kernel Learning
and Hyperparameter GLasso}

\author{
        \name Aleksandr Aravkin \email saravkin@us.ibm.com\\
        \addr IBM T.J. Watson Research Center\\
       Yorktown Heights, NY, 10598
        \AND
        \name James V. Burke \email burke@math.washington.edu\\
        \addr Department of Mathematics\\
        University of Washington\\
        Seattle, WA
        \AND
        \name Alessandro Chiuso \email chiuso@dei.unipd.it\\
        \addr Department of Information Engineering\\
        University of Padova\\
        Padova, Italy
        \AND
        \name Gianluigi Pillonetto \email giapi@dei.unipd.it\\
        \addr Department of Information Engineering\\
        University of Padova\\
        Padova, Italy}
%
\editor{}

\maketitle

\begin{abstract}
The popular Lasso approach for sparse estimation can be derived via marginalization of a joint density associated with a particular stochastic model. A different marginalization of the same probabilistic model leads to a different non-convex estimator where hyperparameters are optimized. Extending these arguments to problems where groups of variables have to be estimated, we study a computational scheme for sparse estimation that differs from the Group Lasso. Although the underlying optimization problem defining this estimator is non-convex, an initialization strategy based on a univariate Bayesian forward selection scheme is presented. This also allows us to define an effective non-convex estimator where only one scalar variable is involved in the optimization process. Theoretical arguments, independent of the correctness of the priors entering the sparse model, are included to clarify the advantages of this non-convex technique in comparison with  other  convex estimators. Numerical experiments are also used to compare the performance of these approaches.
\end{abstract}



\section{Introduction}

We consider sparse estimation in a linear regression model where the explanatory
factors $\theta\in \R^m$ are naturally grouped so that $\theta$ is partitioned as
$\theta = [{\theta^{(1)}}^\top \quad {\theta^{(2)}}^\top \quad \ldots \quad {\theta^{(p)}}^\top]^\top$.
In this setting we assume that $\theta$ is group (or block) sparse in the sense that many
of the constituent vectors $\theta^{(i)}$ are zero or have a negligible influence on the
output $y\in\R^n$. In addition, we assume that the number of \emph{unknowns} $m$ is large,
possibly larger than the size of the available data $n$.
Interest in general sparsity estimation and optimization
has attracted the interest of many researchers in statistics, machine learning, and signal processing
with numerous applications in feature selection, compressed sensing, and selective shrinkage
\citep{Hastie90,Lasso1996,Donoho2006,CandesTao2007}. The motivation for our study of
the group sparsity problem comes from the ``dynamic Bayesian network'' scenario identification
problem as discussed in \citep{ChiusoPAuto2011,ChiusoPNIPS2010,ChiusoPCDC2010}.
In a dynamic network scenario
the ``explanatory variables'' are often the past histories of different input signals with the
``groups'' $\theta^{(i)}$ representing the impulse responses\footnote{An thus may, in principle,
be infinite dimensional.} describing the relationship between the $i$-th input and the output
$y$.
This application informs our view of the group sparsity problem as well as our measures of success for
a particular estimation procedure.

Several approaches have been put forward in the literature for joint estimation and
variable selection problems. We cite the well known Lasso \citep{Lasso1996}, Least Angle
Regression (LAR) \citep{LARS2004}, their ``group'' versions Group Lasso (GLasso) and Group Least
Angle Regression (GLAR) \citep{Yuan_JRSS_B_2006}, Multiple Kernel Learning (MKL)
\citep{Bach_MKL_2004,Evgeniou05,PAMI10}. 
Methods based on hierarchical Bayesian models have also
been considered such as Automatic Relevance Determination (ARD) \citep{McKayARD}, the Relevance
Vector Machine (RVM) \citep{Tipping2001}, and the exponential hyperprior in
\citep{ChiusoPNIPS2010,ChiusoPAuto2011}. The Bayesian approach considered in
\citep{ChiusoPNIPS2010,ChiusoPAuto2011} and further developed in this paper is intimately related
to \citep{McKayARD,Tipping2001}; in fact, the exponential hyperprior algorithm in
\citep{ChiusoPNIPS2010,ChiusoPAuto2011} is a penalized version of ARD.
A variational approach based on the golden standard spike and slab prior,
also called  two-groups prior \citep{EfronTwogroup}, has been also recently proposed in 
\citep{Titsias}.

An interesting series of
papers \citep{Wipf_IEEE_TSP_2007,Wipf_ARD_NIPS_2007,Wipf_IEEE_TIT_2011} provide a nice link
between penalized regression problems like Lasso, also called type-I methods, and Bayesian methods
(like RVM \citep{Tipping2001} and ARD \citep{McKayARD}) with hierarchical hyperpriors where the
``hyperparameters'' are estimated via maximizing the marginal likelihood and then inserted in the
Bayesian model following the Empirical Bayes paradigm \citep{Maritz:1989}; these latter methods
are also known as type-II methods \citep{BergerBook}. Note that this Empirical Bayes paradigm has
also been recently used in the context of System Identification
\citep{DNPAut2010,PillonettoCD_Auto2011,LjungIFAC2011}.

In \citep{Wipf_ARD_NIPS_2007,Wipf_IEEE_TIT_2011} it is argued that type-II methods have advantages
over type-I methods; some of these advantages are related to the fact that, under suitable
assumptions, the former can be written in the form of type-I with the addition of
a non-separable penalty term
(a function $g(x_1,..,x_n)$ is non-separable if it cannot be written as
$g(x_1,..,x_n)=\sum_{i=1}^{n} h(x_i)$). The analysis in \citep{Wipf_IEEE_TIT_2011} also suggests
that in the low noise regime the type-II approach results in a ``tighter'' approximation to the
$\ell_0$ norm. This is supported by experimental evidence showing that these Bayesian approaches
perform well in practice.  Our experience is that the approach based on the marginal likelihood is
particularly robust w.r.t. noise regardless of the ``correctness'' of the Bayesian prior.

Motivated by the nice performance of the exponential hyperprior approach introduced in the
dynamic network identification scenario \citep{ChiusoPNIPS2010,ChiusoPAuto2011}, we provide some
new insights clarifying the above issues. The main contributions are as follows:
\begin{enumerate}
\item[(i)] in the first part of the paper we discuss the relation among Lasso (and GLasso), the
Exponential Hyperprior (HGLasso algorithm hereafter, for reasons which will become clear later on)
and MKL by putting all these methods in a common Bayesian framework (similar to that discussed in
\citep{BayesianLasso}). Lasso/GLasso and MKL boil down to convex optimization problems, 
leading to identical estimators, while
HGLasso does not.
\item [(ii)] All these methods are then compared in terms of optimality (KKT)
conditions and tradeoffs between sparsity and shrinkage are studied illustrating the advantages of
HGLasso over GLasso (or, equivalently, MKL). Also the properties of Empirical Bayes estimators which form the basis of our computational scheme are studied in terms of their Mean Square Error properties;
this is first established in the simplest case of orthogonal regressors and then extended to more
general cases allowing for the regressors to be  realizations from, possibly correlated, stochastic
processes. This,  of course, is of paramount importance for the system identification scenario
studied in \citep{ChiusoPNIPS2010,ChiusoPAuto2011}. 

Our analysis avoids assumptions on the
correctness of the priors entering the stochastic model and clarifies why HGLasso is likely to
provide more sparse and accurate estimates in comparison with the other two convex estimators.
As a byproduct, our study also clarifies the asymptotic properties of ARD.

\item[(iii)] Since HGLasso requires solving non-convex, and possibly high-dimensional, optimization
problems we introduce a version of our computational scheme which can be used as an
initialization for the full non-convex search requiring optimization with respect to only
one scalar variable representing a common scale factor for the hyperparameters.
Such Bayesian schemes with a hyperprior having a common scale factor, or more
generally group problems in which each group is described by one hyperparameter, can be seen as
instances of ``Stein estimators'' \citep{Stein1961,EfronMorris1973,SteinAS1981} and have close
connections to the non-negative garrote estimator \citep{Breiman1995}. The initialization we
propose is based on a selection scheme which departs from classical Bayesian variable selection
algorithms \citep{GeorgeJASA1993,GeorgreFoster_Biometrika_2000,ScottBerger_AS2010}. These
latter methods are based on the introduction of binary (Bernoulli) latent variables. Instead our
strategy involves a ``forward selection'' type of procedure which may be seen as an instance of
the ``screening'' type of approach for variable selection discussed in
\citep{ForwardSelectionJASA2009}; note however that while classical forward selection procedures
work in ``parameter space'' our forward selection is performed in hyperparameter space through the
marginal posterior (i.e. once the parameters $\theta$ are integrated out); in the asymptotic regime this procedure is equivalent to performing forward selection using BIC as a criterion. This ``finite data'' Bayesian flavor seems to be a key
feature which makes the procedure remarkably robust as the experimental results confirm. Note that backward and forward-backward versions of this procedure have also been tested with no notable differences. 
\item[(iv)] Extensive numerical experiments involving artificial and real data
are included which confirm the superiority of
HGLasso.\\
\end{enumerate}

The paper is  organized as follows. In Section \ref{LassoandALTERNATIVE} we
introduce the Lasso approach in a Bayesian framework, as well as
another estimator, namely HLasso, that requires the optimization of
hyperparameters. Section \ref{GLassoHGLasso} extends the arguments
to a group version of the sparse estimation problem introducing
GLasso and the group version of HLasso which we call HGLasso. In
Section \ref{MKLandHGLasso} the relationship between HGLasso and MKL is discussed, reviewing 
the equivalence between GLasso and MKL. 
Section \ref{SparsvsSh} clarifies the advantages
of HGLasso over GLasso and MKL on a simple example. In Section \ref{sec:MSE} the Mean Squared
Error properties of the Empirical Bayes estimators are studied, including their asymptotic
behavior. In Section \ref{Impl} we discuss the implementation of our computational scheme, also
deriving a version of HGLasso that requires the optimization of an objective only with respect to
one scalar variable. Section \ref{sec:sim} reports numerical experiments involving artificial and
real data, also comparing the new approach with the \emph{adaptive Lasso} 
described in \citep{AdaptiveLasso}.  Some conclusions end the paper.

\section{Lasso and HLasso}\label{LassoandALTERNATIVE}

\begin{figure}
\begin{center}
\includegraphics[width=0.7\linewidth,angle=0]{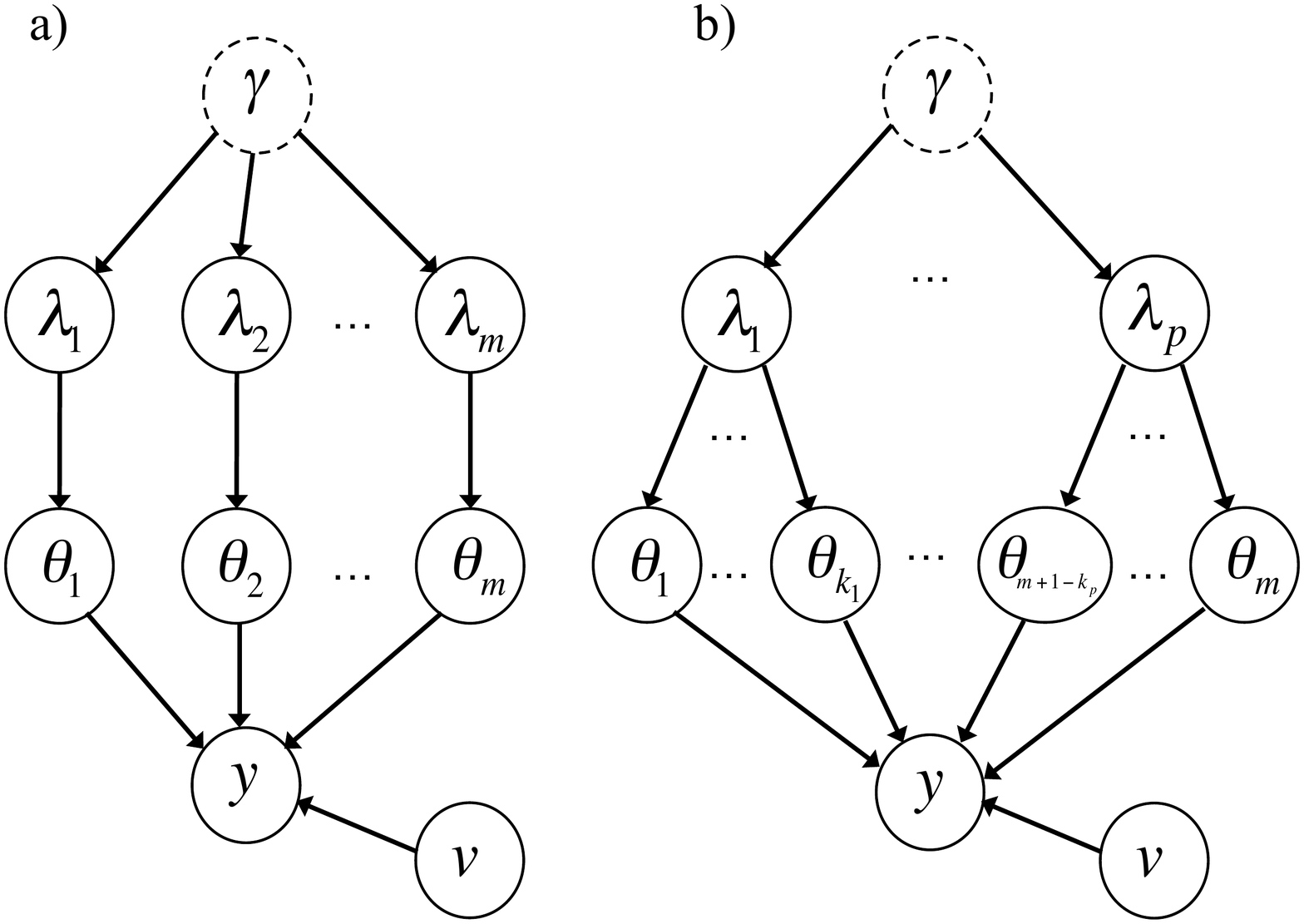}
\caption{Bayesian networks describing the stochastic model for
sparse estimation (a) and group sparse estimation (b)}
\label{BN}
\end{center}
\end{figure}

Let $\theta=[\theta_1 \quad \theta_2 \ldots \theta_m]^\top$
be an unknown parameter vector while
$y \in \R^n$ denotes the vector containing some noisy data. In particular,
our measurements model is
\begin{equation}\label{Measmod}
y=G \theta +v
\end{equation}
where $G \in \R^{n \times m}$ and $v$ is the vector whose components are
white noise of known variance $\sigma^2$.

\subsection{The Lasso}

Under the assumption that $\theta$ is sparse, i.e. many of its components are equal to zero or
have a negligible influence on $y$, a popular approach to reconstruct the parameter vector is the
Lasso \citep{Lasso1996}.
The Lasso estimate of $\theta$ is given by
\begin{equation}\label{Lassoobj} \hat{\theta}_L = \arg \min_{\theta}
\frac{(y-G\theta)^\top(y-G\theta)}{2 \sigma^2} + \gamma_L \sum_{i=1}^{m} |\theta_i|
\end{equation}
where $\gamma_L \in \R_+$ is the regularization parameter.
A key feature of the Lasso is that the estimate $\hat{\theta}_L$ is the solution to
a convex optimization problem.

As in \citep{BayesianLasso}, we describe a derivation of the Lasso through the
marginalization of a suitable probability density function.
This hierarchical representation is useful for establishing a connection
with the variety of estimators considered in this paper.
The Bayesian model we consider
is depicted in Fig. \ref{BN}(a). Nodes and arrows are either dotted or solid depending on being
representative of, respectively, deterministic or stochastic quantities/relationships. Here,
$\lambda$ denotes a vector whose components $\{\lambda_i\}_{i=1}^{m}$ are independent and identically distributed exponential
random variables with probability density
\begin{equation}\label{p_lambda}
p_{\gamma} (\lambda_i) = \gamma e^{-\gamma \lambda_i} \chi(\lambda_i) ,
\end{equation}
where
$\gamma$ is a positive scalar while $\chi(t)=1$ if $t\geq 0$, 0 otherwise. In addition
\begin{equation}\label{p_theta}
\theta_i | \lambda_i \sim {\mathcal N}(0,\lambda_i) \qquad\mbox{and}\qquad v \sim
{\mathcal N}(0, \sigma^2 I_n) ,
\end{equation}
where ${\mathcal N}(\mu,\Sigma)$ is the Gaussian
density of mean $\mu$ and covariance $\Sigma$ while $I_n$ is the $n \times n$ identity matrix.
We have the following result from Section 2 in \citep{BayesianLasso}.

\begin{theorem} \label{Lassomarg}
Given the Bayesian network in Fig. \ref{BN}(a), let
\begin{equation}\label{AlternativeTeta}
\hat{\theta} = \arg \max_{\theta \in \R^{m}}
\int_{\R_{+}^{m}} p(\theta,\lambda | y) d\lambda,
\end{equation}
Then $\hat{\theta}=\hat{\theta}_L $ provided that
$\gamma_L=\sqrt{ 2 \gamma}$.
\end{theorem}

\subsection{The HLasso}\label{HLasso}

Theorem \ref{Lassomarg} inspires the definition of an estimator
obtained by marginalizing with respect to $\theta$ instead of $\lambda$ and then
maximizing the resulting marginal density $p(\lambda | y)$ with respect to $\lambda$
obtaining an estimate $\hat\lambda$ for $\lambda$.
Having $\hat\lambda$, we use an empirical Bayes approach and set
$\hat\theta_{HL}:=\E[\theta | y, \hat{\lambda}]$ (the minimum variance estimate of
$\theta$ given $y$ and $\lambda=\hat\lambda$).
We call $\hat\theta_{HL}$ the Hyperparameter Lasso (HLasso).
This estimator is given in the next theorem
which uses the fact that $\theta$ conditional on $\lambda$ is Gaussian,
so that the marginal density of $\lambda$ is available
in closed form. The proof flows from the following observations:
\begin{eqnarray}\label{joint pdf 2}
p(\theta,\lambda | y)&\propto& |\Lam|^{-1/2}\exp[-\half y^\top\Sig_y(\lam)^{-1}y]
\\ && \bullet
\exp[-\half (\theta -\theta_{HL}(\lam))^\top(\Lam^{-1}+\sig^{-2}G^\top G)(\theta
-\theta_{HL}(\lam))]\exp[-\gam\one^\top\lam],\nonumber
\end{eqnarray}
where
\begin{equation}\label{Sigma y}
\Lam=\Diag(\lam),\qquad
\Sig_y(\lam):=(\sig^2I+G\Lam G^\top),
\end{equation}
and
\begin{equation}\label{HLtheta}
\theta_{HL}(\lam):=\E[\theta | y, {\lambda}]
=(\sigt\Lam^{-1}+G^\top G)^{-1}G^\top y=\Lam G^\top\Sig_y(\lam)^{-1}y
\end{equation}
Notice that the equivalence of the two expressions for the $\theta_{HL}(\lam)$ follows from
the matrix inversion formula
\begin{equation}\label{mif}
\Sig_y(\lam)^{-1}=\sig^{-2}\left[I-G(\sigt\Lam^{-1}+G^TG)^{-1}G^\top\right].
\end{equation}
Also note that 
(\ref{HLtheta}) assures us that $\theta_{HL}(\lam)$
well defined even when some of the components of $\lam$ are zero.
The matrix  \(\Sig_y(\lam)\) plays a fundamental role in
much of the analysis of this paper. The Law of Iterated Expectation
tells us that \(\Sig_y(\lam)\) is simply the second moment of $y$ given $\lam$, indeed,
\begin{eqnarray}
\E[yy^\top\, |\, \lam]&=&\E[\E[yy^\top\, |\, \theta]\, |\, \lam]\nonumber\\
&=&\E[\Var[y\, |\, \theta]+\E[y\, |\, \theta]\E[y\, |\, \theta]^\top \, |\, \lam]\nonumber\\
&=&\E[\sig^2I+G\theta\theta^\top G^\top\, |\, \lam]\nonumber\\
&=&\sig^2I+G\E[\theta\theta^\top\, |\, \lam]G^\top\nonumber\\
&=&\Sig_y(\lam).\label{2nd moment of y given lam}
\end{eqnarray}

\begin{theorem} \label{LassotheALTERNATIVE}
Given the Bayesian network in Fig. \ref{BN}(a), let
\begin{equation}\label{Alternativelambda}
\hat{\lambda} = \arg \max_{\lambda \in \R_+^{m}}  \int_{\R^{m}}
p(\theta,\lambda|y) d\theta .
\end{equation}
Then
\begin{equation}\label{Alternativelambda2}
\hat{\lambda} = \arg \min_{\lambda \in \R_+^{m}}  \frac{1}{2} \log
\det(\Sigma_y(\lam)) + \frac{1}{2} y^\top (\Sigma_y(\lam))^{-1} y + \gamma
\sum_{i=1}^{m}  \lambda_i,
\end{equation}
and, given $\lambda=\hat{\lambda}$, the HLasso estimate of $\theta$ is given by
\begin{equation}\label{Alternativeteta}
\hat{\theta}_{HL} := \E[\theta | y, \hat{\lambda}]  .
\end{equation}
\end{theorem}
\begin{flushright}
$\blacksquare$
\end{flushright}

The objective in (\ref{Alternativelambda}) 
depends on $m$ variables as in the Lasso case,
however the optimization problem is no-longer convex since the function
$\log\det(\Sigma_y(\lam))$ is a concave function of $\lam$ as it is the composition
of the concave function $\log\det(\Sigma)$ and the affine function $\Sigma_y(\lam)$.

It is worth observing that the estimator obtained from \eqref{Alternativelambda} and
\eqref{Alternativeteta} is a form of ``Sparse Bayesian Learning'' having close resemblance with
ARD \citep{McKayARD} and RVM \citep{Tipping2001}, see also \citep{Wipf_ARD_NIPS_2007}. In fact ARD
is obtained by setting $\gamma$ in \eqref{Alternativelambda2} to zero while the Gamma prior used
in RVM (see eq. (6) in \citep{Tipping2001}) seems to play a symmetric role, favoring large values
of $\lambda_i$'s. Note that the Gamma prior in equation (6) of \citep{Tipping2001} becomes flat as
$a\rightarrow $ and $b\rightarrow 0$, similarly to \eqref{p_lambda} as $\gamma \rightarrow 0$. A
similar discussion applies also to the Group version of this estimator to be introduced in Section
\ref{HGLasso}.

In Sections \ref{SparsvsSh} and \ref{sec:MSE} we show that the parameter $\gamma$ plays a
fundamental role in enforcing sparsity. In addition,
we establish an interesting interpretation in terms of the Mean Squared Error properties of the resulting estimators as $\gamma\rightarrow 0$. Note also that $\gamma$ plays a fundamental role in model selection consistency
(see Remark \ref{rem:lambda_zero}). 

%

\section{GLasso and HGLasso}\label{GLassoHGLasso}

We now consider a situation where the explanatory factors $G$
used to predict $y$ are grouped. Think of
$\theta$ as being partitioned into $p$ sub-vectors $\theta^{(i)}$, $i=1,\dots,p$, so that
\begin{eqnarray}\label{Factorteta}
\theta = [{\theta^{(1)}}^\top \quad {\theta^{(2)}}^\top \quad \ldots \quad
{\theta^{(p)}}^\top]^\top.
\end{eqnarray}
For $i=1,\dots,p$, assume that the sub-vector $\theta^{(i)}$ has dimension $k_i$
so that $m = \sum_{i=1}^p k_i$.
Next, conformally partition the matrix $G=[G^{(1)}, \dots, G^{(p)}]$
to obtain the measurement model
\begin{equation}\label{GroupMeasmod}
y=G\theta+v=\sum_{i=1}^{p}  G^{(i)} \theta^{(i)} +v.
\end{equation}
In what follows, we assume that $\theta$ is \emph{block sparse} in
the sense that many of the blocks $\theta^{(i)}$
are null, i.e. with all of their components equal to zero,
or have a negligible effect on $y$.

\subsection{The GLasso}

A leading approach for
the block sparsity problem
is the Group Lasso (GLasso) \citep{Yuan_JRSS_B_2006}.
The Group Lasso
determines the estimate of $\theta$ as
\begin{equation}\label{GLasso}
\hat{\theta}_{GL} = \arg \min_{\theta \in \R^{m}}
\frac{(y-G\theta)^\top(y-G\theta)}{2\sigma^2} + \gamma_{GL}
\sum_{i=1}^{p} \|\theta^{(i)}  \|\, ,
\end{equation}
where $\| \cdot \|$ denotes the classical Euclidean norm.
Notice that the representation (\ref{GLasso}) assumes that the $\theta^{(i)}$ are i.i.d. with
\[
p(\theta^{(i)}\,\, |\gamma_{GL})\propto \mathrm{exp}\left[-\gamma_{GL}\norm{\theta^{(i)}}\right].
\]
It is easy to see that, as in the Lasso case, the objective is convex.



\subsection{The HGLasso}\label{HGLasso}

An alternative approach to the block sparsity problem is
discussed in \citep{ChiusoPNIPS2010}. This approach
relies on the group version
of the model in Fig. 1(a) illustrated
in Fig. 1(b).
In the network, $\lambda$ is now a $p$-dimensional vector
with $i-th$ component given by $\lambda_i\in \R_+$.
In addition, conditional on $\lambda$,
each block $\theta^{(i)}$ of the vector $\theta$
is zero-mean Gaussian with covariance $\lambda_i I_{k_i}$, $i=1,..,p$, i.e.
\begin{equation}\label{p_lambdaA}
\theta^{(i)} | \lambda_i \sim N(0,{\lambda}_i I_{k_i}) .
\end{equation}
As for the HLasso, the proposed estimator
first optimizes the marginal density of $\lambda$, and then
again using an empirical Bayes approach,
the minimum variance estimate of $\theta$
is computed with $\lambda$ taken as known and set to its estimate.
We call this scheme Hyperparameter Group Lasso
(HGLasso). It is described in the following
theorem.

\begin{theorem}\label{GLassotheALTERNATIVE}
Consider the Bayesian network  in Fig. 1 (b) with measurement model given by
(\ref{GroupMeasmod}), (\ref{p_lambdaA}), and (\ref{p_lambda}), and  define
\begin{equation}\label{AlternativeGlambda_int}
\hat{\lambda} = \arg \max_{\lambda \in \R_+^{p}}  \int_{\R^{m}}
p(\theta,\lambda | y) d\theta .
\end{equation}
Then, $\hat{\lambda} $ is given by
\begin{equation}\label{AlternativeGlambda}
\arg \min_{{\lambda} \in \R_+^{p}}  \frac{1}{2} \log
\det(\Sigma_y(\lam)) + \frac{1}{2} y^\top \Sigma_y^{-1}(\lam) y + \gamma
\sum_{i=1}^{p} {\lambda}_i ,
\end{equation}
where
\begin{equation}\label{H Lam}
\Sigma_y(\lam) := G \Lambda G^\top+\sigma^2 I,  \qquad  \Lambda :=\mbox{blockdiag}(\{ {\lambda}_i I_{k_i} \}).
\end{equation}
In addition, the HGLasso estimate of $\theta$, denoted $\hat\theta_{HGL}$, is given by
setting $\lambda = \hat{\lambda}$ in the function
\begin{equation}\label{AlternativeGteta}
{\theta}_{HGL}(\lam) := \E[\theta | y, \lambda]  =
\Lambda G^\top  (\Sigma_y({\lambda}))^{-1} y.
\end{equation}
\end{theorem}
\begin{flushright}
$\blacksquare$
\end{flushright}

The derivation of this estimate is virtually identical to the derivation of the estimate given
in Theorem \ref{LassotheALTERNATIVE}. 
For this reason, we slightly abuse our notation
by not introducing a new notation for the key affine matrix mapping $\Sig_y(\lam)$.
Just as in the HLasso case, the objective in  (\ref{AlternativeGlambda})
is not convex in $\lam$.
However, now the optimization is performed
in the lower dimensional space $\R^p$, rather than in $\R^m$ where the GLasso
objective is optimized.

Let the vector $\mu$ denote the dual vector for
the constraint $\lambda \geq 0$. Then the Lagrangian
for the problem (\ref{AlternativeGlambda}) is given by
\begin{equation}
\label{LagrangianHGL}
\begin{array}{lll}
L(\lambda, \mu) :=\frac{1}{2} \log\det(\Sigma_y(\lambda))
+
\frac{1}{2} y^\top \Sigma_y(\lambda)^{-1} y
+
\gamma \B{1}^\top\lambda
- \mu^\top\lambda .
\end{array}
\end{equation}
Using the fact that
\begin{eqnarray*}
\partial_{\lambda_i}L(\lambda,\mu)&=&
\frac{1}{2} \Trace\left(G^{(i)\top}\Sigma_y(\lambda)^{-1}G^{(i)}\right) \\
&-& \frac{1}{2} y^\top\Sigma_y(\lambda)^{-1}G^{(i)}G^{(i)\top}\Sigma_y(\lambda)^{-1}y+\gamma\ -\mu_i,
\end{eqnarray*}
we obtain
the following KKT conditions
for (\ref{AlternativeGlambda}).

\begin{proposition}\label{KKTofHGL}
The necessary conditions for $\lambda$
to be a solution of (\ref{AlternativeGlambda}) are
\begin{equation}\label{KKTHGL}
\begin{array}{l}
 \Sigma_y = \sigma^2 I + \sum_{i = 1}^p \lambda_i G^{(i)}G^{(i)\top} \\
 W\Sigma_y = I \\
 \Trace\left(G^{(i)\top}WG^{(i)}\right)
-
\|G^{(i)\top}Wy\|_2^2
+
2\gamma -2\mu_i = 0, \quad i = 1, \dots, p \\
 \mu_i\lambda_i = 0, \quad i = 1, \dots, p \\
 0 \le \mu,\ \lambda \mbox{ and } 0 \preceq W, \Sigma_y .
\end{array}
\end{equation}
\end{proposition}

It is interesting to observe that, by (\ref{2nd moment of y given lam}), one has
\[
\E\left[\theta_{HGL}(\lambda)
\theta_{HGL}(\lambda)^\top\, |\, \lam\right]=
\Lambda G^\top  \Sigma_y({\lambda})^{-1}\E[ yy^\top\, |\, \lam]
\Sigma_y({\lambda})^{-1}G\Lam=\Lambda G^\top  \Sigma_y({\lambda})^{-1}G\Lam,
\]
and so
\begin{equation}\label{VarTheta}
\E\left[\left.\theta_{HGL}^{(i)}(\lambda)\left(
\theta_{HGL}^{(i)}(\lambda)\right)^\top\, \right|\, \lam\right]=
\lambda_i^2\left(G^{(i)\top}WG^{(i)}\right),\quad i=1,\dots,p.
\end{equation}
In addition,
$$
\|\theta_{HGL}^{(i)}(\lambda)\|^2=\lambda_i^2\|G^{(i)\top}Wy\|_2^2,\quad i=1,\dots,p.
$$
Equation
\eqref{KKTHGL} indicates that when tuning $\lambda$ there should be a link between
the ``norm'' of the
actual estimator
$\|\hat\theta^{(i)}(\lambda)\|^2$ to its a priori second moments \eqref{VarTheta}.
In particular, when
no regularization is imposed on $\lambda$ (i.e. $\gamma=0$) and the nonnegativity constraint is
not active, i.e. $\mu_i=0$, one finds that the optimal value of $\lambda_i$ makes the norm of
the estimator equal to (the trace of) its a priori matrix of second moments.

\subsection{GLasso does not derive from marginalization of the posterior}

Differently from the Lasso case, when the block size is larger than 1, GLasso does not derive
from marginalization of the Bayesian model depicted in Fig. \ref{BN}(b). 
To see this, consider the problem of integrating out $\lambda$ from the
joint density of $\theta$ and $\lambda$ described by the model in Fig. \ref{BN}(b). The
result is the product of multivariate Laplace densities. In particular, if $B^{(i)}(\cdot)$ is
the modified Bessel function of the second kind and order $k_i/2-1$, then, following
\citep{Eltoft2006}, we obtain
\begin{equation}\label{MultLap}
\int_{\lambda \in \R_+^{p}}
p(\theta,\lambda) d\lambda = \frac{(2\gamma)^p}{(2 \pi) ^{m/2}} \prod_{i=1}^p (2\gamma))^{2-k_i/4}
\frac{B^{(i)}(2\gamma\sqrt{\theta^{(i)\top}
\theta^{(i)}})}{(\theta^{(i)\top}\theta^{(i)})^{k_i/4-2} },
\end{equation}
whereas the prior
density underlying the GLasso must satisfy
\begin{equation}\label{GLassoprior}
p(\theta)
\propto \exp(-\gamma_{GL} \sum_{i=1}^p \norm{\theta^{(i)}}).
\end{equation}
One can show that, for $k_i>1$ with $\theta^{(i)}$ tending to zero the prior density on
$\theta^{(i)}$
used in the GLasso remains bounded, while the marginal of the density used for  HGLasso in (\ref{MultLap})
tends to $\infty$.

\section{Relationship with Multiple Kernel Learning}
\label{MKLandHGLasso}


Multiple Kernel Learning (MKL) can be used for the
block sparsity problem
\citep{Bach_MKL_2004,Evgeniou05,DinuzzoTL,Bach_Consistency_GLassoMKL}.
To introduce this approach consider
the measurements model
\begin{equation}\label{GroupMeasmod_f}
y= f+v = \sum_{i=1}^{p} f^{(i)} +v\ ,
\end{equation}
where $\nu$ is as specified in (\ref{p_theta}).
In the MKL framework, $f$ 
represents the sampled version of a scalar function
assumed to belong to a (generally infinite-dimensional)
reproducing kernel Hilbert space (RKHS) \cite{Wahba1990}.
For our purposes, we consider a simplified scenario where
the domain of the functions in the RKHS is the finite set $[1,\ldots,n]$.
In this way, $f$ represents the entire function and
$y$ is the noisy version of $f$ sampled over its whole domain.
In addition, we assume that $f$ belongs to the RKHS, denoted $\cH_K$, having kernel
defined by the matrix
\begin{equation}\label{Klambda}
K(\lambda) = \sum_{i=1}^{p} \lambda_i K^{(i)},
\end{equation}
where it is further assumed that
each of the functions $f^{(i)}$ is an element
of a RKHS, denoted $\cH^{{(i)}}$, having kernel $\lambda_i K^{(i)}$ with associated
norm denoted by $\| f^{(i)} \|_{(i)}$.

According to the MKL approach,
the estimates of the unknown functions $f^{(i)}$ are obtained \emph{jointly}
with those of the scale factors $\lambda_i$
by solving the following inequality constrained  problem:
\begin{eqnarray}\nonumber
&& (\{\hat{f}^{(i)}\} , \hat{\lambda}) = \displaystyle\mathop{\arg \min}_{\{f^{(i)}\},\lambda \in \R_+^p}
\frac{(y-f)^\top(y-f)}{\sigma^2} + \sum_{i=1}^{p}\| f^{(i)} \|_{{(i)}}^2 \\ \label{MKL}
&& \qquad \qquad \qquad \quad \mbox{s.t.}  \quad \sum_{i=1}^p \lambda_i \leq  M\ ,
\end{eqnarray}
where $M$ plays the role of a regularization parameter.
Hence, the ``scale factors''
contained in $\lambda\in \R^+_p$ are optimization variables,
thought of as ``tuning knobs'' adjusting the kernel
$K(\lambda)$ to better suit the measured data.
Using the extended version of the representer theorem,
e.g. see  \citep{DinuzzoTL,Evgeniou05}, the solution is
\begin{equation}\label{sol}
\hat{f}^{(i)} = \hat{\lambda}_i K^{(i)}  \hat{c}, \qquad i=1,\ldots,p,
\end{equation}
where
\begin{eqnarray}
\nonumber && \{\hat{c},\hat{\lambda}\} =\displaystyle\mathop{\arg \min}_{c \in \R^{n},\lambda \in \R_p^+}
\frac{(y-K(\lambda) c)^\top(y-K (\lambda) c)}{\sigma^2}
+ c^{\top} K(\lambda) c \\ \label{MKL2}
&& \qquad \qquad  \quad \mbox{s.t.}  \quad \sum_{i=1}^p \lambda_i \leq  M.
\end{eqnarray}
It can be shown that every local solution of the above optimization problem is also
a global solution, see \citep{DinuzzoTL} for details.\\

For our purposes, it is useful to define $\phi$
as the Gaussian vector with independent components of unit variance such that
\begin{equation}\label{Eqphi}
\theta_i = \sqrt{\lambda_i}\,\, \phi_i .
\end{equation}
We partition
$\phi$ conformally with $\theta$, i.e.
\begin{eqnarray}\label{Factorphi}
\phi = \left[{\phi^{(1)}}^\top \quad {\phi^{(2)}}^\top \quad \ldots \quad {\phi^{(p)}}^\top\right]^\top .
\end{eqnarray}
Then, the following connection with the
Bayesian model in Fig. \ref{BN}(b) holds.

\begin{theorem} \label{MKLBayes}
Consider the joint density of $\phi$ and $\lambda$
conditional on $y$ induced by the Bayesian network in Fig. \ref{BN}(b).
Set $K^{(i)}=G^{(i)}G^{(i)\top},\ i=1,\dots,p$.
Then, there exists a value of $\gamma$ (function of $M$) such that
the maximum a posteriori estimate of $\lambda$ for this value of
$\gam$ (obtained optimizing the joint density
of $\phi$ and $\lambda$)
is the $\hat{\lambda}$ from (\ref{MKL2}). In addition, for this value of $\gam$ one has
\begin{eqnarray}  \label{redMarg}
&& \hat{\lambda}=\arg \min_{\lambda \in \R_+^{p}}
\frac{y^\top(K(\lambda)+\sigma^2 I)^{-1}y}{2} + \gamma \sum_{i=1}^p \lambda_i
\end{eqnarray}
and the $\hat{c}$ in (\ref{MKL2}) is given by
\begin{equation}\label{optim_c}
\hat{c}(\hat{\lambda})=(K(\hat{\lambda})+\sigma^2 I)^{-1} y.
\end{equation}
Again, for this value of $\gam$,
the maximum a posteriori estimates of the blocks of $\phi$ are
\begin{equation}\label{optPhi}
\hat{\phi}^{(i)} = \sqrt{\lambda_i}  G^{(i)\top} \hat{c}\ .
\end{equation}
Finally, one has
\begin{equation}\label{MKL=GLasso}
\hat{\theta}_{GL}^{(i)}  = \sqrt{\lambda_i}  \hat{\phi}^{(i)} ,
\end{equation}
where 
$\hat{\theta}_{GL}$ is the GLasso estimate (\ref{GLasso})
for a suitable value of $\gamma_{GL}$.
\end{theorem}

We supply the KKT conditions
(\ref{redMarg}) in the
following proposition.

\begin{proposition}\label{KKTMKL}
The necessary and sufficient conditions for $\lambda$
to be a solution of (\ref{redMarg}) are
\begin{equation}\label{FullKKT}
\begin{array}{l}
 \Sigma_y = K(\lambda)+\sigma^2 I \\
 W\Sigma_y = I \\
 - \|G^{(i)\top}Wy\|_2^2
+
2\gamma -2\mu_i = 0, \quad i = 1, \dots, p \\
 \mu_i\lambda_i = 0, \quad i = 1, \dots, p \\
 0 \le \mu,\ \lambda \mbox{ and } 0 \preceq W, \Sigma_y .
\end{array}
\end{equation}
\end{proposition}
\begin{flushright}
$\blacksquare$
\end{flushright}

\subsection{Concluding remarks of the section}

Eq. \ref{MKL=GLasso} in Theorem \ref{MKLBayes} states the equivalence 
between MKL and GLasso. It is a particular instance of the relationship between
regularization on kernel weights and block-norm based regularization, see
Theorem 1 in \cite{Tomioka}.
In the next sections, such connections will help in 
understanding the differences between GLasso and HGLasso
by comparing the KKT conditions derived in Propositions \ref{KKTofHGL} and \ref{KKTMKL}.\\
Notice also that the GLasso estimate provides 
the maximum a posteriori (MAP) estimate of $\phi$ but not that of $\theta$.
In fact, $\sqrt{\lambda_i}  \hat{\phi}^{(i)}$
is not the MAP estimate of $\theta^{(i)}$. In this regard, it is not difficult
to see that, according to the model in Fig. \ref{BN}(b), the joint density of $\theta$ and $\lambda$
given $y$ is not bounded above in a neighborhood of the origin.
Hence, the MAP estimator of $\theta$ would always return
an estimate equal to zero.
One can however conclude from Theorem \ref{MKLBayes} that MKL (GLasso) 
arises from the same Bayesian model
as the HGLasso considering $\phi$ and $\lambda$ as unknown variables. 
The difference is that the MKL estimate of
$\lambda$ is obtained by maximizing a joint
rather than a marginal density. 
It is worth comparing the
expression for the MKL estimator in (\ref{redMarg}) with the expression for the HGLasso
estimator given in (\ref{AlternativeGlambda}).
Under the assumptions stated in Theorem \ref{MKLBayes},
$\Sigma_y(\lambda)=K(\lambda)+\sigma^2 I$. Hence,
the objectives in  (\ref{redMarg}) and (\ref{AlternativeGlambda})
differ only in
the term $ \frac{1}{2} \log \det(\Sigma_y)$ appearing
in the HGLasso objective (\ref{AlternativeGlambda}).
Notice also that this is the component that makes
problem (\ref{AlternativeGlambda}) non-convex.
On the other hand, it is also the term that forces
the HGLasso to favor sparser solutions than the
MKL since it makes the marginal density of $\lambda$
more concentrated around zero.

\section{Sparsity vs. Shrinkage:
A simple experiment} \label{SparsvsSh}

It is well known that the $\ell_1$ penalty in Lasso tends to induce an excessive shrinkage of ``large'' coefficient in order to obtain sparsity. Several variations have been proposed in the literature in order to overcome this problem, including the so called \emph{Smoothly-Clipped-Absolute-Deviation} (SCAD) estimator in \citep{FanLiSCAD_JASA2001} and re-weighted versions of $\ell_1$ like the \emph{adaptive Lasso} \citep{AdaptiveLasso}. We now study the tradeoffs between sparsity and shrinking for HLasso/HGLasso.
By way of introduction to the more general analysis in the next section,
we first compare the sparsity conditions for HGLasso and MKL (or, equivalently, GLasso)
in a simple, yet instructive, two group example.  In this example, it is straightforward to show that
HGLasso guarantees a  more favorable tradeoff between sparsity and shrinkage,
in the sense that it induces greater sparsity with the same
shrinkage (or, equivalently, for a given level of sparsity it
guarantees less shrinkage).

Consider
two groups of dimension $1$, i.e.
\begin{equation}\label{ExampleKKT}
y = G^{(1)} \theta^{(1)} + G^{(2)} \theta^{(2)} + v \quad y\in \R^2,\ \theta_1,\ \theta_2\in \R,
\end{equation}
where  $G^{(1)}=[1 \;\; \delta]^\top$, $G^{(2)}=[0 \;\; 1]^\top$, $v \sim {\cal N}(0,\sigma^2)$.
Assume $\theta^{(1)}=0$, $\theta^{(2)}=1$.
Our goal is to understand how the hyperparameter $\gamma$ influences sparsity and the estimates of
$\theta^{(1)}$ and $\theta^{(2)}$ using HGLasso and MKL.
In particular, we would like to determine which values of $\gamma$ guarantee that $\hat \theta^{(1)}=0$ and how the estimator  $\hat\theta^{(2)}$ varies with $\gamma$.
These questions can be answered by using the KKT conditions obtained
 in Propositions \ref{KKTofHGL} and \ref{KKTMKL}.

Let $y:=[y_1 \; y_2]^\top$ and recall that  $K^{(i)}:=G^{(i)} \left(G^{(i)}\right)^\top$.
By (\ref{KKTHGL}),
the necessary conditions for $\hat \lambda_1=0$
and $\hat\lambda_2\ge 0$ to be the hyperparameter estimators for the HGLasso estimator
(for fixed $\gamma$) are
 \begin{equation}\label{KKT_example_HGLasso}
 \begin{array}{l}
2\gamma_{HGL} \geq
\left[\frac{y_1}{\sig^2}+\frac{\delta y_2}{\sig^2+\hat\lam^{HGL}_2}\right]^2-
\left[\frac{1}{\sig^2}+\frac{\delta }{\sig^2+\hat\lam^{HGL}_2}\right]\quad\mbox{and}
\\ \\
\hat\lambda^{HGL}_2 =  {\max}
\left\{\frac{-1+\sqrt{1+8\gam_{HGL} y_2^2}}{4\gam_{HGL}}-\sig^2,\ 0\right\}.
\end{array}
 \end{equation}
 Similarly, by (\ref{FullKKT}), the same conditions for MKL read as
  \begin{equation}\label{KKT_example_MKL}
 \begin{array}{l}
2\gamma_{MKL} \geq
\left[\frac{y_1}{\sig^2}+\frac{\delta y_2}{\sig^2+\hat\lam^{MKL}_2}\right]^2\quad\mbox{and}\\ \\
\hat\lambda^{MKL}_2  =
{\max}\left\{\frac{|y_2|}{\sqrt{2\gam_{MKL}}}-\sig^2,0\right\}.
\end{array}
 \end{equation}
 Note that it is always the case that the lower bound for $\gam_{MKL}$ is strictly
 greater than the lower bound for $\gam_{HGL}$ and that
 $\hat\lambda^{HGL}_2\le\hat\lambda^{MKL}_2$ when $\gam_{HGL}=\gam_{MKL}$,
 where the inequality is strict whenever
 $\hat\lambda^{MKL}_2>0$.
 The corresponding estimators for $\theta^{(1)}$ and $\theta^{(2)}$ are
  \begin{equation}\label{ThetaHat}
 \begin{array}{c}
\hat \theta^{(1)}_{HGL}= \hat\theta^{(1)}_{MKL} = 0  \\ \\
\hat \theta^{(2)}_{HGL} = \frac{\hat\lambda^{HGL}_2y_2}{\sig^2+\hat\lambda^{HGL}_2}
\qquad\mbox{and}\qquad
\hat \theta_{MKL}^{(2)} = \frac{\hat\lambda^{MKL}_2y_2}{\sig^2+\hat\lambda^{MKL}_2}\ .
\end{array}
 \end{equation}
 Hence, $|\hat \theta^{(2)}_{HGL}|<|\hat \theta^{(2)}_{MKL}|$ whenever $y_2\ne 0$ and
 $\hat\lambda^{MKL}_2>0$. However,
it is clear that the lower bounds on $\gam$ in (\ref{KKT_example_HGLasso}) and (\ref{KKT_example_MKL})
indicate that $\gamma_{MKL}$ needs to be larger than $\gamma_{HGL}$
in order to set $\hat\lambda_1^{MKL}=0$ (and hence $\hat\theta_{MKL}^{(1)}=0$). 
Of course, having a larger $\gamma$ tends to yield smaller $\hat\lambda_2$ and hence more shrinking on $\hat\theta^{(2)}$. This is illustrated in figure \ref{Fig_HGLassovsMKL} where we report the estimators $\hat\theta_{HGL}^{(2)}$ (solid) and $\hat\theta^{(2)}_{MKL}$ (dotted) for
 $\sigma^2=0.005$, $\delta=0.5$. The estimators are arbitrarily set to zero for the values of $\gamma$ which do not yield $\hat\theta^{(1)}=0$. In particular from \eqref{KKT_example_HGLasso} and \eqref{KKT_example_MKL} we find that HGLasso sets $\hat\theta_{HGL}^{(1)}=0$ for $\gamma_{HGL} >5$ while
 MKL sets $\hat\theta_{MKL}^{(1)}=0$ for $\gamma_{MKL}>20$. 
 In addition, it is clear that MKL  tends to yield greater shrinkage on $\hat\theta_{MKL}^{(2)}$ (recall that $\theta^{(2)}=1$).

 \begin{figure}
  \begin{center}
    \includegraphics[width=.6\linewidth,angle=0]{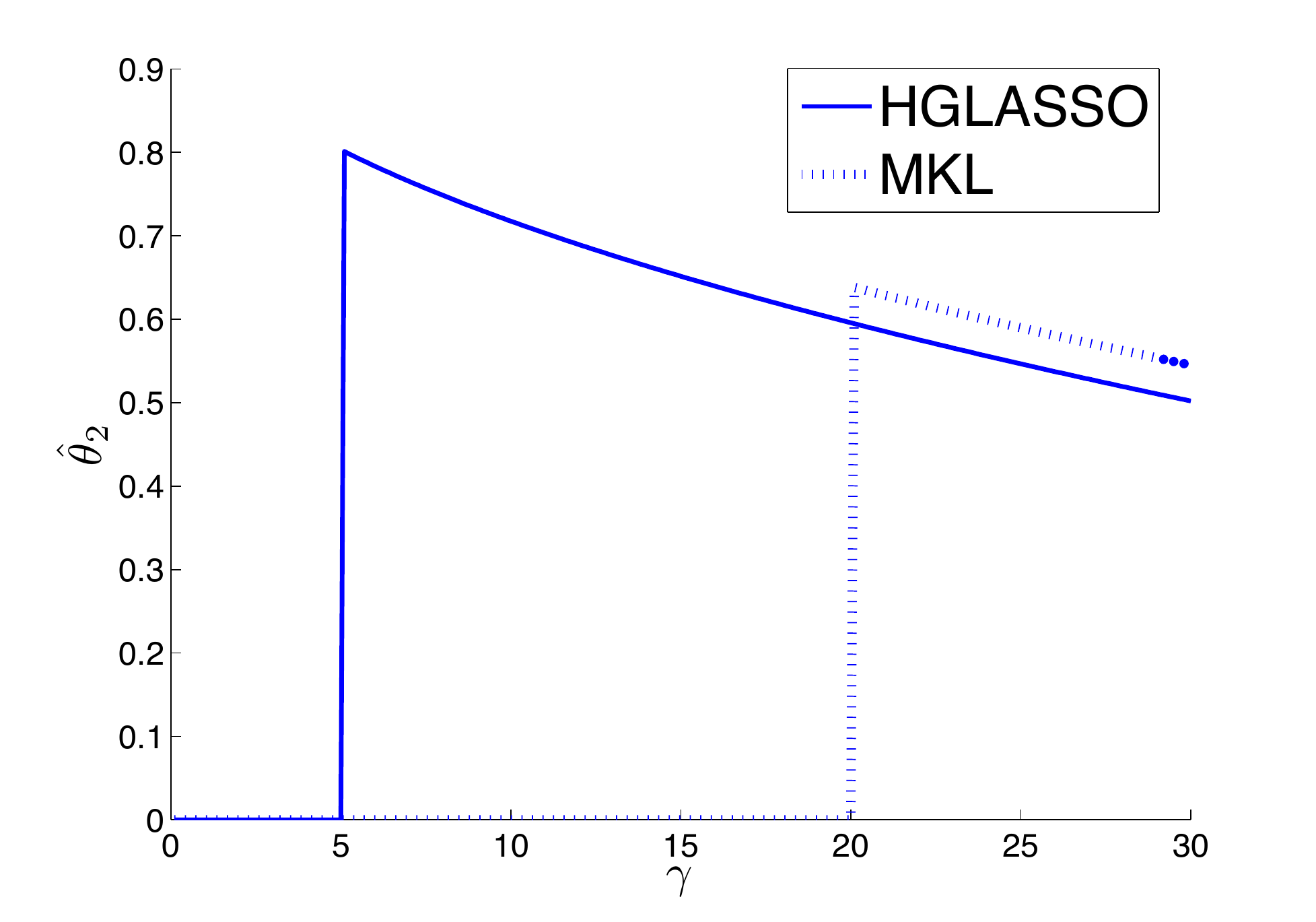}
\caption{Estimators $\hat\theta^{(2)}$ as a function of $\gamma$. The curves are plotted only for the values of $\gamma$ which yield also $\hat\theta^{(1)}=0$ (different for HGLasso ($\gamma_{HGL}>5$) and MKL ($\gamma_{MKL}>20$)).}
    \label{Fig_HGLassovsMKL}
  \end{center}
\end{figure}

\section{Mean Squared Error properties of Empirical Bayes Estimators}\label{sec:MSE}

In this Section we evaluate the performance of an estimator $\hat \theta$ using its Mean Squared Error (MSE) i.e. its expected quadratic loss
$$
\Trace\left[\E \left[\left.\left(\hat \theta-\btheta\right)\left(\hat \theta-\btheta\right)^\top\,\right|\,
\lambda,\theta=\btheta\right]\right],
$$
where $\btheta$ is the ``true'' but unknown value of $\theta$.
When we speak about ``Bayes estimators'' we think of estimators of the form
$\hat \theta(\lambda):= \E\left[\theta\, |\, y,\lambda\right]$ computed using the probabilistic
model Fig. \ref{BN} with $\gamma$ fixed.

\subsection{Properties using ``orthogonal'' regressors}

We first derive the MSE formulas
under the  simplifying assumption of ``orthogonal'' regressors $(G^\top G =n I)$ and
show that the Empirical Bayes estimator converges to an ``optimal'' estimator in terms of its MSE.
This fact has close connections to the so called ``Stein'' estimators
\citep{Stein1961}, \citep{SteinAS1981}, \citep{EfronMorris1973}.
The same optimality properties are attained,
asymptotically, when the columns of $G$ are realizations of uncorrelated processes
having the same variance. This is of interest in
the system identification scenario considered in \citep{ChiusoPCDC2010,ChiusoPNIPS2010,ChiusoPAuto2011}
since it arises when one performs identification with i.i.d. white noises as inputs.
We then consider the more general case of correlated regressors
(see Section \ref{sec:MSE2}) and show that essentially the same holds for a weighted
version of the MSE.

In this section, it is convenient
to introduce the following notation:
\[
\vE[\,\cdot\,]:=\E[\,\cdot\,|\,\lam,\,\theta=\btheta]\quad\mbox{and}\quad
\vVar[\,\cdot\,]:=\E[\,\cdot\,|\,\lam,\,\theta=\btheta] .
\]
We now report an expression for the MSE of the Bayes estimators
$\hat \theta(\lambda):= \E\left[\theta\,|\,y,\lambda\right]$ (proof follows from
standard calculations and is therefore omitted). 

\begin{proposition}\label{prop_MSE}
Consider the model \eqref{GroupMeasmod}
under the probabilistic model
described in Fig. \ref{BN}(b).
The Mean Squared Error of the Bayes estimator
 $\hat \theta(\lambda):= \E\left[\theta|y,\lambda\right]$ given $\lam$ and $\theta=\btheta$
 is 
 \begin{eqnarray}
 \nonumber
 MSE(\lambda)&=&\!\!\!\Trace\left[\vE\left[(\hat \theta(\lambda)-\theta)
 (\hat \theta(\lambda)-\theta)^\top\right]\right]
 \\ \label{MSE}
  &=&\!\!\! {\Trace}\left[\sigma^2 \left(G^\top G + \sigma^2 \Lambda^{-1}\right)^{-1}\left(G^\top G + \sigma^2\Lambda^{-1} \btheta\btheta^\top\Lambda^{-1}  \right)\left(G^\top G + \sigma^2 \Lambda^{-1}\right)^{-1}\right].
 \end{eqnarray}
\end{proposition}

We can now minimize the expression for $MSE(\lambda)$ given in (\ref{MSE}) with respect to
$\lam$ to obtain the optimal minimum mean squared error estimator. In the case where
$G^\top G=nI$ this computation is straightforward and is recorded in the following proposition.

\begin{corollary}\label{prop_min_MSE}
Assume that $G^\top G = nI$ in Proposition \ref{prop_MSE}.
Then
MSE($\lam$) is globally minimized by choosing
\begin{equation}\label{lambda_opt}
\lambda_i=\lambda^{opt}_i :=\frac{\|\btheta^{(i)}\|^2}{k_i},\quad i=1,\dots,p.
\end{equation}
\end{corollary}

Next consider the Maximum a Posteriori estimator of $\lambda$ again
under the simplifying assumption $G^\top G = n I$.
Note that, under the noninformative prior ($\gamma=0$), this Maximum a Posteriori  
estimator reduces to the standard Maximum (marginal) Likelihood approach to estimating the prior distribution of $\theta$.
Consequently, we continue to call the resulting procedure
Empirical Bayes (a.k.a. Type-II Maximum Likelihood, \citep{BergerBook}).

\begin{proposition}\label{Prop_lambda_HGLasso}
Consider model \eqref{GroupMeasmod}  under the probabilistic model
described in Fig. \ref{BN}(b), and assume that $G^\top G = n I$.
Then the estimator of $\lambda_i$ obtained by maximizing the marginal posterior ${\bf p}(\lambda|y) $,
\begin{equation}\label{LambdaEB}
\{\hat\lambda_1(\gamma),...,\hat\lambda_p(\gamma)\}:=
\arg \max_{{\lambda} \in \R_+^{p}} {\bf p}(\lambda|y) =
\arg \max_{{\lambda} \in \R_+^{p}}\int {\bf p}(y,\theta|\lambda){\bf p}_\gamma(\lambda)\, d\theta,
\end{equation}
is given by
\begin{equation}\label{LambdaEB_explicit_gamma}
\hat \lambda_i(\gamma) = {\rm max}\left(0,
\frac{1}{4\gam}\left[\sqrt{k_i^2+8\gam \norm{\hat\theta^{(i)}_{LS}}^2}-
\left(k_i+\frac{4\sig^2\gam}{n}\right)\right]\right)\ ,
\end{equation}
where
$$
\hat\theta_{LS}^{(i)}=\frac{1}{n}\left(G^{(i)}\right)^\top y
$$
is the Least Squares estimator of the $i-$th block $\theta^{(i)}$.
As $\gamma \rightarrow 0$ ($\gamma=0$ corresponds to an improper flat prior) the expression \eqref{LambdaEB_explicit_gamma} yields:
\begin{equation}\label{LambdaEB_explicit}
\mathop{\rm lim}_{\gamma\rightarrow 0} \hat \lambda_i(\gamma) = {\rm max}\left(0,\frac{\|\hat\theta_{LS}^{(i)}\|^2}{k_i} - \frac{\sigma^2}{n}\right)\ .
\end{equation}
In addition,
the probability $\P[\hat\lambda_i(\gamma) =0\,|\,\theta=\btheta]$ of setting $\hat\lambda_i=0$ is given by
\begin{equation}\label{eqP0}
\P [\hat\lambda_i(\gamma) =0\,|\,\theta=\btheta] =
\P \left[\chi^2\left(k_i,\|\btheta^{(i)}\|^2\frac{n}{\sigma^2}\right)
\leq \left(k_i+2\gamma\frac{\sigma^2}{n}\right)\right]\ ,
\end{equation}
where $\chi^2(d,\mu)$ denotes a noncentral $\chi^2$ random variable with $d$ degrees of freedom and noncentrality parameter $\mu$.
\end{proposition}

Note that the expression of $\hat\lambda_i(\gamma)$ in Proposition \ref{Prop_lambda_HGLasso}
has the form of a  ``saturation''. In particular,  for $\gamma=0$, we have
\begin{equation}\label{unsaturated}
\hat\lambda_i(0) = {\rm max}(0,\hat\lambda_i^*), \quad
\mbox{where}\quad \hat\lambda_i^*:= \frac{\|\hat\theta_{LS}^{(i)}\|^2}{k_i} - \frac{\sigma^2}{n}\ .
\end{equation}
The following proposition shows that the ``unsaturated'' estimator $\hat\lambda_i^*$ is an
unbiased and consistent estimator of $\lambda_i^{opt}$ which minimizes the
Mean Squared Error while $\hat\lambda_i(0)$
is only asymptotically unbiased and consistent.

\begin{corollary}\label{Prop_unbiased}
Under the assumption $G^\top G = nI$, the estimator of
$\hat\lambda^*:=\{\lambda_1^*,..,\lambda_p^*\}$ in \eqref{unsaturated} is an unbiased and mean square consistent estimator of
$\lambda^{opt}$ which minimizes the Mean Squared Error,
while $\hat\lambda(0):=\{\lambda_1(0),..,\lambda_p(0)\}$
is asymptotically unbiased and consistent, i.e.:
\begin{equation}\label{unbiased}
\E [\hat \lambda^*_i\,|\,\theta=\btheta ] =
\lambda^{opt}_i \quad
\mathop{\rm lim}_{n\rightarrow \infty} \E [\hat \lambda_i(0) \,|\,\theta=\btheta ]= \lambda^{opt}_i
\end{equation}
and
\begin{equation}\label{consistent}
\mathop{\rm lim}_{n\rightarrow \infty} \hat \lambda^*_i \mathop{=}^{m.s.}\lambda^{opt}_i \quad \mathop{\rm lim}_{n\rightarrow \infty} \hat \lambda_i(0) \mathop{=}^{m.s.} \lambda^{opt}_i
\end{equation}
where $\displaystyle{\mathop{=}^{m.s.}}$ denotes convergence in mean square.
\end{corollary}

\begin{remark}\label{rem:lambda_zero}
Note that if $\btheta^{(i)}=0$ the optimal value  $\lambda_i^{opt}$ is zero.
Hence \eqref{consistent} shows that asymptotically $\hat\lambda_i(0)$ converges to zero.
However, in this case,
 it is easy to see from \eqref{eqP0} that
$$
\mathop{\rm lim}_{n\rightarrow\infty}\P[\hat \lambda_i(0) = 0\,|\,\theta=\btheta ] <1.
$$
There is in fact no contradiction between these two statements because one can easily show that
for all $\epsilon >0$,
$$\P[\hat \lambda_i(0) \in [0,\epsilon)\,|\,\theta=\btheta ] \mathop{\longrightarrow}^{n\rightarrow\infty} 1.$$
In order to guarantee that
$\mathop{\rm lim}_{n\rightarrow\infty}\P[\hat \lambda_i(\gamma) = 0\,|\,\theta=\btheta ] =1$
one must chose
$\gamma=\gamma_n$ so that $2\frac{\sigma^2}{n}\gamma_n \rightarrow \infty$,
so that $\gamma_n$ grows faster than $n$.
This is in line with the well known requirements for Lasso to be model selection consistent.
In fact, Theorem \ref{Lassomarg} shows that the link between   $\gamma$  and the regularization parameter $\gamma_L$ for Lasso is given by $\gamma_L = \sqrt{2\gamma}$. The condition
$n^{-1}\gamma_n \rightarrow \infty$ translates into $n^{-1/2} \gamma_{Ln} \rightarrow \infty$,
 a well known condition for Lasso to be model selection consistent \citep{Model_Selection_Lasso,Bach_Consistency_GLassoMKL}.
\end{remark}

The results obtained so far suggest that the Empirical Bayes resulting from HGLasso has desirable properties with respect to the MSE of the estimators.
One wonders whether the same favorable properties are inherited by MKL or, equivalently, by GLasso.
The next proposition shows that this is not the case.
In fact, for $\btheta^{(i)}\neq 0$, MKL does not yield consistent estimators for $\lambda_i^{opt}$;
in addition, for $\theta^{(i)}=0$, the probability of setting $\hat\lambda_i(\gamma)$ to zero
(see equation \eqref{eqP0MKL}) is much smaller than that obtained using HGLasso (see equation \eqref{eqP0}); this is also illustrated in Figure \ref{fig:HGLvsMKL} (top).
Also note that, as illustrated in Figure \ref{fig:HGLvsMKL} (bottom),
when the ``true'' $\theta$ is equal to zero, MKL tends to give much larger values of $\hat\lambda$
than those given by HGLasso.
This results in larger values of $\|\hat\theta\|$ (see Figure \ref{fig:HGLvsMKL}).

\begin{proposition}\label{Prop_lambda_MKL}
Consider model \eqref{GroupMeasmod} under the probabilistic model
described in Fig. \ref{BN}(b),  and assume $G^\top G = n I$.
Then the estimator of $\lambda_i$ obtained by maximizing the joint
posterior ${\bf p}(\lambda,\phi|y) $ (see equations \eqref{Eqphi} and \eqref{Factorphi}),
\begin{equation}\label{LambdaEBMKL}
\{\hat\lambda(\gamma),...,\hat\lambda_p(\gamma)\}
:= \arg \max_{{\lambda} \in \R_+^{p}, {\phi} \in \R_+^{m}} {\bf p}(\lambda,\phi|y),
\end{equation}
is given by
\begin{equation}\label{LambdaEB_explicit_gammaMKL}
\hat \lambda_i(\gamma) = {\rm max}
\left(0,\frac{\|\hat\theta_{LS}^{(i)}\|}{\sqrt{2\gamma}} - \frac{\sigma^2}{n}\right),
\end{equation}
where
$$
\hat\theta_{LS}^{(i)}=\frac{1}{n}\left(G^{(i)}\right)^\top y
$$
is the Least Squares estimator of the $i-$th block $\theta^{(i)}$ for $i=1,\dots,p$.
For $n \rightarrow \infty$ the estimator $\hat \lambda_i(\gamma)$ satisfies
\begin{equation}\label{not_consistent_MKL}
 \mathop{\rm lim}_{n\rightarrow\infty}
 \hat \lambda_i(\gamma)  \mathop{=}^{m.s.} \frac{\|\btheta^{(i)}\|}{\sqrt{2\gamma}}\ .
 \end{equation}
In addition,
the probability $\P[\hat\lambda_i(\gamma) =0\,|\,\theta=\btheta ]$
of setting $\hat\lambda_i(\gamma)=0$ is given by
\begin{equation}\label{eqP0MKL}
\P_\theta[\hat\lambda_i(\gamma) =0\,|\,\theta=\btheta ] =
\P \left[\chi^2\left(k_i,\|\btheta^{(i)}\|^2\frac{n}{\sigma^2}\right)\leq 2\gamma\frac{\sigma^2}{n}\right]\ .
\end{equation}
\end{proposition}

Note that the limit of the MKL estimators $\hat\lambda_i(\gamma)$ as $n\rightarrow\infty$ depends on
$\gamma$.
Therefore, using MKL (GLASSO), one cannot hope to get consistent estimators of $\lambda_i^{opt}$.
Indeed, for $\|\btheta^{(i)}\|^2\neq 0$, consistency of $\hat\lambda_i(\gamma)$
requires $\gamma \rightarrow \frac{k^2_i}{2\|\btheta^{(i)}\|^2}$, which is a circular requirement.

\begin{figure*}
  \begin{center}
  \begin{tabular}{cc}
    \includegraphics[width=0.5\linewidth,angle=0]{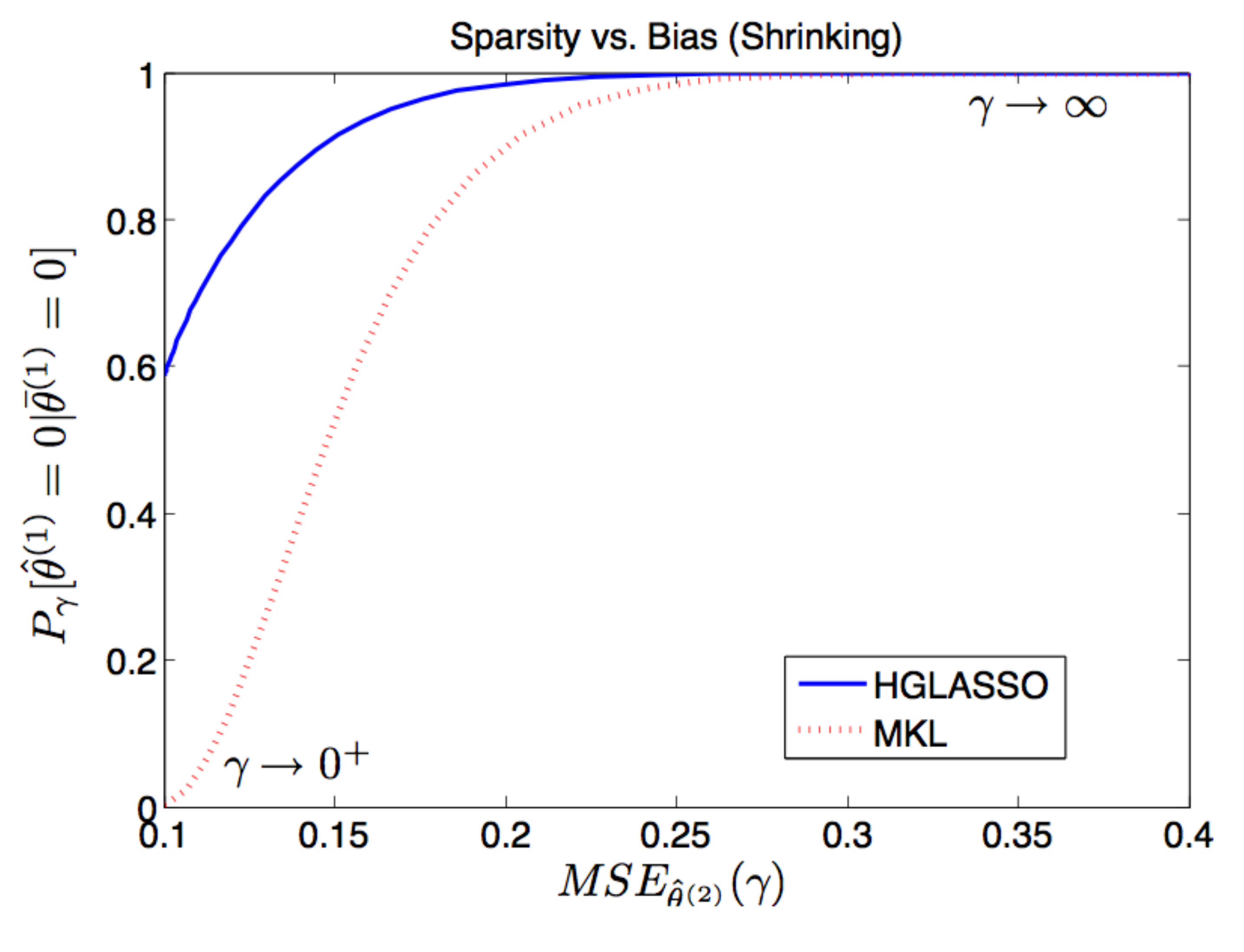}&
    \includegraphics[width=0.5\linewidth,angle=0]{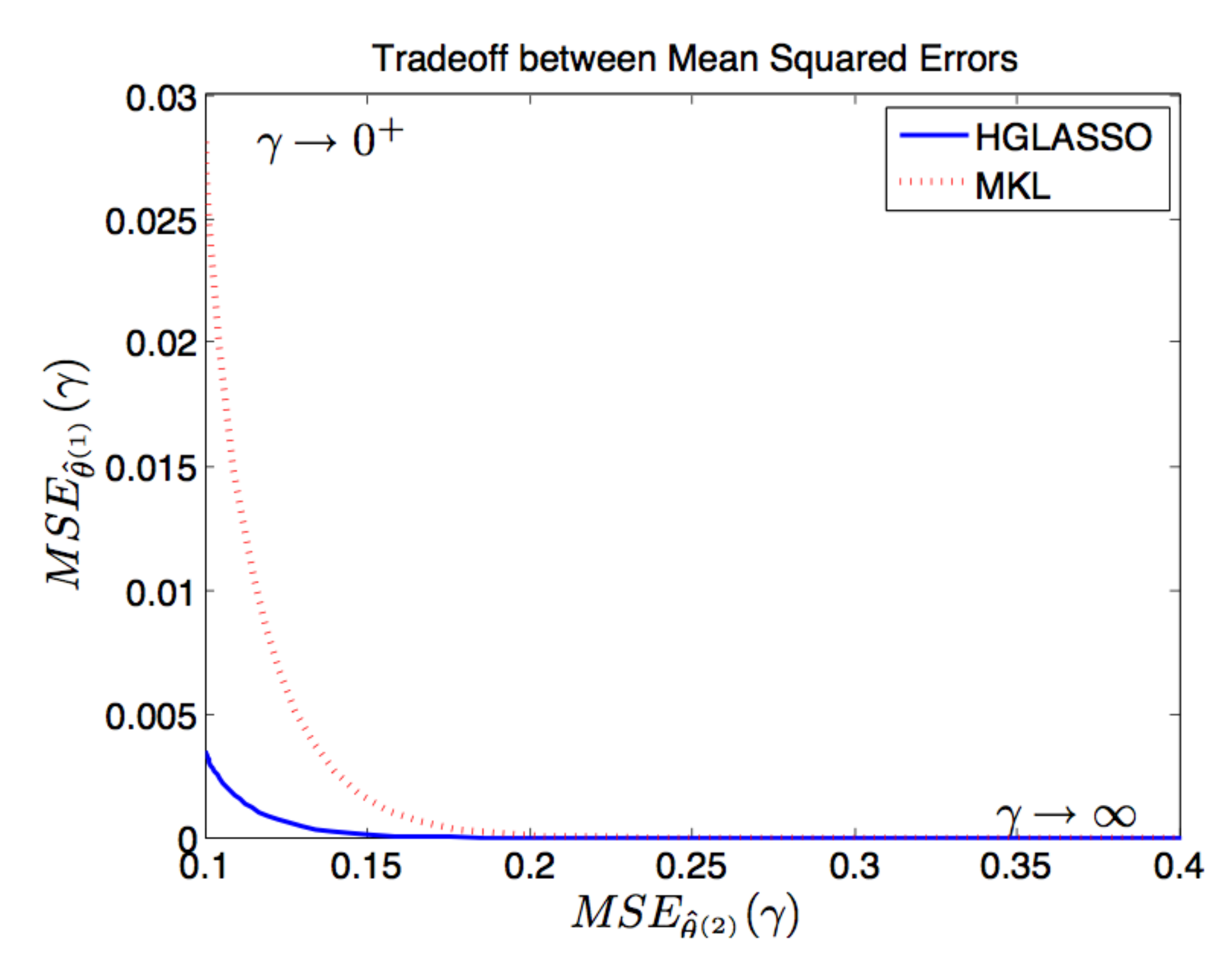}
    \end{tabular}
\caption{This plot has been generated assuming that there are two  blocks ($p=2$) of dimension $k_1=k_2=10$ with $\bar\theta^{(1)}=0$ and all the components of the true $\bar\theta^{(2)} \in \R^{10}$ set to one. The matrix $G$ equal to the identity, so that the output dimension ($y\in \R^n$) is $n=20$; the noise variance equal to $0.1$.
Left: probability of setting $\hat\theta^{(1)}$ to zero vs Mean Squared Error in $\hat\theta^{(2)}$. Curves are parametrized  in $\gamma \in [0,+\infty)$. Right:  Mean Squared Error in $\hat\theta^{(1)}$ vs Mean Squared Error in $\hat\theta^{(2)}$. Curves are parametrized  in $\gamma \in [0,+\infty)$.
}
    \label{fig:HGLvsMKL}
  \end{center}
\end{figure*}

\subsection{Asymptotic properties using general regressors}\label{sec:MSE2}

In this subsection, we replace the deterministic matrix $G$ with $G_n(\omega)$,
where $G_n(\omega)$ represents an $n \times m$ matrix defined on the complete
probability space $(\Omega, \mathcal{B},\P)$ with $\omega$ a generic element
of $\Omega$ and  $\mathcal{B}$ the sigma field of Borel regular measures. In particular, the rows of $G_n$ are independent\footnote{The independence assumption can be removed and replaced by mixing conditions.} realizations from a zero-mean random vector with positive definite covariance $\Psi$. 
We will also assume that
the (mild) assumptions for the convergence in probability of $G_n^ \top G_n/n$ to $\Psi$,
as $n$ goes to $\infty$, are satisfied, see e.g. \citep{Loeve}.\\
As in the previous part of this section,
$\lambda$ and $\theta$ are seen as parameters, and the ``true''
value of $\theta$ is $\btheta$.
Hence, all the randomness present in the next formulas
comes only from $G_n$ and the measurement noise.
Below, the dependence of $\Sigma_y(\lam)$ on $G_n$, and hence of $n$, 
is omitted to simplify the notation. Furthermore, $ \longrightarrow_p$
denotes convergence in probability.\\

\begin{theorem}
\label{MainConvergenceTheorem}
For known $\gamma$ and conditional on $\theta=\btheta$, define 
\begin{equation}\label{lambda^n}
\hat{\lambda}^n =  \arg \min_{{\lambda} \in \mathcal{C} \bigcap \R_+^{p}}  \frac{1}{2} \log
\det(\Sigma_y(\lam)) + \frac{1}{2} y^\top \Sigma_y^{-1}(\lam) y + \gamma
\sum_{i=1}^{p} {\lambda}_i ,
\end{equation}
where $\mathcal{C}$ is any $p$-dimensional ball with radius larger than 
$\max_i \frac{\| \btheta^{(i)}\|^2}{k_i}$.\\
Then, we have  
\begin{eqnarray}\label{conv1}
& \hat \lambda_{i}^n \longrightarrow_p \frac{-k_i + \sqrt{k_i^2 + 8\gamma\|\btheta^{(i)}\|^2}}{4\gamma} \quad & \mbox{if \quad $\gamma>0$ \quad and \quad $\|\theta^{(i)}\| > 0$,}
\\ \label{conv2}
& \hat \lambda_{i}^n \longrightarrow_p  \frac{\|\btheta^{(i)}\|^2}{k_i} \quad & \mbox{if \quad  $\gamma=0$ \quad and \quad $\|\theta^{(i)}\| > 0$, and} \\ \label{conv3}
& \hat \lambda_{i}^n \longrightarrow_p  0 \quad & \mbox{if \quad  $\gamma\geq 0$ \quad and \quad $\|\theta^{(i)}\| = 0$.} 
\end{eqnarray}
\end{theorem}

We now show that, when $\gamma=0$, the above result 
relates to the problem of minimizing
the MSE of the $i$-th block with respect to $\lambda_i$,
with all the other components of
$\lambda$ coming from $\hat{\lambda}^n$.
If $\hat \theta^{(i)}_n(\lambda)$ denotes the $i$-th component of 
the HGLasso estimate of $\theta$ defined in (\ref{AlternativeGteta}),
our aim is to optimize the objective
$$
MSE_n(\lambda_i) := \Trace\left[\vE\left[(\hat \theta_n^{(i)}(\lambda)-\theta^{(i)})
 (\hat \theta_n^{(i)}(\lambda)-\theta^{(i)})^\top\right]\right] \quad \mbox{with} \quad  \lambda_j=\bar{\lambda}_j^n \quad \mbox{for} \quad j \neq i
$$
where is $\bar\lambda_j^n$ is any sequence satisfying condition 
\begin{equation}\label{Mlambda-first}
\begin{array}{c} \displaystyle{\mathop{\rm lim}_{n\rightarrow\infty}}\; f_n =+\infty \quad {\rm where} \quad f_n : = \displaystyle{\mathop{\rm min}_{j\in I_1}} \;n\lambda_j^n,
\end{array}
\end{equation}
where $I_1:=\{j\,:\, j\ne i\mbox{ and }\bar{\theta}^{(j)}\ne 0\}$
(condition \eqref{Mlambda-first} appears again in the Appendix as \eqref{Mlambda}).  Note that, in particular, $\bar\lambda_j^n = \hat \lambda_j^n $ in \eqref{lambda^n} satisfy  \eqref{Mlambda-first} in probability. 

In the following lemma, whose proof is in the Appendix, 
we introduce a change of variables that is key for our understanding
of the asymptotic properties of these more general regressors.

\begin{lemma}\label{lemma_diagonal_model}
Fix $i \in \{1, \dots, p\}$ and consider the decomposition
\begin{equation}\label{GroupMeasmod_i}
\begin{array}{rcl}
y &=& G^{(i)} \theta^{(i)} +\sum_{j=1, j\neq i}^{p}  G^{(j)} \theta^{(j)} +v \\
& = &  G^{(i)} \theta^{(i)} + \bar v
\end{array}
\end{equation}
of the
linear measurement model  \eqref{GroupMeasmod}  and assume~\eqref{BoundedG} holds. 
Define
$$\Sigma_{\bar v}:= \sum_{j=1, j\neq i}^{p}  G^{(j)} \left(G^{(j)}\right)^\top \lambda_j + \sigma^2I $$
and assume that $\lambda_j$ finite $\forall j\neq i$. Consider now the singular value decomposition
\begin{equation}\label{SVD}
\frac{\Sigma_{\bar v}^{-1/2} G^{(i)}}{\sqrt{n}}   =U_n^{(i)} D_n^{(i)} \left(V_n^{(i)}\right)^\top
\end{equation} where each $D_n^{(i)}=\mbox{diag}(d_{k,n}^{(i)})$ is $k_i\times k_i$
diagonal matrix. 
Then \eqref{GroupMeasmod_i} can be transformed into the equivalent linear model
\begin{equation}\label{Diagonal_model}
\begin{array}{rcl}
z_n^{(i)} &=& D_n^{(i)} \beta_n^{(i)} +\epsilon_n^{(i)},
\end{array}
\end{equation}
where
\begin{equation}\label{Diagonal_model_definition}
\begin{array}{ccc}
z_n^{(i)} := \left(U_n^{(i)}\right)^\top \frac{\Sigma_{\bar v}^{-1/2} y}{\sqrt{n}}=(z_{k,n}^{(i)}),&
\beta_n^{(i)} := \left(V_n^{(i)}\right)^\top \theta^{(i)}=(\beta_{k,n}^{(i)}),&
\epsilon_n^{(i)} := \left(U_n^{(i)}\right)^\top \frac{\Sigma_{\bar v}^{-1/2} \bar v}{\sqrt{n}}
=(\epsilon_{k,n}^{(i)}),
\end{array}
\end{equation}
and  $D_n^{(i)}$ is uniformly (in $n$) bounded and bounded away from zero.
\end{lemma}

Lemma \ref{lemma_diagonal_model} shows that
we can consider the transformed linear model associated with the $i$-th block,
i.e.
\begin{equation}\label{Diagonal_model_int}
\begin{array}{rcl}
z_{k,n}^{(i)} &=& d_{k,n}^{(i)} \beta_{k,n}^{(i)} +\epsilon_{k,n}^{(i)}, \quad k=1,\ldots,k_i,
\end{array}
\end{equation}
where all the three variables on the RHS depend on $\bar{\lambda}_j^n$ for $j \neq i$.
In particular, the vector $\beta_{n}^{(i)}$ consists of an orthonormal transformation of $\theta^{(i)}$ 
while the $d_{k,n}^{(i)}$ are all bounded below in probability. 
In addition, 
by letting
\begin{equation}\label{epsilon}
\vE\left[ \epsilon_{k,n}^{(i)} \right] = m_{k,n}, \quad  \vE\left[ (\epsilon_{k,n}^{(i)}-m_{k,n})^2  \right] = \sigma^2_{k,n},
\end{equation}
we also know from Lemma \ref{lemma_structure_epsilon} (see equations \eqref{m_sigma_epsilon} and \eqref{convergence_variance_conditioned})
that, provided $\bar\lambda_j^n$ $(j\neq i)$ satisfy condition \eqref{Mlambda-first}, 
both $m_{k,n}$ and $\sigma^2_{k,n}$ tend to zero (in probability) as $n$ goes to $\infty$. 
Then, after simple computations, one finds that the $MSE$ relative to
$\beta_{n}^{(i)}$ is the following random variable whose statistics depend on $n$: 
\begin{eqnarray}\label{MSE2} \nonumber
MSE_n(\lambda_i) =
\sum_{k=1}^{k_i} \frac{ \beta_{k,n}^{2} +n \lambda_i^2 d_{k,n}^{2} (m_{k,n}^2 + \sigma^2_{k,n}) -2\lambda_i  d_{k,n} m_{k,n}\beta_{k,n} }
{ (1 + n \lambda_i  d_{k,n}^{2})^2 }  \quad \mbox{with} \quad  \lambda_j=\bar{\lambda}_j^n \quad \mbox{for} \quad j \neq i .
\end{eqnarray}
Above, except for $\lambda_i$, the dependence on the block number $i$ was omitted
to improve readability.\\
Now, let $\breve{\lambda}_i^{n}$ denote the minimizer of the following weighted
version of the $MSE_n(\lambda_i)$:
$$
\breve{\lambda}_i^{n} =\arg \min_{\lambda  \in \R_+}
\sum_{k=1}^{k_i} d_{k,n}^{4}\frac{ \beta_{k,n}^{2} +n \lambda_i^2 d_{k,n}^{2} (m_{k,n}^2 + \sigma^2_{k,n}) -2\lambda_i  d_{k,n} m_{k,n}\beta_{k,n} }
{ (1 + n \lambda_i  d_{k,n}^{2})^2 } .
$$
Then, the following result holds.

\begin{proposition}
\label{CommonConvergence}
For $\gamma=0$ and conditional on $\theta=\btheta$, 
the following convergences in probability hold 
\begin{equation}\label{brevelambda} 
\lim_{n \mapsto \infty }  \breve{\lambda}_i^{n} =\frac{\|\btheta^{(i)}\|^2}{k_i}
= \lim_{n \mapsto \infty }  \hat{\lambda}_i^{n}\ , \quad i=1,2,\ldots,p.
\end{equation}
\end{proposition}

The proof follows arguments similar to those used in last part of the proof of Theorem \ref{MainConvergenceTheorem}, see also proof of 
Theorem 6 in \cite{Sysid1}, and is therefore omitted.

We can summarize the two main findings reported in this subsection as follows.
As the number of measurements go to infinity:
\begin{enumerate}
\item  regardless of the value of $\gamma$, the proposed estimator will correctly set to zero only those $\lambda_i$ associated with null blocks;
\item when $\gamma=0$, (\ref{conv2}) and (\ref{conv3}) provide the asymptotic properties of 
ARD, showing that the estimate of $\lambda_i$ will converge to the energy of the $i$-th block (divided by its dimension). This same value also represents the asymptotic minimizer of a weighted version of the $MSE$ relative to the $i$-th block. In particular, the weights change over time, being defined by singular values $d_{k,n}^{(i)}$,
(raised at fourth power) that depend on the trajectories of the other components of $\lambda$.
\end{enumerate}

\subsubsection{Marginal likelihood and weighted MSE: perturbation analysis}\label{sec:MSE3}

We now provide some additional insights on point 2 above, investigating 
why the weights $d_{k,n}^{4}$ may lead to an effective strategy for 
hyperparameter estimation.\\ 
For our purposes, 
just to simplify the notation, let us consider the case of a single
$m$-dimensional block. In this way, $\lambda$ becomes a scalar and the noise $\epsilon_{k,n}$ 
in (\ref{Diagonal_model_int}) is zero-mean of variance $1/n$.\\ 
Under the stated assumptions, the $MSE$ weighted by $d_{k,n}^{\alpha}$, with $\alpha$
an integer, becomes
\begin{equation}\label{WeightedMSEsb}
\sum_{k=1}^{m} d_{k,n}^{\alpha}  \frac{ n^{-1} \beta_{k,n}^{2} +\lambda^2 d_{k,n}^{2} }
{ (n^{-1} + \lambda  d_{k,n}^{2})^2 } ,
\end{equation}
whose partial derivative with respect to $\lambda$, apart from the scale factor $2/n$,  is
\begin{equation}\label{DerWeightedMSEsb}
F_{\alpha}(\lambda)=  \sum_{k=1}^{m} d_{k,n}^{\alpha+2} \frac{ \lambda - \beta_{k,n}^{2}}
{ (n^{-1} + \lambda  d_{k,n}^{2})^3 } .
\end{equation}
Let $\beta_k = \lim_{n \mapsto \infty} \beta_{k,n}$ and  $d_k = \lim_{n \mapsto \infty} d_{k,n}$
\footnote{We are assuming that both of the limits exist. This holds 
under conditions ensuring that the SVD decomposition leading
to (\ref{Diagonal_model_int}) is unique, e.g. see the discussion in Section 4 of \citep{Bauer-05},
and combining the convergence of sample
covariances with a perturbation result for the
Singular Value Decomposition of
symmetric matrices (such as Theorem 1 in \citep{Bauer-05}, see also \citep{Chatelin})}.
When $n$ tends to infinity, arguments similar
to those introduced in the last part of the proof of Theorem \ref{MainConvergenceTheorem}
show that, in probability, the zero of $F_{\alpha}$ becomes 
\begin{equation}\label{Noiselesslambda}
\breve{\lambda}(\alpha) =   \frac{ \sum_{k=1}^{m} d_{k}^{\alpha-4}  \beta_{k}^{2}}
{  \sum_{k=1}^{m} d_{k}^{\alpha-4} } .
\end{equation}
Notice that the formula above is a generalization of the first equality in (\ref{brevelambda})
that was obtained by setting $\alpha=4$. 
However, for practical purposes, the above expressions
are not useful since the true values of $\beta_{k,n}$ and  $\beta_{k}$ 
depend on the unknown $\btheta$.
One can then consider a noisy version of $F_{\alpha}$ obtained by replacing
$\beta_{k,n}$ with its least squares estimate, i.e.
\begin{equation}\label{DerWeightedMSEsb}
\tilde{F}_{\alpha}(\lambda)= \sum_{k=1}^{m} d_{k,n}^{\alpha+2}  \frac{ \lambda - \left(\beta_{k,n} +
\frac{v_{k,n}}{\sqrt{n}d_{k,n}}\right)^{2}}
{ (n^{-1} + \lambda  d_{k,n}^{2})^3 } ,
\end{equation}
where the random variable $v_{k,n}$ is of unit variance.
For large $n$, considering 
small additive perturbations around the model $z_k=d_k\beta_k$,
it is easy to show that the
minimizer tends to the following perturbed version of $\breve{\lambda}$: 
\begin{equation}\label{Noisylambda}
\breve{\lambda}(\alpha)  +   2\frac{ \sum_{k=1}^{m}  d_{k}^{\alpha-5} \beta_k v_{k,n} }
{ \sqrt{n} \sum_{k=1}^{m}   d_{k}^{\alpha-4} } .
\end{equation}

It remains to choose the value of $\alpha$ that should enter the above formula.
This is far from trivial
since the optimal value (minimizing MSE) depends on the unknown $\beta_k$.
On one hand, it would seem advantageous to have 
$\alpha$ close to zero. In fact, $\alpha=0$ relates $\breve{\lambda}$ 
to the minimization of the $MSE$ on $\theta$ while $\alpha=2$ minimizes 
the $MSE$ on the output prediction, see the discussion in Section 4 of \cite{Sysid1}.
On the other hand,
a larger value for $\alpha$ could help in controlling the additive perturbation term
in (\ref{Noisylambda}) possibly reducing its sensitivity to small values of $d_k$. 
For instance, the choice $\alpha=0$ introduces
in the numerator of (\ref{Noisylambda}) the term $\beta_k/d_k^{5}$. 
This can make numerically unstable the convergence towards  $\breve{\lambda}$,
leading to poor estimates of the regularization parameters, 
as e.g. described via simulation studies in Section 5 of \cite{Sysid1}.
In this regard, the choice
$\alpha=4$ appears interesting: it sets $\breve{\lambda}$ 
to the energy of the block divided by $m$, removing the dependence 
of the denominator in (\ref{Noisylambda}) on $d_k$. In particular, it reduces 
(\ref{Noisylambda}) to
\begin{equation}\label{NoisylambdaW4}
\frac{\| \beta \|^2}{m}  +  \frac{2}{m}  \sum_{k=1}^{m} \frac{ \beta_k v_{k,n} } 
{ \sqrt{n} d_{k} }  =  \sum_{k=1}^{m}  \frac{\beta_k^2}{m} \left(1 + 2\frac{ v_{k,n} } 
{\beta_k \sqrt{n} d_{k} }\right).
\end{equation}

It is thus apparent that $\alpha=4$ makes the perturbation
on $\frac{\beta_k^2}{m}$ dependent on $\frac{ v_{k,n} } 
{ \beta_k \sqrt{n} d_{k} }$ that is the relative reconstruction error on $\beta_k$.
This appears a reasonable choice to account for the ill-conditioning possibly affecting least-squares. 


Interestingly, for large $n$, this same philosophy is followed by the marginal likelihood procedure
for hyperparameter estimation up to first-order approximations.
In fact, under the stated assumptions, apart from constants, the minus two log of  the marginal likelihood is 
\begin{equation}\label{MLsb}
\sum_{k=1}^{m}  \log(n^{-1}  + \lambda  d_{k,n}^{2}) + \frac{ z_{k,n}^{2}}
{ n^{-1} + \lambda  d_{k,n}^{2} } ,
\end{equation}
whose partial derivative w.r.t. $\lambda$ is
\begin{equation}\label{derMLsb}
\sum_{k=1}^{m} \frac{ d_{k,n}^{4}+n^{-1}d_{k,n}^{2}- z_{k,n}^{2}d_{k,n}^{2}}
{ (n^{-1} + \lambda  d_{k,n}^{2})^2 } .
\end{equation}
As before, we consider small perturbations around $z_k=d_k\beta_k$ to find that a
critical point occurs at
\begin{equation}\label{NoisylambdaML}
\sum_{k=1}^{m}  \frac{\beta_k^2}{m} \left(1 + 2\frac{ v_{k,n} } 
{\beta_k \sqrt{n} d_{k} }\right),
\end{equation}
which is exactly the same minimizer reported in (\ref{NoisylambdaW4}).



\bigskip

\section{Three variants of HGLasso and their implementation}\label{Impl}

In this section we discuss the implementation of our HGLasso
approach. In particular, the results introduced in the previous section point out
some distinctive features of HGLasso with respect
to GLasso (MKL). In fact, we have shown that HGLasso relies upon an estimator
for the hyperparameter $\lambda$ having some favorable properties in terms of MSE minimization.
In addition, sparsity can be induced using a smaller value for $\gamma$
than that needed by GLasso. This is an important point since we have seen that
nice MSE properties are obtained optimizing the marginal posterior of $\lambda$
with $\gamma$ close to zero.

On the other hand, a drawback of the HGLasso is that it requires the solution to a
non-convex optimization problem in a possibly high-dimensional space.
We show that this problem can be faced by introducing
a variant of HGLasso where only one scalar variable
is involved in the optimization process. In addition, this procedure is able to promote
greater sparsity than the full version of HGLasso while continuing to accurately reconstruct
the nonzero blocks.
This new computational scheme, which we call {\bf HGLa},
relies on the combination of marginal likelihood optimization and Bayesian forward selection
equipped with cross validation to select $\gamma$. It is introduced in the following subsections
together with two other versions that will be called {\bf HGLb} and {\bf HGLc}.
The following two subsections are instrumental to the introduction of the
three algorithms.

\subsection{Bayesian Forward Selection}\label{FB}

In this section we introduce a forward-selection procedure which will be useful to define
the computationally efficient version of the HGLasso estimator.
Hereafter, we use $y_{tr}$ and $y_{val}$ to indicate the output data contained
in the training and validation data set, respectively. This also induces
a natural partition of $G$ into $G_{tr}$ and $G_{val}$.
In order to obtain an estimator of $\lambda$ we consider
the constraint $\kappa=\lambda_1=\lambda_2=\ldots=\lambda_p$ and
treat $\kappa$ as a deterministic hyperparameter whose knowledge
makes the covariance $\Sigma_{y_{tr}}$ of the data $y_{tr}$
completely known.
Therefore we set:
 \begin{equation}\label{AlternativeGlambda_constrained}
 \begin{array}{rcl}
\hat\kappa&:=&\displaystyle{\arg \min_{\kappa \in \R_+}  \frac{1}{2}
\log \det(\Sigma_{y_{tr}}) + \frac{1}{2} y_{tr}^\top \Sigma_{y_{tr}}^{-1} y_{tr}}
\end{array}
\end{equation}
Now, we consider again the Bayesian model in Fig. \ref{BN}(b),
where all the components of $\lambda$ are fixed to $\hat{\kappa}$
while $\gamma$ may vary on a grid $C$ built around $\hat{\kappa}^{-1}$.
The forward-selection procedure is then designed as follows;
for each value of $\gamma$ in the grid $C$ let
$I\subseteq  \{1,2,..,p\}$ be the subset of currently selected
groups and, using the Bayesian model in Fig. \ref{BN}(b) (see also (\ref{joint pdf 2})-(\ref{HLtheta})),
define the marginal log posterior
\begin{equation}\label{LogPosterior}
L(I,\kappa,\gamma):={\rm log}\left[p_\gamma(\tilde\lambda_I|y_{tr})\right]
\end{equation}
$\tilde\lambda_I:=[\tilde\lambda_{I,1},...,\tilde\lambda_{I,p}]$ and $\tilde\lambda_{I,i}=\hat{\kappa}$ if $i\in I$ and $\tilde\lambda_{I,i}=0$ otherwise.
Then, for each value of $\gamma$ in the grid $C$, perform the following procedure:
\begin{itemize}
\item
initialize $I(\gam):=\emptyset$
\item repeat the following procedure:
    \begin{itemize}
    \item[(a)] for $j \in \{1,..,p\}\setminus I(\gam)$, define $I_j^{'}(\gam):=I(\gam)\cup {j}$
    and compute $L(I_j^{'}(\gam);\hat{\kappa},\gamma)$.
    \item[(b)] select $$\bar j:=
    \mathop{arg\;max}_{j \in \{1,..,p\}\setminus I(\gam)}
    L(I_j^{'}(\gam);\hat{\kappa},\gamma)-L(I(\gam);\hat{\kappa},\gamma)$$
        \item[(c)] \emph{if} $L(I_{\bar j}^{'}(\gam);\hat{\kappa},\gamma)-L(I(\gam);\hat{\kappa},\gamma)>0$\\
        set $I(\gam):=I^{'}_{\bar j}(\gam)$ and go back to (a)\\
        \emph{else}\\
         finish.
    \end{itemize}
\end{itemize}
Note that, for each $\gam$ in the grid, the set $I(\gam)$
contains the indexes of selected variables different from zero.
Let $\hat\gamma$ denote the value of $\gam$ from the grid that
yields the best prediction on the validation data set $y_{val}$, i.e.
$$
\hat\gamma = \arg \min_{\gamma \in C} \| y_{val} -G_{val} {\theta}_{HGL}(\tilde\lambda_{I(\gamma)})     \|
$$
 and set
$I_{FS}=I(\hat\gam)$.

\subsection{Projected Quasi-Newton Method}

The objective in (\ref{AlternativeGlambda}) is a differentiable function
of $\lambda$. The computation of its derivative requires a one time
evaluation of the matrices $G^{(i)}G^{(i)^\top},\ i=1,\dots,p$. However,
for each new value of $\lam$, the inverse of the matrix $\Sigma_y(\lambda)$ also needs to be computed.
Hence, the evaluation
of the objective and its derivative
may be costly since it requires computing the
inverse of a possibly large matrix as well as large matrix products.
On the other hand, the dimension of the parameter vector $\lambda$
can be small, and projection onto the feasible set is trivial.

We experimented with several methods available in the Matlab package \verb{minConf{ to optimize (\ref{AlternativeGlambda}).
In these experiments, the fastest method was
the limited memory projected quasi-Newton algorithm
detailed in \citep{Schmidt09optimizingcostly}.
It uses L-BFGS updates to build a diagonal plus low-rank quadratic approximation to the function,
and then uses the Projected Quasi-Newton Method to minimize a quadratic approximation subject to the original constraints to obtain a search direction.
A backtracking line search is applied to this direction terminating at a step-size satisfying a
Armijo-like sufficient decrease condition.
The efficiency of the method derives in part from the simplicity of the projections onto the
feasible region. We have also implemented 
 the re-weighted method described in \citep{Wipf_ARD_NIPS_2007}.
 In all the numerical experiments described below, we have assessed that
 it returns results virtually identical to those achieved by our method,
with a similar computational effort.
It is worth recalling that both the projected quasi-Newton method and
the re-weighted approach guarantee only converge to a 
stationary point of the objective.

\subsection{The three variants of HGLasso}
\label{3variants}

We consider the three version of HGLasso.

\begin{itemize}
\item {\bf HGLa}: Output data $y$ are split in a training and validation data set.
The optimization problem (\ref{AlternativeGlambda_constrained})
is solved using only the training data obtaining $\hat{\kappa}$. The regularization parameter
$\gamma$ is estimated using the
forward-selection procedure described in the previous subsection equipped
with cross-validation. This procedure also returns the set $I_{FS}$
containing the indexes of the selected variables different from zeros.
This index set gives an estimate
$\hat{\lambda}_{FS}$ of the hyperparameter vector, whose components
are equal to $\hat{\kappa}$ for $i\in I_{FS}$, and zero otherwise . Finally,
$\hat{\theta}_{HGL}$ is estimated by the formula given in (\ref{AlternativeGteta}),
$\E(\theta\,|\,y,\,\hat\lam_{FS})
=\mbox{blockdiag}((\hat\lam_{FS})_i I_{k_i}) G^\top \Sig_y(\hat\lam_{FS})^{-1}y$,
using all the available data, i.e. the union of the training and validation
data sets.

\item {\bf HGLb}: 
The optimization problem (\ref{AlternativeGlambda})
is solved using the Projected Quasi-Newton method  with starting point
defined by the $\hat{\lambda}_{FS}$ returned by  {\bf HGLa}.
The regularization parameter $\gamma$ is set to the estimate
obtained by {\bf HGLa}, $\hat\gam$.
Once the new estimate of $\lambda$ is obtained,
$\hat{\theta}_{HGL}$ is computed using (\ref{AlternativeGteta}).

\item {\bf HGLc}: This estimator performs
the same operations as {\bf HGLb} except that
the components of $\lambda$ set to zero by  {\bf HGLa}
are kept at zero, i.e. $\lam_i=0,\ i\notin I_{FS}$.
In addition, the regularization parameter $\gamma$
is set to zero in order to obtain the 
MSE properties in the reconstruction of the
blocks different from zero established in
Proposition \ref{CommonConvergence}. Hence, the problem
(\ref{AlternativeGlambda}) is optimized with $\gam=0$ and
only over those $\lam_i$ for $i\in I_{FS}$.
\end{itemize}


\section{Numerical experiments}\label{sec:sim}

\subsection{Simulated data}\label{sec:sim1}

\begin{figure*}
  \begin{center}
    \includegraphics[width=0.9\linewidth,angle=0]{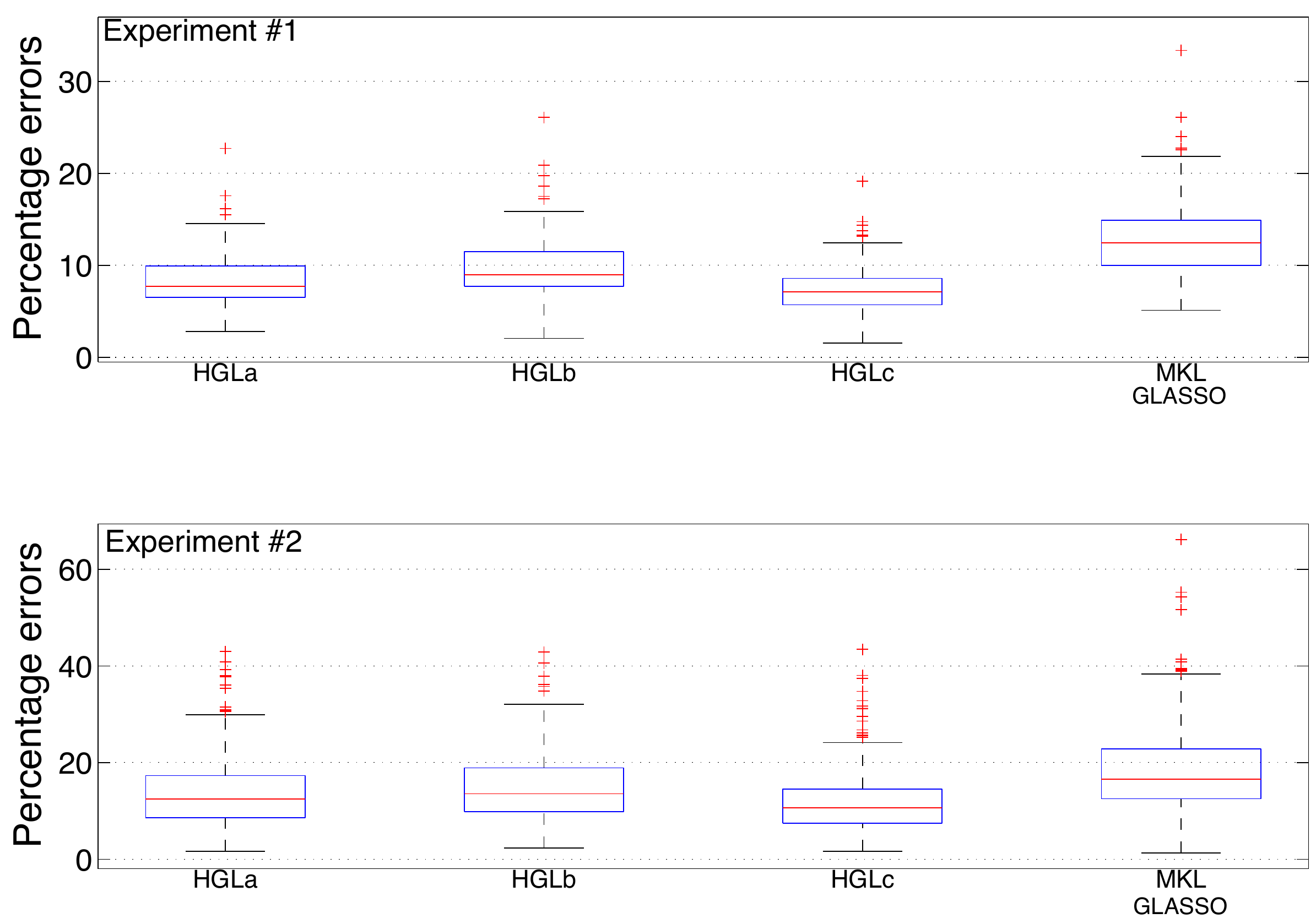}
\caption{Comparison with MKL/GLasso (section \ref{sec:sim1}). Boxplot of the percentage errors in the reconstruction of $\theta$ (top)
obtained by the 4 estimators after the 300 Monte Carlo runs in Experiment $\#1$ (top panel)
and $\#2$ (bottom panel).}
    \label{Fig1}
  \end{center}
\end{figure*}

We consider two Monte Carlo studies of $300$ runs each 
on the linear model \eqref{GroupMeasmod} 
with $p=10$ groups, each composed of $k_i=4$ parameters, and $n=100$.
For each run,  $5$ of the $\theta^{(i)}$ groups are set to zero, one is always taken different from zero
while each of the remaining $4$ $\theta^{(i)}$ groups are set to zero with probability $0.5$.
The components
of every $\theta^{(i)}$ block not set to zero are independent realizations from a uniform distribution on
$[-a_i,a_i]$ where $a_i$
is an independent realization (one for each block) from a uniform distribution on $[0,100]$.
The value of $\sigma^2$ is equal to the variance of the noiseless output divided by 25.
The noise variance is assumed unknown and its estimate is determined at each run
as the sum of the residuals coming from the least squares estimate
divided by $n-m$.
The two experiments differ in the way the columns of $G$ are
generated at each run. In the first experiment, the entries of
$G$ are independent realizations of zero mean unit variance Gaussian noise.
In the second experiment the columns of $G$ are correlated, being defined at every run by
\begin{eqnarray*}
&& G_{i,j} = G_{i,j-1} + 0.2 v_{i,j-1}, \quad i=1,..,n, \quad j=2,..,m\\
&& v_{i,j}\sim {\cal N}(0,1)
\end{eqnarray*}
where $v_{i,j}$ are i.i.d. (as $i$ and $j$ vary) zero mean unit variance Gaussian and $G_{i,1}$ are
i.i.d. zero mean unit variance Gaussian random variables. Note that correlated inputs  renders the estimation problem more challenging.\\

We compare the performance of the following $4$
estimators.
\begin{itemize}
\item {\bf HGLa,HGLb,HGLc}: These are the three variants of our HGLasso procedure defined at the end of Section \ref{Impl}. The data is split into training and validation data sets of equal size and the grid $C$
used by the cross validation procedure to select $\gamma$ contains 30 elements
 logarithmically distributed between $10^{-2} \times \hat{\kappa}^{-1}$ and $10^{4} \times \hat{\kappa}^{-1}$

    \item {\bf MKL (GLasso)}: The regularization parameter is determined via cross validation, splitting the data set in two segments of the same size and testing a finite number of parameters from a grid
    with $30$ elements logarithmically distributed between $10^{-2} \times \hat{\gamma}$ and $10^{4} \times \hat{\gamma}$
    where $\hat{\gamma}$ is the regularization parameter adopted by the three HLasso procedures.
    Finally, MKL (GLasso) is reapplied to the full data set fixing the regularization parameter
    to its estimate.
\end{itemize}
The $4$ estimators are compared using the two performance indexes listed below:
\begin{enumerate}
\item Percentage estimation error: this is computed at each run as
\begin{equation}\label{Err1}
100\times \frac{\|\theta-\hat{\theta}\|}{\|\theta\|}\,\%
\end{equation}
where $\hat{\theta}$ is the estimate of $\theta$.
\item Percentage of the blocks equal to zero correctly set to zero by the estimator
after the $300$ runs.
\end{enumerate}
The top and bottom panel of Fig. \ref{Fig1} displays the boxplots of the 300 percentage errors
obtained by the 4 estimators in the first and second experiment, respectively.
It is apparent that all of the three versions of HGLasso outperform
GLasso.

In Table \ref{table1} we report
the sparsity index. One can see that in the first and second experiment
the first and third version of HGLasso obtain the remarkable performance of
around $99\%$ of blocks correctly set to zero, while the second version obtains
a value close to $76\%$.
Instead, in the two experiments GLasso (MKL) correctly set to zero no more than $40\%$  of the blocks.
This result, which can appear surprising, is explained by the arguments in Sections \ref{SparsvsSh}
and \ref{sec:MSE};
in a nutshell,  GLasso trades sparsity for shrinkage.
The value of the regularization parameter $\gamma$ needed  to avoid oversmoothing is not sufficiently large to induce ``enough''  sparsity.
This drawback does not affect our new nonconvex estimators. These estimators have the additional
advantage of selecting the regularization parameters leading to more favorable MSE properties
for the reconstruction of the non zero blocks,
as discussed in Section \ref{sec:MSE} and illustrated in Section \ref{SparsvsSh} in a simplified scenario.
\begin{table}
\begin{center}
\begin{tabular}{ccccc}\hline
& HGLa & HGLb & HGLc  & MKL (GLasso)
\\ \hline
\mbox{Experiment $\#1$} & $99.2\%$  &  $76.1\%$  &  $99.2\% $ & $36.1 \%$  \\
\mbox{Experiment $\#2$} & $99.0\%$  &  $76.5\%$  &  $99.0\% $ & $39.5 \%$ \\
\hline \phantom{|}
\end{tabular}
\end{center}
\caption{Comparison with MKL/GLasso (section \ref{sec:sim1}). Percentage of the $\theta^{(i)}$ equal to zero correctly set to zero by the four estimators.} \label{table1}
\end{table}

\subsubsection{Testing a variant of HGLa}\label{subsec:sim1}

To better point out the role played by $\gamma$ in our numerical schemes, we have also 
considered a variant of HGLa where $\gamma$ is always set to 0 and the parameter $\sigma^2$ is used to induce sparsity. More precisely, the only difference with respect to HGLa is that, after obtaining $\sigma^2$ from least squares and determining $\hat{\kappa}$, $\sigma^2$ is re-estimated
using the forward-selection procedure equipped with cross-validation.\\
For the sake of comparison, we have considered 3 Monte Carlo studies 
of 300 runs. Data are generated as in the second experiment described above
except that the 3 cases exploit different values of $\sigma^2$ equal to the noiseless output
variance divided by $5$ (case a), $2$ (case b) or $1$ (case c). Table \ref{table2} reports the mean of the 300 percentage
errors while Table \ref{table3} reports the sparsity index.\\
It is apparent that the variant of HGLa performs quite well, but HGLa outperforms it in
all the experiments.
These results can be given the following interpretation. When one adopts HGLa, 
the "high level" of the Bayesian network depicted in Fig. \ref{BN} (b) (represented by $\gamma$) 
is used to to induce sparsity with the parameters entering  
the lower level of the Bayesian network that need not to be changed.
Thus, $\theta$ is eventually reconstructed adopting the $\sigma^2$
estimated from data and the $\lambda$ determined by marginal likelihood optimization, thus possibly  
exploiting the MSE properties reported in Proposition \ref{CommonConvergence}.  
When $\gamma$ is instead set to $0$, the estimator has to trade sparsity and shrinkage 
using a less flexible structure. In particular, the lower level of the Bayesian network is now also in charge 
of enforcing sparsity and this can be done only increasing $\sigma^2$, possibly loosing performance in terms of MSE. 

\begin{table}
\begin{center}
\begin{tabular}{ccc}\hline
& HGLa & Variant of HGLa 
\\ \hline
\mbox{Experiment $\#2$ (case a)} & $12.2\%$  &  $15.6\%$  \\
\mbox{Experiment $\#2$ (case b)} & $30.1\%$  &  $39.2\%$  \\
\mbox{Experiment $\#2$ (case c)} & $61.5\%$  &  $73.8\%$  \\
\hline \phantom{|}
\end{tabular}
\end{center}
\caption{Comparison with the variant of HGLa (section \ref{subsec:sim1}). Mean of the percentage errors in the reconstruction of $\theta$ obtained
by HGLa and by the variant of HGLa where sparsity is induced by $\sigma^2$.} \label{table2}
\end{table}
\begin{table}
\begin{center}
\begin{tabular}{ccc}\hline
& HGLa & Variant of HGLa 
\\ \hline
\mbox{Experiment $\#2$ (case a)} & $99.2\%$  &  $93.1\%$  \\
\mbox{Experiment $\#2$ (case b)} & $96.2\%$  &  $86.5\%$  \\
\mbox{Experiment $\#2$ (case c)} & $88.1\%$  &  $71.4\%$  \\
\hline \phantom{|}
\end{tabular}
\end{center}
\caption{Comparison with the variant of HGLa (section \ref{subsec:sim1}). Percentage of the $\theta^{(i)}$ equal to zero correctly set to zero
by HGLa and by the variant of HGLa where sparsity is induced by $\sigma^2$.} \label{table3}
\end{table}

\subsection{Comparison with Adaptive Lasso}\label{sec:ADA}

In this section we compare the performance of HGLa 
with that obtainable by the Adaptive Lasso (AdaLasso) procedure introduced in \citep{AdaptiveLasso}.
In particular, we consider an example taken from \citep{AdaptiveLasso} where the components of 
$\theta$ are $\{3,1.5,0,0,2,0,0,0\}$ and each component represents a group (block size  is equal to 1). The rows of the design matrix $G$ are independent realizations from a zero mean Gaussian vector, with $(i,j)$-entry of its covariance
equal to $\beta^{|i-j|}$. To be more specific, we consider 6 Monte Carlo studies, each of 200 runs, 
where at each run $\beta$ is drawn uniformly from the open interval $(0.5,1)$. The 6 experiments
then differ in the number of  data used to reconstruct $\theta$ (20 or 60) and in the variance of the Gaussian measurement noise ($\sigma^2$=1, 9 or 16). We implemented HGLa and Lasso as 
described in the previous subsection. For what regards AdaLasso, as in \citep{AdaptiveLasso} we exploited two-dimensional cross validation to estimate the regularization parameter
and the variable $\eta$ defining the weights. In particular, the latter were set to the 
inverse of the absolute value of the least squares estimates 
raised at $\eta$, where $\eta$ may vary on the grid $[0.5,1,\ldots,4]$.\\
Results are summarized in Table \ref{table4}, that reports the mean of the 200 percentage errors, and in Table \ref{table5}, where the sparsity index is displayed.
One can notice that HGLa outperforms Lasso and AdaLasso\footnote{In this experiment AdaLasso
enforces more sparsity than Lasso but leads to larger reconstruction errors on $\theta$
since it tends more frequently to set to
zero also components of $\theta$ that are not null.}, 
achieving both a smaller reconstruction error and a better sparsity index. 
To further illustrate this fact, at each Monte Carlo run
we have also computed the Euclidean norm of the estimates of the null components of $\theta$
returned by the three estimators, divided by the norm of the true $\theta$. Since the number 
of null components of $\theta$ is 5 and the overall number of Monte Carlo runs is 1200, 
6000 values were stored. Fig. \ref{FigAda} plots them (as a function of the Monte Carlo run) 
as points when the estimated value is different from zero (no point is displayed if the
corresponding value is zero). It is apparent that  
in this example HGLa correctly detects the null components
of $\theta$ more frequently than Lasso and AdaLasso, also providing a smaller 
reconstruction error when a component is not set to zero.

\begin{table}
\begin{center}
\begin{tabular}{cccc}\hline
& HGLa & Lasso & AdaLasso 
\\ \hline
\mbox{Exp. $\#1$ (n=20, $\sigma^2=1$)} & $30.2\%$  &  $34.5\%$  &  $38.2\%$  \\
\mbox{Exp. $\#2$  (n=60, $\sigma^2=1$)} & $12.1\%$  &  $15.3\%$  &  $17.1\%$ \\
\mbox{Exp. $\#3$  (n=20, $\sigma^2=9$)} & $68.5\%$  &  $81.2\%$  &  $100.1\%$ \\
\mbox{Exp. $\#4$ (n=60, $\sigma^2=9$)} & $43.1\%$  &  $46.3\%$  &  $56.6\%$ \\
\mbox{Exp. $\#5$  (n=20, $\sigma^2=16$)} & $78.7\%$  &  $110.1\%$  &  $141.2\%$ \\
\mbox{Exp. $\#6$  (n=60, $\sigma^2=16$)} & $53.7\%$  &  $60.0\%$  &  $73.6\%$ \\
\hline \phantom{|}
\end{tabular}
\end{center}
\caption{Comparison with AdaLasso (section \ref{sec:ADA}). Mean of the percentage errors in the reconstruction of $\theta$ obtained
by the five estimators.} \label{table4}
\end{table}





\begin{table}
\begin{center}
\begin{tabular}{cccc}\hline
& HGLa & Lasso & AdaLasso 
\\ \hline
\mbox{Exp. $\#1$ (n=20, $\sigma^2=1$)} & $87.9\%$  &  $38.5\%$  &  $69.1\%$  \\
\mbox{Exp. $\#2$  (n=60, $\sigma^2=1$)} & $96.5\%$  &  $45.6\%$  &  $75.2\%$ \\
\mbox{Exp. $\#3$  (n=20, $\sigma^2=9$)} & $80.9\%$  &  $49.6\%$  &  $61.0\%$ \\
\mbox{Exp. $\#4$ (n=60, $\sigma^2=9$)} & $87.4\%$  &  $46.8\%$  &  $69.4\%$ \\
\mbox{Exp. $\#5$  (n=20, $\sigma^2=16$)} & $78.4\%$  &  $51.4\%$  &  $59.3\%$ \\
\mbox{Exp. $\#6$  (n=60, $\sigma^2=16$)} & $85.7\%$  &  $41.1\%$  &  $65.5\%$ \\
\hline \phantom{|}
\end{tabular}
\end{center}
\caption{Comparison with AdaLasso (section \ref{sec:ADA}). Percentage of the components of $\theta$ equal to zero correctly set to zero by the five estimators.} \label{table5}
\end{table}

\begin{figure*}
  \begin{center}
    \includegraphics[width=0.9\linewidth,angle=0]{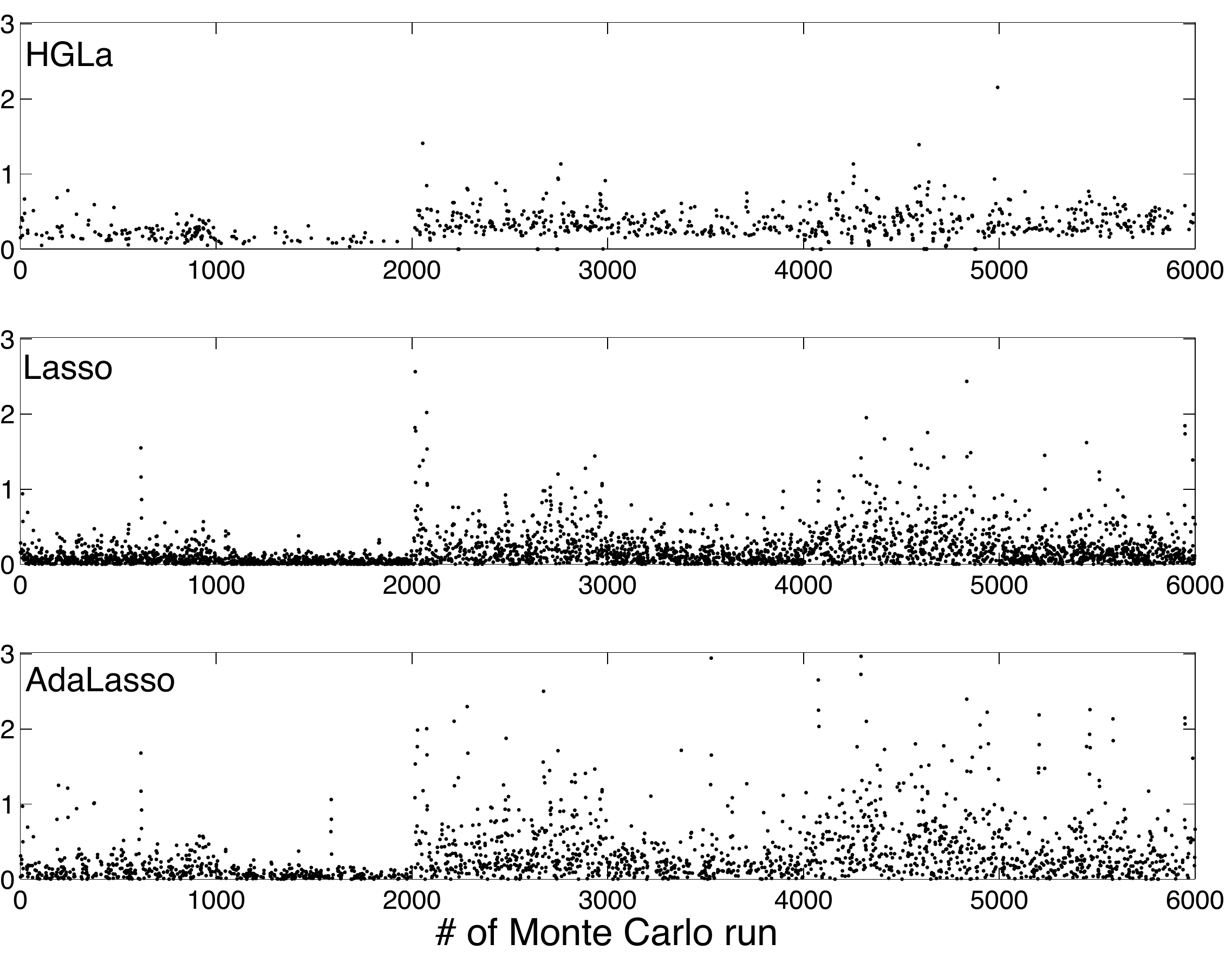}
\caption{Comparison with AdaLasso (section \ref{sec:ADA}). Euclidean norm of the estimates of the null components of $\theta$ returned by the three estimators (divided by the norm of the true $\theta$) as a function of the Monte Carlo runs performed in the 6 experiments. The values different from zero are displayed as points while no point is displayed if the obtained estimate is zero.}
    \label{FigAda}
  \end{center}
\end{figure*}

\subsection{Real data}\label{sec:sim2}

In order to test the algorithms on real data we have considered thermodynamic modeling of a small residential
building. We placed sensors in two rooms of a small
two-floor residential building of about 80 $\textrm{m}^2$ and 200
$\textrm{m}^3$; the sensors have been placed only on one floor (approximately 40 $\textrm{m}^2$ ) and their location is approximately shown in  Figure
\ref{figPianteFoto}. The larger room is the living room while the smaller is the kitchen.
\begin{figure*}
\begin{center}
\begin{tabular}{c}
\includegraphics[width=0.55\columnwidth,angle=-90]{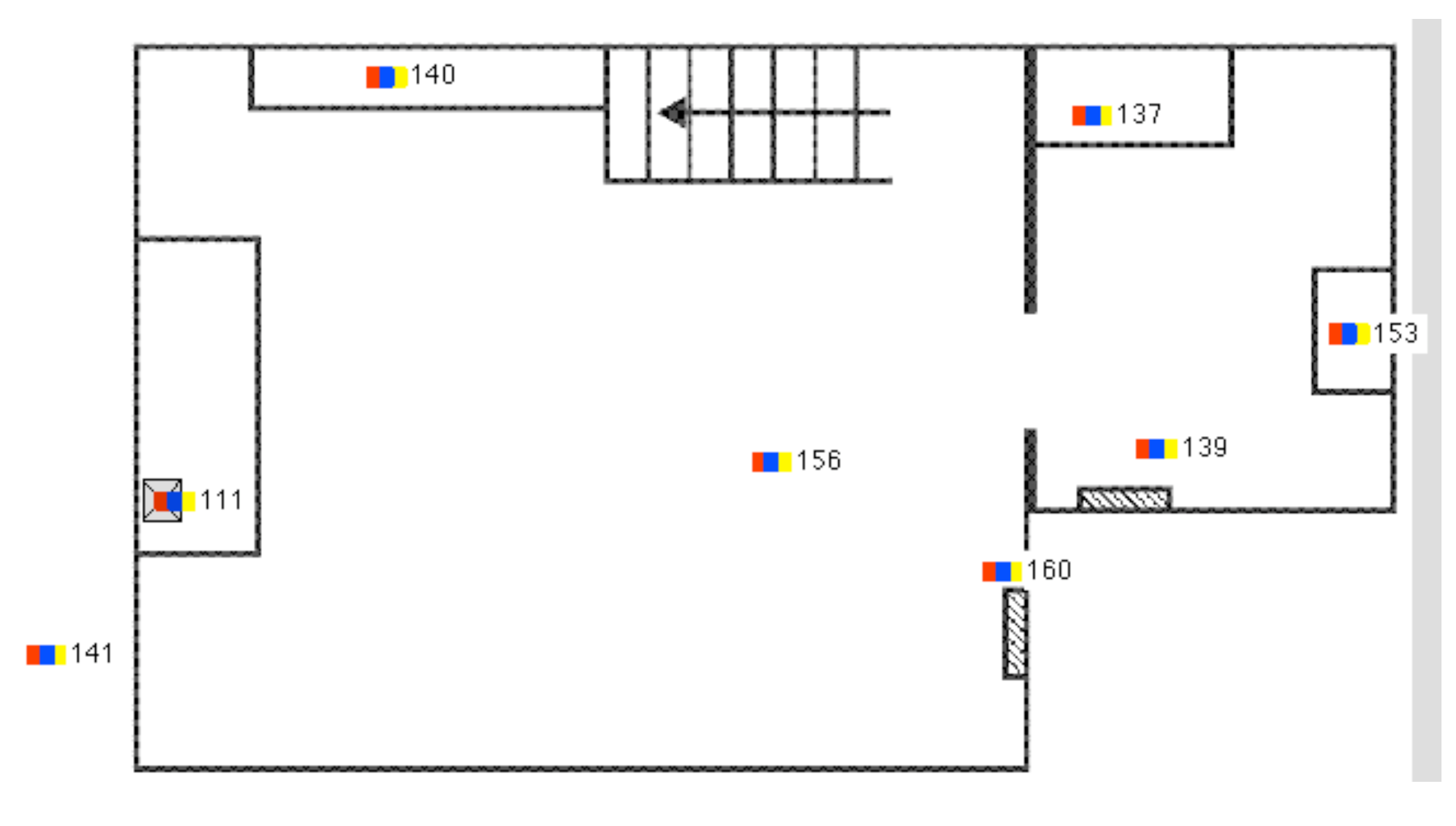}
\end{tabular}
\caption{Nodes location: $8$ nodes each equipped with $3$ sensors: temperature, humidity and total radiation.} \label{figPianteFoto}
\end{center}
\end{figure*}
The experimental
data was collected through a WSN made of 8 \emph{Tmote-Sky} nodes
produced by Moteiv Inc. Each Tmote-Sky is provided with a
temperature sensor, a humidity sensor, and a total solar radiation
photoreceptor (visible + infrared). 
 The building was inhabited
during the measurement period, which lasted for 8 days starting from February 24th, 2011; samples
were taken every 5 minutes.
The heating systems was controlled by a thermostat; the reference temperature
was manually set every day depending upon occupancy and other needs.

 The location of the sensors was as follows:

 \begin{itemize}
 \item
 Node $\#$1 (label 111 in Figure \ref{figPianteFoto}) was above a sideboard, about $1.8$ meters high, located close to thermoconvector.
 \item Node $\#$2 (label 137 in Figure \ref{figPianteFoto}) was above a cabinet (2.5 meters high).
 \item Node $\#$3 (label 139 in Figure \ref{figPianteFoto}) was above a cabinet (2.5 meters high).
  \item  Node $\#$4 (label 140 in Figure \ref{figPianteFoto}) was placed on a bookshelf (1.5 meters high).
 \item Node $\#$5 (label 141 in Figure \ref{figPianteFoto}) was placed outside.
 \item Node $\#$6 (label 153 in Figure \ref{figPianteFoto}) was placed above the stove  (2 meters high).
 \item Node $\#$7 (label 156 in Figure \ref{figPianteFoto}) was placed in the middle of the room, hanging from the ceiling (about 2 meters high).
 \item Node $\#$8 (label 160 in Figure \ref{figPianteFoto}) was placed above one radiator and was meant to provide
a proxy of water temperature in the heating systems.
\end{itemize}

This gives a total of $24$ sensors (8 temperature + 8 humidity + 8 radiation signals). A preliminary inspection of the measured signals (see Figure \ref{signals}) reveals the high level of collinearity which is well-known to complicate the estimation process 
in System Identification \citep{Soderstrom,Ljung,BoxJenkins}.

\begin{figure*}
\begin{center}
\begin{tabular}{cc}
\includegraphics[width=0.45\columnwidth,angle=0]{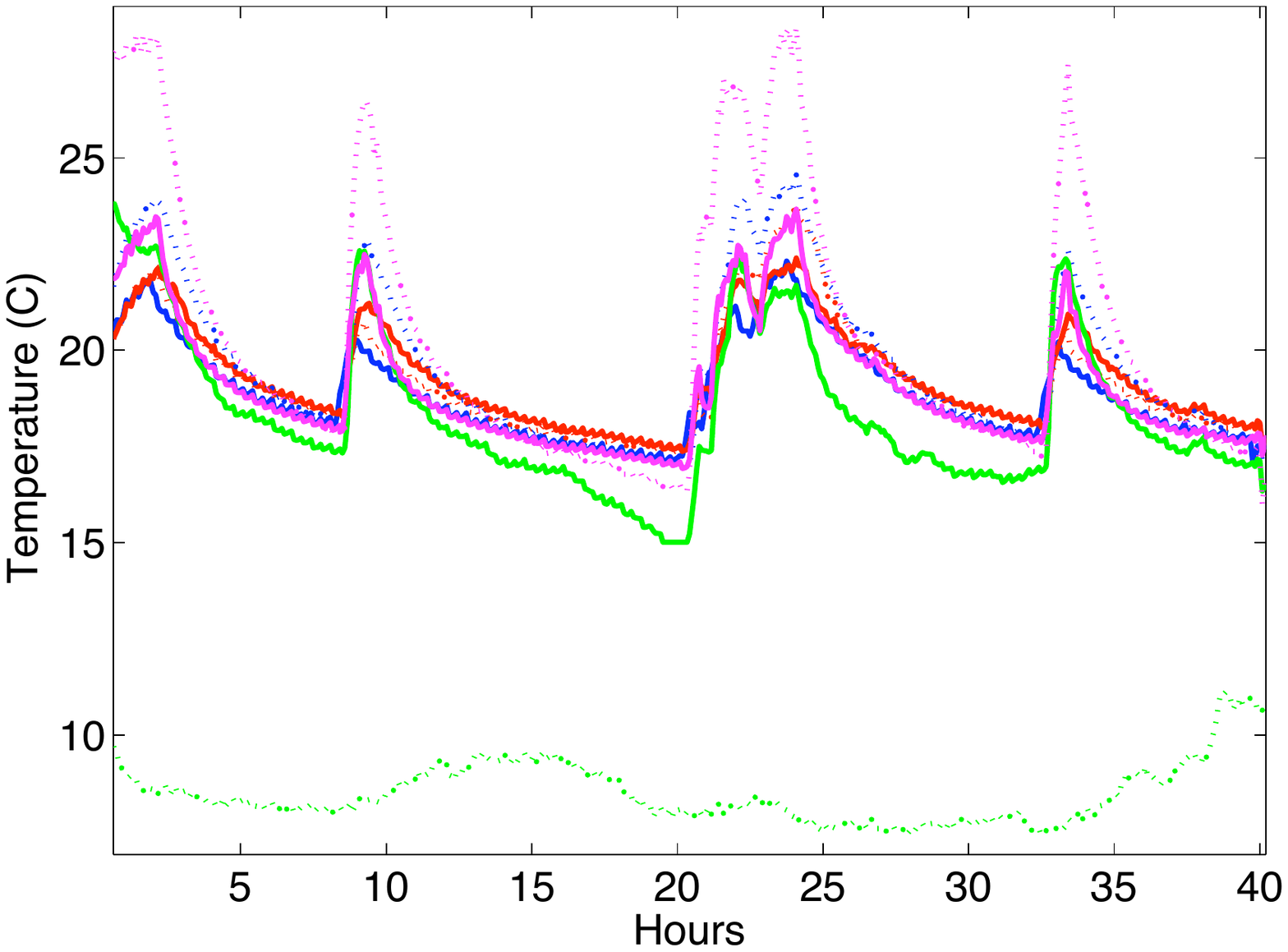} & \includegraphics[width=0.45\columnwidth,angle=0]{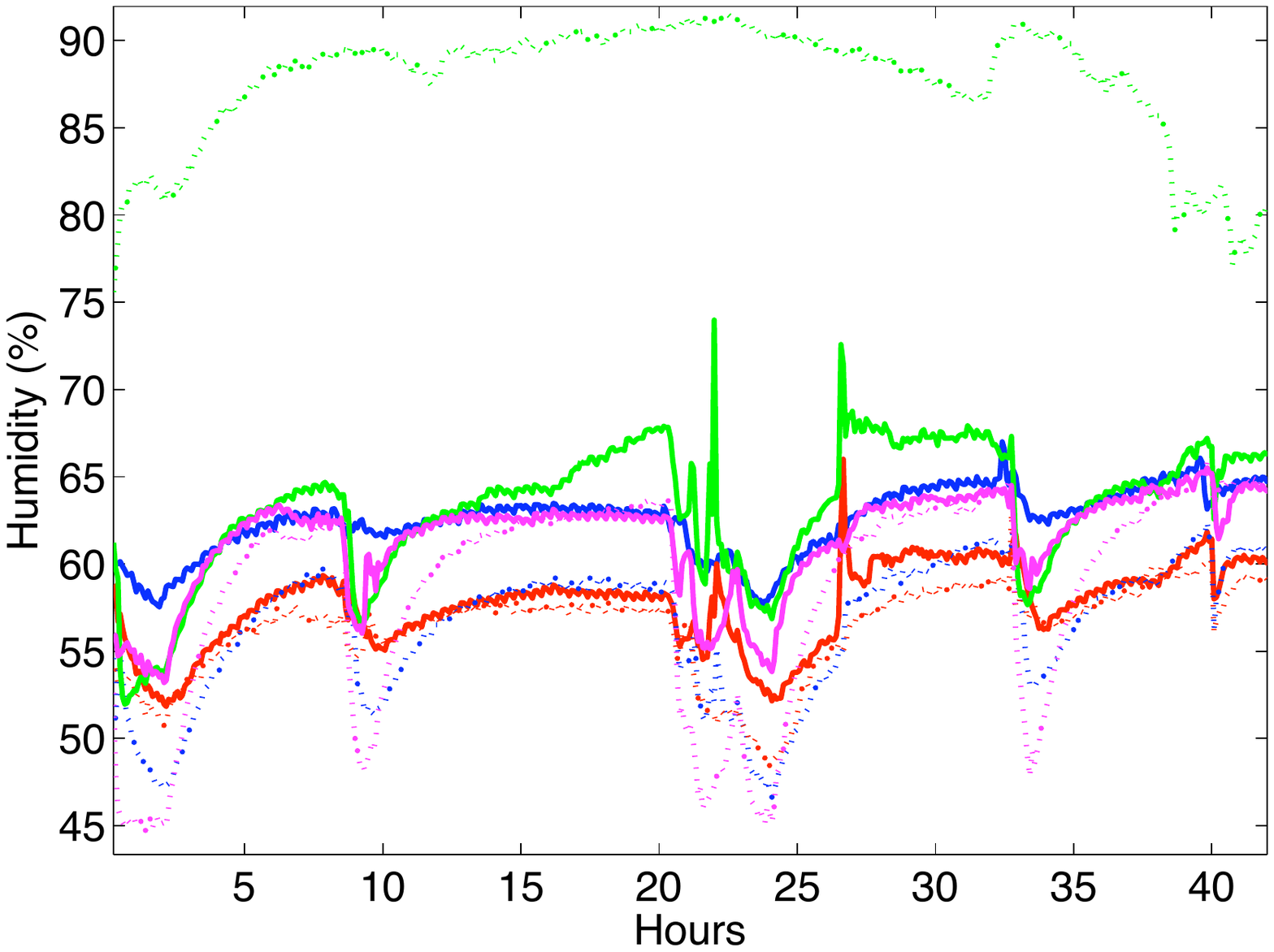}
\end{tabular}
\caption{Measured temperatures (left) and humidity (right), first 40 hours.} \label{signals}
\end{center}
\end{figure*}

We only consider
Multiple Input-Single Output (MISO) models,
with the temperature from each of the nodes as output ($y_t$) and all the other signals
(7 temperatures, 8 humidities, 8 radiations) as inputs ($u^i_t$, $i=1,..,23$
\footnote{Even though one might argue that inside radiation does not play a role,
we prefer not to embed this knowledge in order to make identification more challenging.
After all, even though our experimental setup has a small number of sensors,
a full scale monitoring system for a large building may have hundreds of sensors; in this scenario input selection is in our opinion a major issue.}).
We leave identification of a full Multiple Input-Multiple Output (MIMO) model for future investigation.  We split the available data into $2$ parts; the first, composed of $N_{id}=700$ data points, is used for
learning and validation and the second, composed of $N_{test}=1500$ data points,
is used for test purposes.
The notation $y^{id}$ identifies the
training and validation data while $y^{test}$ identifies the test data.
Note that $N_{id}=700$, with $5$ minute sampling times,
corresponds to $\simeq 58\;hours$;
this is a rather small time interval and, as such,
models based on these data cannot capture seasonal variations.
Consequently, in our experiments we assume a ``stationary'' environment and normalize
the data so as to have zero mean and unit variance before identification is performed.

Our two main goals are as follows.
\begin{enumerate}
\item Provide meaningful models with as small data set as possible. This has clear advantages if identification is being performed for, e.g., certification purposes or as a preliminary step for deciding, having monitoring/control objectives in mind, how many sensors are needed and where these should be installed.

\item  Provide sensor selection rules in order to reduce the number of sensors needed to
effectively monitor the environment.
A setup we have in mind is the following: one first deploys a large number of nodes,
collects data and performs identification experiments. As an outcome, in addition to the models,
we identify a subset of sensors that are sufficient to effectively monitor the environment;
based on the measurements from this subset of sensors one can then reliably ``predict''
the evolution of temperature (and possibly humidity) across the building.
\end{enumerate}

We envision that model predictive based methodologies, (see \citep{ModelPredictiveBookCamacho} and the recent papers \citep{Borrelli2010}, \citep{Privara2011}, \citep{Dong2008}), may be effective for these applications and, as such, we evaluate our models based on their ability to predict future data.
The predictive power of the model is measured for $k$-step-ahead prediction on \emph{validation} data, as:
\begin{equation}\label{FIT}
COD_k:=1 -{\frac{\sum_{t=k}^{N_{test}}(y^{test}_t-\hat y_{t|t-k})^2}{\sum_{t=k}^{N_{test}}(y^{test}_t-\bar y^{test})^2}}
\end{equation}
where $\bar y^{test}:=\frac{1}{N_{test}}\sum_{t=1}^{N_{test}}y^{test}_t$.

 We consider
 AutoRegressive models with eXogenous inputs (ARX) \citep{Ljung,Soderstrom,BoxJenkins} of the form
 $$
 y_t =\sum_{k=1}^{q} h_{k,1} y_{t-k}+ \sum_{i=1}^{23} \sum_{k=1}^q h_{k,i+1} u_{t-k}^i + e_t\ .
 $$
This model is \emph{linear in the parameters}
($h_{k,i}$, $k=1,..,q$, $i=1,..,24$) and as such falls within the general structure \eqref{Measmod}.
Experiments with different lengths $q$ were investigated.
We report only the results for  $q=20$ which seemed to be the most reasonable choice for all methods. Note that, with reference to \eqref{GroupMeasmod}, here we have $p=24$ groups of $k_i=q=20$ parameters each, for a total of $480$ parameters.\\
We compare the performance of the following $2$
estimators\footnote{We do not report the performance of GLasso which was similar (if not worse) than MKL, in line with the synthetic experiments in Section \ref{sec:sim1}.}:
\begin{itemize}
\item {\bf HGLc}: the variant of  HGLasso procedure defined at the end of Section \ref{Impl}. Identification data are split in a training and validation data set of equal size and the grid $C$
built around ${\hat\kappa}^{-1}$ used by the cross validation procedure to select $\gamma$
turns out to be $[ 25:25:1000]$.

        \item {\bf MKL (GLasso)}: the regularization parameter is estimated by
        cross validation using the same grid $C$ adopted for HGLc.
\end{itemize}
The $2$ estimators are compared using as performance indexes the $COD_k$ defined in equation \eqref{FIT}.
Sample trajectories of one-step-ahead prediction on both identification and test data are displayed in Figures \ref{Fig2} and \ref{Fig3} while $5$-hours ($=60$ steps) ahead prediction is shown in Figure \ref{Fig35hours}.

The $COD_k$ up to $16$ hours ahead for ARX models are plotted in Figures  \ref{Fig5} while Figure \ref{Fig6} shows the norm of the estimated impulse responses $h_{k,i}$, $i=1,..,24$.

It is clear that HGLc performs better than MKL in terms of prediction while
achieving a higher level of sparsity. Note also that the higher sparsity achieved by HGLc results in a much more stable behavior in terms of multi-step prediction (see Figures \ref{Fig35hours} and \ref{Fig5}). These results are in line with the theoretical findings in the paper as well as with the simulation results on synthetic data on Section \ref{sec:sim1}.

\begin{figure*}
  \begin{center}
    \includegraphics[width=0.7\columnwidth,angle=-90]{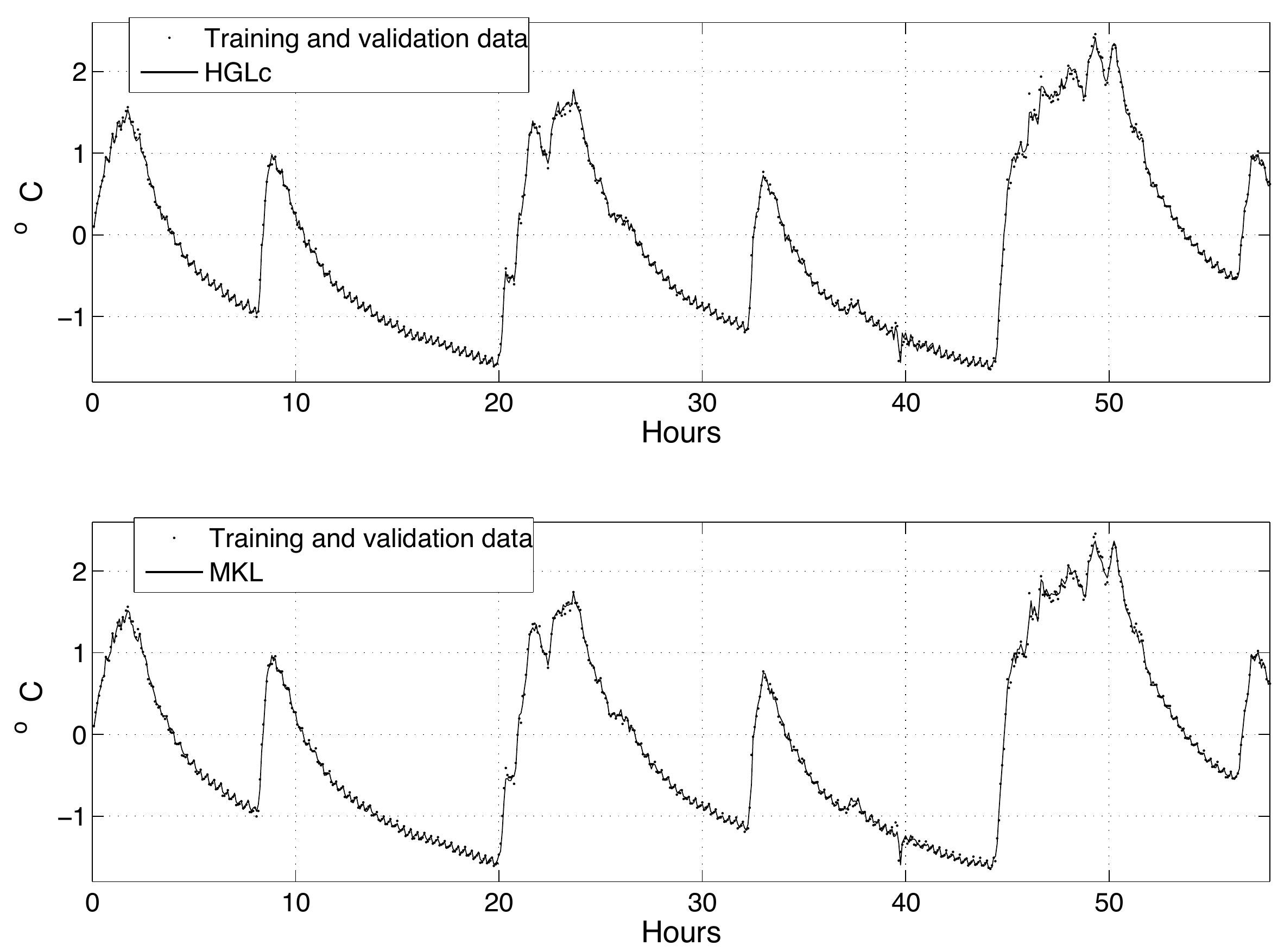}
\caption{ARX model: training and validation data ($\cdot$) with output (solid line) estimated by HGLc (top) and MKL (bottom).}
    \label{Fig2}
  \end{center}
\end{figure*}

\begin{figure*}
  \begin{center}
    \includegraphics[width=0.7\columnwidth,angle=-90]{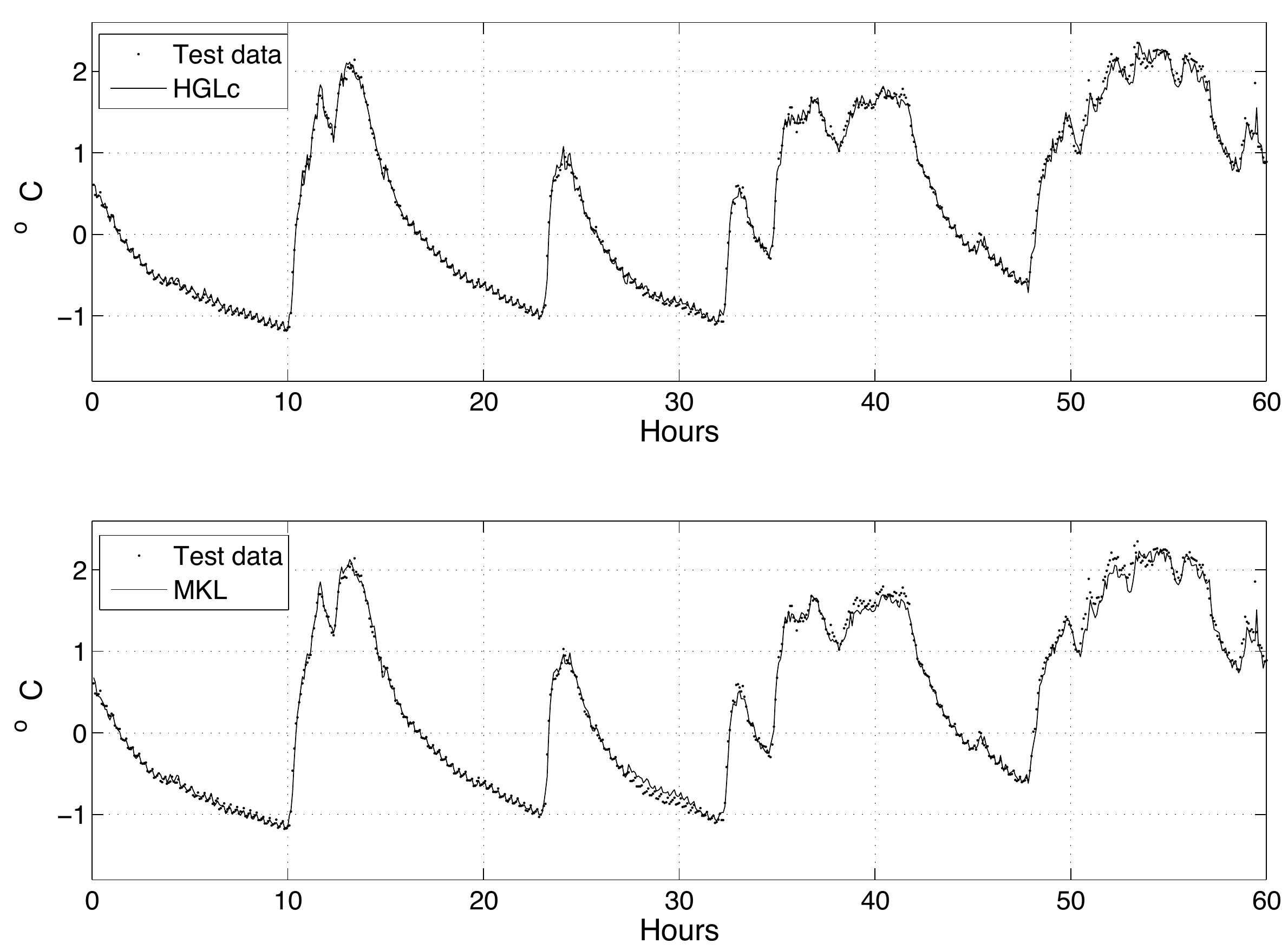}
\caption{ARX model: test data ($\cdot$) and prediction (solid line) obtained by HGLc (top) and MKL (bottom).}
    \label{Fig3}
  \end{center}
\end{figure*}
\begin{figure*}
  \begin{center}
    \includegraphics[width=0.7\columnwidth,angle=-90]{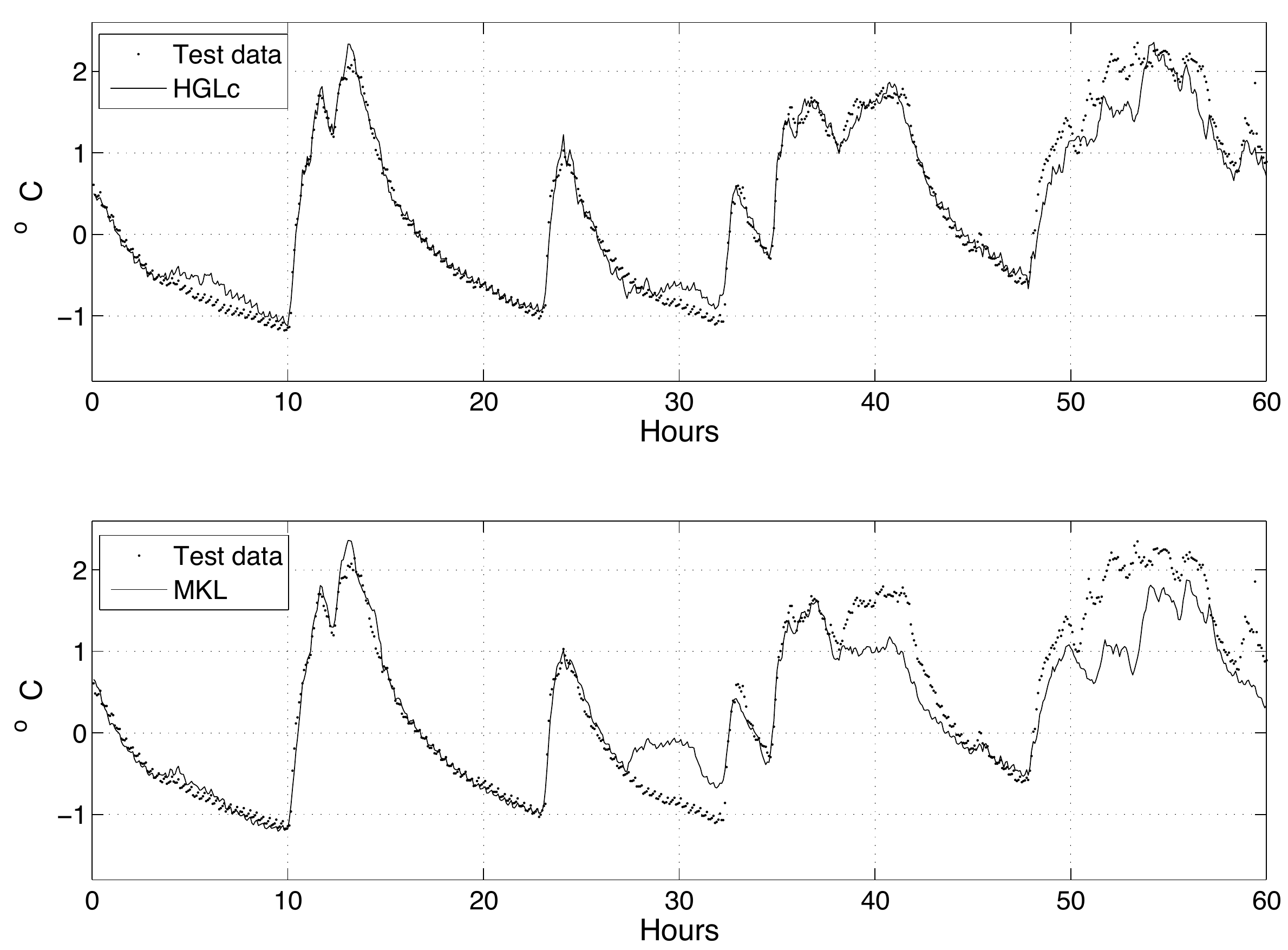}
\caption{ARX model: test data ($\cdot$) and 5-hours ahead prediction (solid line) obtained by HGLc (top) and MKL (bottom).}
    \label{Fig35hours}
  \end{center}
\end{figure*}


\begin{figure*}
  \begin{center}
    \includegraphics[width=0.5\linewidth,angle=-90]{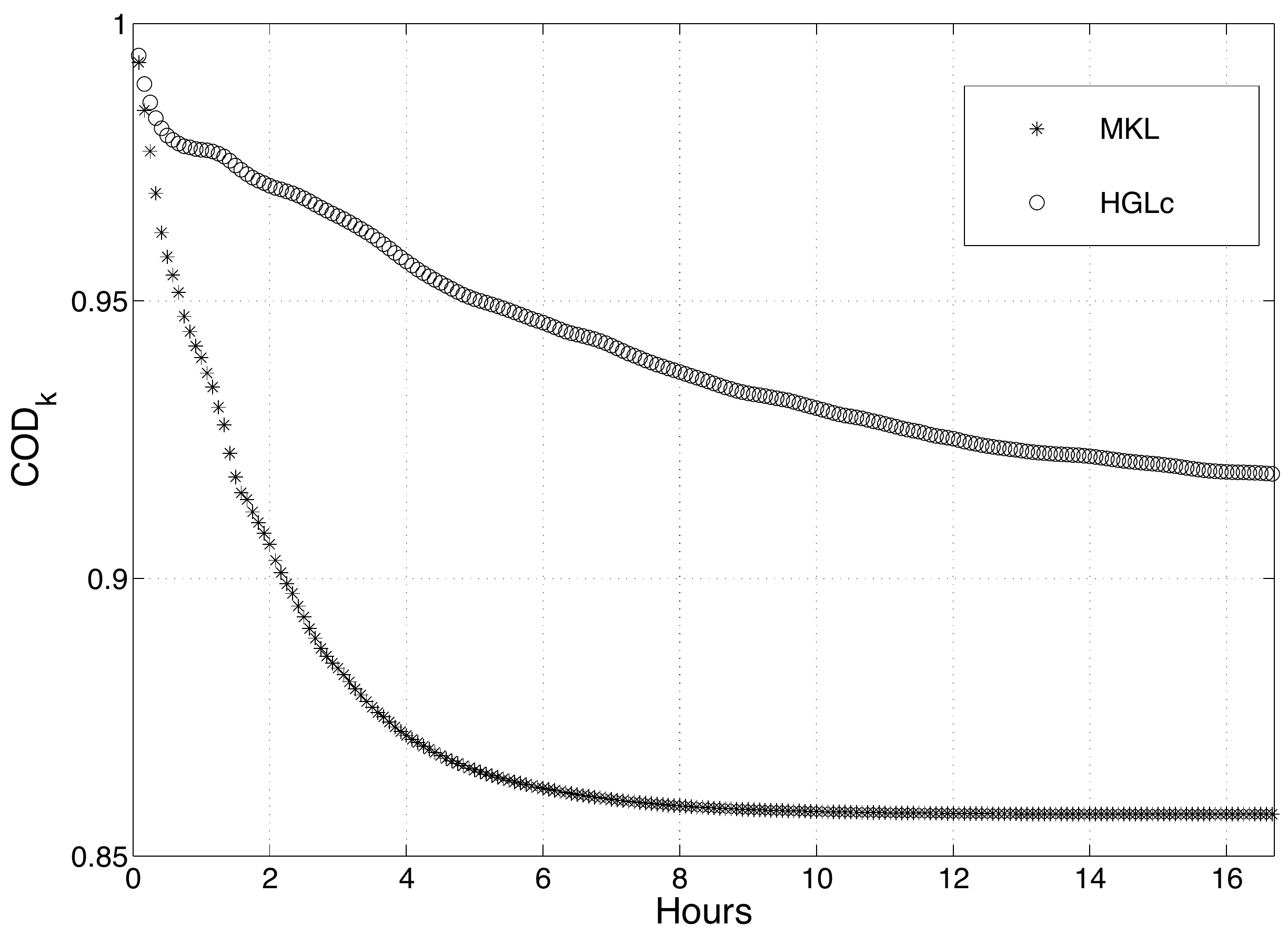}
\caption{ARX model: coefficient of determination as a function of the prediction horizon (one step = 5 minutes) using HGLc (o) and MKL ($*$).}
    \label{Fig5}
  \end{center}
\end{figure*}

\begin{figure*}
  \begin{center}
    \includegraphics[width=0.5\linewidth,angle=-90]{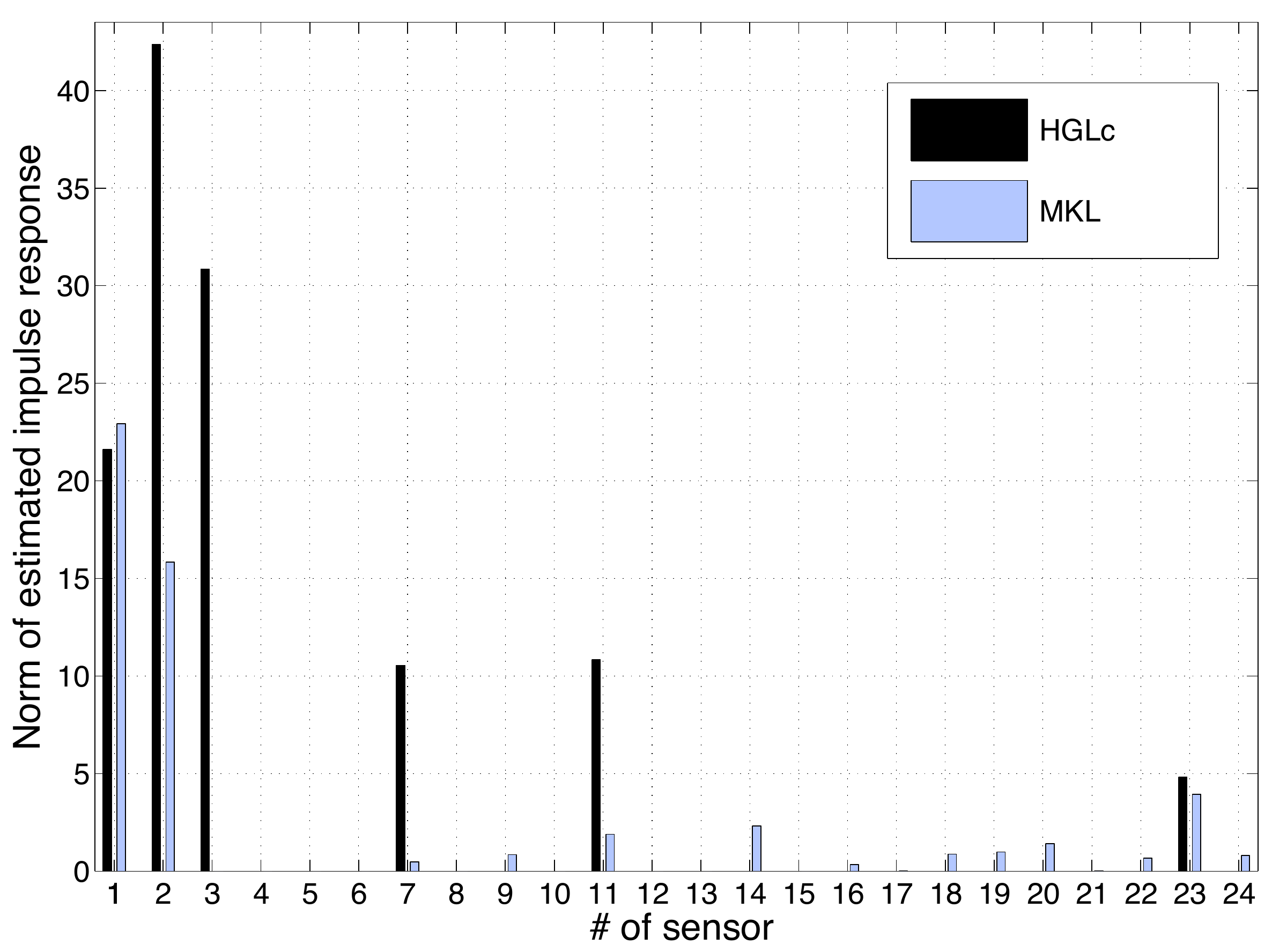}
\caption{ARX model: norm of estimated impulse responses using HGLc and MKL.}
    \label{Fig6}
  \end{center}
\end{figure*}

\section{Conclusions}
We have presented a comparative study of two methods for sparse estimation: GLasso (equivalently, MKL) and the new HGLasso. They derive from the same Bayesian model, yet in a different way.
The peculiarities of HGLasso can be summarized as follows:
\begin{itemize}
\item in comparison with GLasso, HGLasso derives from a marginalized joint density
with the resulting estimator involving optimization of a non-convex objective;
\item the non-convex nature allows HGLasso to achieve higher levels of sparsity than
GLasso without introducing too much regularization in the estimation process;

\item the MSE analysis reported in this paper reveals the superior performance
of HGLasso also in the reconstruction of the parameter
groups different from zero.
Remarkably, our analysis elucidates this issue showing the robustness
of the empirical Bayes procedure, based on marginal likelihood optimization,
independently of the correctness of the priors entering the stochastic model
underlying HGLasso. It also clarifies the asymptotic properties of ARD;
\item the non-convex nature
of HGLasso is not a limitation for its practical application.
Indeed, the Bayesian Forward Selection used in HGLa provides a highly successful initialization
procedure for the regularization parameter $\gam$ ($\hat\gam$), an initial estimate for $\lam$
(using $\hat\kappa$), and an
initial estimate of the non-zero groups ($I_{FS}$).
This procedure requires only the solution of a one dimensional version of the basic problem
(\ref{AlternativeGlambda}).
\end{itemize}
Notice also that, being included in the framework of the Type II Bayesian estimators,
many variations of HGLasso could be considered, adopting different prior models for $\lambda$. 
In this paper, the exponential prior has been used since
the goal was the comparison of different estimators that can be derived from the same Bayesian model
underlying GLasso.  In this way, it has been also shown how, starting from 
the same stochastic framework, an estimator derived from a suitable posterior marginalization can 
have signiÞcant advantages over another one derived from posterior optimization.\\
All theoretical findings have been confirmed by experiments involving real and simulated
data, also comparing the performance of the new approach with adaptive lasso. 
The aforementioned version of HGLasso has been able to promote sparsity
correctly detecting a high percentage (in some experiments also equal to $99\%$) of the null blocks of the parameter vector and to provide accurate estimates of the non-null blocks. \\

\section{Appendix}

\subsection{Proof of Proposition \ref{MKLBayes}}

Given $\gam\ge 0$, 
the maximum a posteriori estimate for $(\phi,\lam)$ given $y$ is obtained by
solving the problem
\begin{equation}\label{joint}
\min_{\phi\in\R^m,\lambda \in \R_+^{p}}
\frac{(y-G\Lam^{1/2}\phi)^{\top}(y-G\Lam^{1/2}\phi)}{2\sigma^2} +
\frac{\phi^{\top}  \phi}{2}
+  \gamma \one^\top \lambda\ .
\end{equation}
Minimizing first in $\phi$ allows us to write $\phi$ as the following function of $\lam$:
\begin{equation}\label{phimarg}
\phi(\lam):=(\sig^2I+\Lam^{1/2}G^\top G\Lam^{1/2})^{-1}\Lam^{1/2}G^\top y
=\Lam^{1/2}G^\top(\sig^2I+K(\lam))^{-1}y,
\end{equation}
where the second expression follows from the Matrix Inversion Lemma.
Substituting this back into (\ref{joint}) yields the optimization problem
(\ref{redMarg}). Hence, if $\hat\lam$ is as defined in (\ref{redMarg}), then (\ref{optPhi})
follows from (\ref{phimarg}).

Now, we show that the pair $(\hat c,\hat\lam)$ described by
(\ref{redMarg}) and (\ref{optPhi}) solves
(\ref{MKL2}) for some value of $\gam\ge 0$. For this, we need only show that the pair
$(\hat c,\hat\lam)$ is a local solution to (\ref{MKL2}) for some $\gam\ge 0$.
To this end, observe that (\ref{MKL}) is coercive in $f$ (the objective goes to $\infty$ 
as the norm of $f$ goes to $\infty$) and $\lam$ is constrained to stay in a compact set, so that a solution exists. Consequently a solution to (\ref{MKL2}) exists. Let $\gam\ge 0$ be the Lagrange multiplier associated with a local solution $(c^*,\lam^*)$ to (\ref{MKL2}) ($\gam$ exists since the constraint
is linear). We show that $(c^*,\lam^*)=(\hat c,\hat\lam)$. Since $\gam$ is the Lagrange multiplier,
$(c^*,\lam^*)$ is a local solution to the problem
\begin{equation}\label{MKL2b}
 \min_{c \in \R^{n},\lambda \in \R_p^+}
\frac{(y-K(\lambda) c)^\top(y-K (\lambda) c)}{\sigma^2}
+ c^{\top} K(\lambda) c +\gam\one^\top\lam\ .
\end{equation}
As above, first optimize (\ref{MKL2b}) in $c$ to obtain
\[
c(\lam)=(\sig^2I+K(\lam))^{-1}y.
\]
Plugging this back into the objective in (\ref{MKL2b}) gives the objective
\[
y^\top(\sig^2I+K(\lam))^{-1}y+\gam\one^\top\lam
\]
which establishes (\ref{redMarg}) and (\ref{optim_c}) for $(c^*,\lam^*)$, and hence
the pair $(c^*,\lam^*)$ satisfies (\ref{redMarg}) and (\ref{optPhi}) by the first part of the proof
and solves (\ref{MKL2}) by definition. Finally, (\ref{MKL=GLasso}) can be obtained
reformulating the objective (\ref{joint}) in terms of $\theta=\Lam^{1/2}\phi$ and $\lam$
(in place of $\phi$ and $\lam$), and then minimizing it first in $\lam$.

\subsection{Proof of Proposition \ref{Prop_lambda_HGLasso}}

Under the simplifying assumption $G^\top G = n I$, one can use (\ref{mif}) to simplify the
 necessary conditions for optimality in \eqref{KKTHGL}.
By (\ref{mif}), we have
$$
{G^{(i)}}^\top\Sig_y(\lam)^{-1}=\frac{1}{n\lam_i+\sig^2}{G^{(i)}}^\top,
$$
and so
$$
\Trace\left(G^{(i)\top}\Sigma_y^{-1}G^{(i)}\right) = \frac{n k_i}{n \lambda_i + \sigma^2}
\quad
\mbox{and}
\quad
\|G^{(i)\top}\Sigma_y^{-1}y\|_2^2 =
\left(\frac{n}{n\lambda_i + \sigma^2}\right)^2 \|\hat\theta_{LS}^{(i)}\|^2\ .
$$
Inserting these expressions into \eqref{KKTHGL} with $\mu_i=0$ yields a quadratic equation in
$\lambda_i$ which always has two real solutions.
One is always negative while the other, given by
$$
\frac{1}{4\gam}\left[\sqrt{k_i^2+8\gam \norm{\hat\theta^{(i)}_{LS}}^2}-
\left(k_i+\frac{4\sig^2\gam}{n}\right)\right]
$$
is non-negative provided
\begin{equation}\label{non_zero}
\frac{\|\hat\theta_{LS}^{(i)}\|^2}{k_i}\geq
\frac{\sigma^2}{n}\left[1+\frac{2\gam\sig^2}{nk_i}\right]\ .
\end{equation}
This concludes the proof of \eqref{LambdaEB_explicit_gamma}. The limiting behavior for $\gamma\rightarrow 0$ can be easily verified, yielding
$$
\hat \lambda_i(0) = {\rm max}\left(0,\frac{\|\hat\theta_{LS}^{(i)}\|^2}{k_i} - \frac{\sigma^2}{n}\right) \quad i=1,..,p.
$$
Also note that $\hat\theta_{LS}^{(i)} =  \frac{1}{n} \left(G^{(i)}\right)^\top y$ and $\left(G^{(i)}\right)^\top G^{(i)} = n I_{k_i}$ while $\left(G^{(i)}\right)^\top G^{(j)} = 0$, $\forall j\neq i$. This implies that $\hat\theta_{LS}^{(i)} \sim {\cal N}(\btheta^{(i)}, \frac{\sigma^2}{n} I_{k_i})$. Therefore
$$
\|\hat\theta_{LS}^{(i)}\|^2 \frac{n}{\sigma^2} \sim \chi^2(d,\mu) \quad d=k_i, \quad \mu = \|\btheta^{(i)}\|^2 \frac{n}{\sigma^2}
$$
This, together with \eqref{non_zero}, proves also \eqref{eqP0}.

\subsection{Proof of Proposition \ref{Prop_unbiased}}

In the proof of Proposition \ref{Prop_lambda_HGLasso}
it was shown that  $\|\hat\theta_{LS}^{(i)}\|^2 \frac{n}{\sigma^2}$
follows a noncentral $\chi^2$ distribution with $k_i$ degrees of freedom and noncentrality parameter $\|\theta_{t}^{(i)}\|^2 \frac{n}{\sigma^2}$.
Hence, it is a simple calculation to show that
\begin{equation}\label{EB_mean_var}
\E[\hat \lambda^*_i\,|\,\theta=\btheta ] = \frac{\|\btheta^{(i)}\|^2}{k_i}
\quad
\Var[\hat \lambda^*_i\,|\,\theta=\btheta ]
= \frac{2\sigma^4}{k_i n^2} + \frac{4 \|\btheta^{(i)}\|^2\sigma^2}{k_i^2 n}\ .
\end{equation}
By Corollary \ref{prop_min_MSE},
the first of these equations shows that $\E[\hat\lambda_i^*\,|\,\theta=\btheta ] = \lambda_i^{opt}$.
In addition, since  $\Var\{\hat \lambda^*_i\}$ goes to zero as $n\rightarrow\infty$, $\hat\lambda_i^*$ converges in mean square (and hence in probability) to $\lambda_i^{opt}$.

As for the analysis of  $\hat\lambda_i(0)$, observe that
$$
\E[\hat \lambda_i(0)\,|\,\theta=\btheta ] =
\E[\hat \lambda^*_i\,|\,\theta=\btheta ] - \int_{0}^{k_i\frac{\sigma^2}{n}} \left(\frac{\|\hat\theta^{(i)}_{LS}\|^2}{k_i} - \frac{\sigma^2}{n}\right) dP({\|\hat\theta^{(i)}_{LS}\|^2}\,|\,\theta=\btheta )
$$
where $dP({\|\hat\theta^{(i)}_{LS}\|^2}\,|\,\theta=\btheta )$
is the measure induced by $\|\hat\theta^{(i)}_{LS}\|^2$.
The second term in this expression can be bounded by
$$
- \int_{0}^{k_i\frac{\sigma^2}{n}} \left(\frac{\|\hat\theta^{(i)}_{LS}\|^2}{k_i} - \frac{\sigma^2}{n}\right) dP
({\|\hat\theta^{(i)}_{LS}\|^2}\,|\,\theta=\btheta ) \leq \frac{\sigma^2}{n} \int_{0}^{k_i\frac{\sigma^2}{n}}
dP({\|\hat\theta^{(i)}_{LS}\|^2}\,|\,\theta=\btheta ),
$$
where the last term on the right hand side goes to zero as $n\rightarrow\infty$.
This proves that $\hat\lambda_i(0)$ is asymptotically unbiased. As for consistency,
it is sufficient to observe that
$\Var[\hat\lambda_i(0)\,|\,\theta=\btheta ]\leq \Var[\hat\lambda^*_i\,|\,\theta=\btheta ]$
since ``saturation'' reduces variance. Consequently,
$\hat\lambda_i(0)$ converges in mean square to its mean,
which asymptotically is $\lambda^{opt}_i$ as shown above.
 This concludes the proof.

\subsection{Proof of Proposition \ref{Prop_lambda_MKL}}

Following the same arguments as in the proof of Proposition \ref{Prop_lambda_HGLasso}, under the assumption $G^\top G = nI$ we have that
$$
\|G^{(i)\top}\Sigma_y^{-1}y\|_2^2 = \left(\frac{n}{n\lambda_i + \sigma^2}\right)^2 \|\hat\theta_{LS}^{(i)}\|^2
$$
Inserting this expression into \eqref{FullKKT} with $\mu_i=0$,
one obtains a quadratic equation in $\lambda_i$ which has always two real solutions. One is always negative while the other, given by
$$
\frac{\|\hat\theta_{LS}^{(i)}\|}{\sqrt{2\gamma}} - \frac{\sigma^2}{n}.
$$
is non-negative provided
\begin{equation}\label{non-zero-MKL}
\|\hat\theta_{LS}^{(i)}\|^2\geq \frac{2\gamma\sigma^4}{n^2}\ .
\end{equation}
This concludes the proof of \eqref{LambdaEB_explicit_gammaMKL}.

The limiting behavior for $n\rightarrow \infty$
in equation \eqref{not_consistent_MKL} is easily verified with arguments similar to those in the proof of Proposition \ref{Prop_unbiased}.
As in the proof of Proposition \ref{Prop_lambda_HGLasso},
$\|\hat\theta_{LS}^{(i)}\|^2 \frac{n}{\sigma^2}$ follows a
noncentral $\chi^2(d,\mu)$ distribution with $d=k_i$ and $\mu=\|\btheta^{(i)}\|^2 \frac{n}{\sigma^2}$, so that
 from \eqref{non-zero-MKL} the probability of setting $\hat\lambda_i(\gamma)$ to zero is as given  in  \eqref{eqP0MKL}.

\subsection{Proof of Theorem \ref{MainConvergenceTheorem}}

Recalling model~\eqref{GroupMeasmod}, 
assume that $G^\top G/n$ is bounded and bounded away from zero in probability, so that  there exist constants 
$\infty> c_{max}\geq c_{min}>0$ with
\begin{equation}\label{BoundedGP}
\lim_{n\rightarrow \infty}P[c_{min}I\leq  G^\top G/n \leq c_{max}I ] = 1\;,
\end{equation}
so as $n$ increases, the probability that a particular realization $G$ satisfies 
\begin{equation}\label{BoundedG}
c_{min}I\leq  G^\top G/n \leq c_{max}I 
\end{equation}
increases to $1$. We now characterize the behavior of key matrices used in the analysis.

We first provide a technical lemma which will become useful in the sequel:
\begin{lemma}\label{lemma_min_angle}
Assume \eqref{BoundedG} holds; then the following conditions hold
\begin{enumerate}
\item[(i)]  Consider an arbitrary subset  $I = [I(1), \dots, I(p_I)]$ of size $p_I$ to be any subset of the indices $[1, \dots, p]$, so $p \leq p_I$
and define \begin{equation}\label{generalIndexG}
G^{(I)} = \begin{bmatrix} G^{({I(1)})} \dots G^{({I(p_I)})}\end{bmatrix}\;,
\end{equation}
obtained by taking the subset of {\it blocks of columns of $G$} indexed by $I$. Then
\begin{equation}
\label{GInterp}
c_{min}I \leq \frac{(G^{(I)})^T G^{(I)}}{n} \leq c_{max}I\;.
\end{equation}
%
\item[(ii)] Let $I^c$ be the complementary set of $I$ in $[1, \dots, p]$, so that $I^c\cap I = \emptyset$ and $I \cup I^c = [1, \dots, p]$. The minimal angle $\theta_{min}$
between the spaces
$${\cal G}^I:={col\,span}\{G^{(i)}/\sqrt{n}, \;\; i\in I\} \quad{\rm and}\quad {\cal G}^{I^c}:=
{col\,span}\{G^{(j)}/\sqrt{n}:j \in I^c\}$$ satisfies:
$$
\theta_{min} \geq {\rm acos}\left(\sqrt{1-\frac{c_{min}}{c_{max}}}\right)>0
$$
\end{enumerate}
\end{lemma}

\begin{proof}  Result~\eqref{GInterp} is a direct consequence of~\cite[Corollary 3.1.3]{HornJohnson}.  As far as condition (ii) is concerned we can proceed as follows:
let $U_{I}$ and $U_{I^c}$ be orthonormal matrices whose columns span ${\cal G}^I$ and ${\cal G}^{I^c}$, so that there exist matrices $T_{I}$ and $T_{I^c}$ so that
$$\begin{array}{c}
G^{(I)}/\sqrt{n} = U_I T_I\\
G^{(I^c)}/\sqrt{n} = U_{I^c} T_{I^c}\end{array}
$$
where $G^{(I^c)}$ is defined analogously to $G^{(I)}$. The minimal angle between  ${\cal G}^I$ and ${\cal G}^{I^c}$
satisfies 
$$
{\rm cos}(\theta_{min}) = \left\|U_{I}^\top U_{I^c}\right\|.$$
Now observe that, up to a permutation of the columns which is irrelevant, $G/\sqrt{n} = [U_I T_I \;\;\; U_{I^c} T_{I^c}]$, so that
$$
U_I^\top G /\sqrt{n} = [T_I\; \;\; U_{I}^\top U_{I^c}\ T_{I^c}] = [I \; \;\;U_{I}^\top U_{I^c}] 
\left[\begin{array}{cc} T_I & 0 \\ 0 & T_{I^c}\end{array}\right].
$$
Denoting with $\sigma_{min}(A)$ and $\sigma_{max}(A)$ the minimum and maximum singular values of a matrix $A$, it is a straightforward calculation to verify that the following chain of inequalities holds:
$$
\begin{array}{rcl}
c_{min} = \sigma_{min}(G^\top G/n) \leq \sigma^2_{min}\left(U_I^\top G /\sqrt{n} \right) &= &  \sigma^2_{min}\left([I \;\;\; U_{I}^\top U_{I^c}] \left[\begin{array}{cc} T_I & 0 \\ 0 & T_{I^c}\end{array}\right]\right) \\
& \leq &  \sigma^2_{min}\left([I \;\;\; U_{I}^\top U_{I^c}] \right) \sigma^2_{max}\left( \left[\begin{array}{cc} T_{I} & 0 \\ 0 & T_{I^c}\end{array}\right]\right) \\
& = &  \sigma^2_{min}\left([I \;\;\; U_{I}^\top U_{I^c}] \right) {\rm max}\left( \sigma^2_{max}(T_{I}), \sigma^2_{max}(T_{I^c}) \right)\\ 
& \leq & \sigma^2_{min}\left([I \;\;\; U_{I}^\top U_{I^c}] \right) c_{max}.
\end{array}
$$
Observe now that $\sigma^2_{min}\left([I \;\;\; U_{I}^\top U_{I^c}] \right)  = 1-{\rm cos}^2(\theta_{min})$ so that 
$$
c_{min} \leq (1-{\rm cos}^2(\theta_{min}))c_{max}
$$
and, therefore,
$$
{\rm cos}^2(\theta_{min}) \leq 1-\frac{c_{min}}{c_{max}}
$$
from which the thesis follow.
\end{proof}

\noindent
{\bf Proof of Lemma \ref{lemma_diagonal_model}:}
Let us consider the Singular Value Decomposition (SVD)
\begin{equation}\label{SVD_P}
\frac{\sum_{j=1, j\neq i}^{p}  G^{(j)} \left(G^{(j)}\right)^\top \lambda_j}{n} = PSP^\top;
\end{equation}
where, by the assumption \eqref{BoundedG}, using $\frac{\sum_{j=1, j\neq i}^{p}  G^{(j)} \left(G^{(j)}\right)^\top \lambda_j}{n}\geq \frac{\sum_{j=1, j\neq i, \lambda_j\neq 0}^{p}  G^{(j)} \left(G^{(j)}\right)^\top }{n}{\rm min}\{\lambda_j, j : \lambda_j\neq 0\}$ and lemma \ref{lemma_min_angle} the minimum singular value $\sigma_{min}(S)$ of $S$ in \eqref{SVD_P} satisfies
     \begin{equation}\label{min_S}
     \sigma_{min}(S)\geq c_{min}{\rm min}\{\lambda_j, j : \lambda_j\neq 0\}.
      \end{equation}
      Then the SVD of $\Sigma_{\bar v}=\sum_{j=1, j\neq i}^{p}  G^{(j)} \left(G^{(j)}\right)^\top \lambda_j +\sigma^2 I$
satisfies
$$
\Sigma_{\bar v}^{-1} = \left[\begin{array}{cc}P & P_\perp\end{array}\right]
\left[\begin{array}{cc}(nS+\sigma^{2})^{-1} & 0 \\ 0 & \sigma^{-2}I\end{array}\right]
\left[\begin{array}{c}P^\top \\ P_\perp^\top\end{array}\right]
$$
so that $\|\Sigma_{\bar v}^{-1} \| =\sigma^{-2}$.

Note now that 
$$
D_n^{(i)} =\left(U_n^{(i)}\right)^\top\frac{\Sigma_{\bar v}^{-1/2} G^{(i)}}{\sqrt{n}}V_n^{(i)}
$$
and therefore, using Lemma \ref{lemma_min_angle},
$$
\|D_n^{(i)}\|\leq \|\Sigma_{\bar v}^{-1} \| \sqrt{c_{max}} = \sigma^{-2}\sqrt{c_{max}}.
$$ proving that $D^{(i)}_n$ is bounded. In addition, again using  Lemma \ref{lemma_min_angle}, condition  \eqref{BoundedG}
 implies   that $\forall a,b$ (of suitable dimensions) s.t. $\|a\|=\|b\|=1$,
$a^\top\frac{P_\perp^\top G^{(i)}}{\sqrt{n}} b \geq k$, $k = \sqrt{1-{\rm cos}^2(\theta_{min})}\geq \frac{c_{min}}{c_{max}}>0$. This, using \eqref{SVD},  guarantees that
$$
\begin{array}{rcl}
D_n^{(i)} & =&\left(U_n^{(i)}\right)^\top\frac{\Sigma_{\bar v}^{-1/2} G^{(i)}}{\sqrt{n}}V_n^{(i)}=
\left(U_n^{(i)}\right)^\top \left(P (nS + \sigma^2)^{-1/2}P^\top+P_\perp \sigma^{-1} P_\perp^\top\right) \frac{G^{(i)}}{\sqrt{n}}\\
& \geq &
 \left(U_n^{(i)}\right)^\top\left(P_\perp \sigma^{-1} P_\perp^\top\right)\frac{G^{(i)}}{\sqrt{n}}\\
 &
 \geq& k\sigma^{-1} I
 \end{array}
$$ and therefore  $D^{(i)}_n$ is bounded away from zero. It is then a matter of simple calculations to show that with the definitions \eqref{Diagonal_model_definition} then \eqref{GroupMeasmod_i} can be rewritten in the equivalent form \eqref{Diagonal_model}.
${\square}$
\smallskip

\begin{lemma}\label{lemma_structure_epsilon}
Assume, w.l.o.g., that the blocks of $\theta$ have been reordered so that $\|\bar\theta^{(j)}\|\neq 0$, $j=1,..,k$
and $\|\bar\theta^{(j)}\| = 0$, $j=k+1,..,m$ and that the spectrum of $G^\top G/n$ is bounded and bounded away from zero in probability, so that
\begin{equation}\label{condition_G}
{\rm lim}_{n\rightarrow\infty}P[c_{max} I\geq G^\top G/n \geq c_{min}I]=1.
\end{equation} Denote
 $$
 \begin{array}{c}
 I_1:=\{j \in [1,k],\; j\neq i\}\\
 I_0:=\{j \in [k+1,p],\; j\neq i\}
 \end{array}
 $$
 and assume also that the numbers $\lambda_j^n$, which are here allowed to depend on $n$, are bounded and satisfy:
\begin{equation}\label{Mlambda}
\begin{array}{c} \displaystyle{\mathop{\rm lim}_{n\rightarrow\infty}}\; f_n =+\infty \quad {\rm where} \quad f_n : = \displaystyle{\mathop{\rm min}_{j\in I_1}} \;n\lambda_j^n
\end{array}\end{equation}
Then, conditioned on $\theta$,  $\epsilon_n^{(i)}$ in \eqref{Diagonal_model_definition} and
\eqref{Diagonal_model} can be decomposed as
\begin{equation}\label{epsilon_n}
\epsilon_n^{(i)} = m_{\epsilon_n}(\theta) + v_{\epsilon_n}.
\end{equation}
The following conditions hold: 
\begin{equation}\label{m_sigma_epsilon}
 \E_v \left[\epsilon_n^{(i)}   \right]= m_{\epsilon_n}(\theta)  = O_P\left(\frac{1}{\sqrt{f_n}}\right) \quad 
 \quad  v_{\epsilon_n}  = O_P\left(\frac{1}{\sqrt{n}}\right)
\end{equation} so that
$\epsilon_n^{(i)}|\theta$
converges to zero in probability (as $n\rightarrow \infty$).
In addition 
\begin{equation}\label{convergence_variance_conditioned}
 {Var}_v\{\epsilon_n^{(i)}\} = \E_{v}\left[v_{\epsilon_n} v_{\epsilon_n} ^\top\right] =  O_P\left(\frac{1}{{n}}\right).
\end{equation}
If in addition
\footnote{This is equivalent to say that the columns of $G^{(j)}$,  $j=1,..,k$, $j\neq i$ are asymptotically orthogonal to the columns of $G^{(i)}$.}
\begin{equation}\label{block_asy_orth}
n^{1/2}\frac{\left(G^{(i)}\right)^\top G^{(j)}}{n} = O_P(1)\;\; ;\;j=1,..,k\;\;j\neq i
\end{equation}
then
\begin{equation}\label{m_epsilon_uncorrelated blocks}
m_{\epsilon_n}(\theta)  = O_P\left(\frac{1}{\sqrt{nf_n}}\right)
\end{equation}
\end{lemma}
\begin{proof}
Consider the Singular Value Decomposition
\begin{equation}\label{Pideale}
\bar P_1 \bar  S_1 \bar  P_1^\top : = \frac{1}{n}\sum_{j \in I_1} G^{(j)} \left(G^{(j)}\right)^\top \lambda_j^n.
\end{equation}
Using \eqref{Mlambda}, there exist $\bar n$ so that, $\forall\; n>\bar n$ we have
 $0< \lambda_j^n\leq M<\infty$, $j\in I_1$. Otherwise, we could find a subsequence $n_k$ so that $\lambda_j^{n_k}=0$ and 
hence $n_k\lambda_j^{n_k} = 0$, contradicting \eqref{Mlambda}. 
Therefore, the matrix $ \bar P_1$ in \eqref{Pideale}  is an orthonormal basis
for the space
${\cal G}_1:={col\,span}\{G^{(j)}/\sqrt{n}:j\in I_1\}$.
Let also  $T_j$ be such that $G^{(j)}/\sqrt{n} =\bar  P_1T_j$, $j\in I_1$. Note that by assumption~\eqref{BoundedGP} and lemma \ref{lemma_min_angle}
%
\begin{equation}\label{OpT}
\|T_j\| = O_P(1)\quad \forall
\, j\in I_1.
\end{equation}

Consider now the Singular Value Decomposition
\begin{equation}\label{Preale}
\begin{array}{rcccc}
\left[\begin{array}{cc}P_1 &  P_0\end{array}\right]
\left[\begin{array}{cc} S_1 & 0 \\ 0 & S_0\end{array}\right]
\left[\begin{array}{c} P_1^\top \\  P_0^\top\end{array}\right]
      &:=& \underbrace{\frac{1}{n}\sum_{j \in I_1} G^{(j)} \left(G^{(j)}\right)^\top \lambda_j^n} &+& \underbrace{\frac{1}{n}\sum_{j \in I_0} G^{(j)} \left(G^{(j)}\right)^\top \lambda_j^n} \\
& = & \bar  P_1 \bar S_1 \bar P_1^\top & + & \Delta.
\end{array}
\end{equation}
For future reference note that $\exists T_{\bar P_1} \; : \; \bar P_1 = \left[\begin{array}{cc} P_1 &  P_0\end{array}\right]T_{\bar P_1}$.
Now, from \eqref{SVD} we have that
\begin{equation}\label{SVD_new}
\frac{\Sigma_{\bar v}^{-1} G^{(i)}}{\sqrt{n}}V_n^{(i)}\left(D_n^{(i)}\right)^{-1} =\Sigma_{\bar v}^{-1/2} U_n^{(i)}.
\end{equation}
Using \eqref{SVD_new} and defining
$$ P:= \left[\begin{array}{cc} P_1 &  P_0\end{array}\right] \quad  S:=\left[\begin{array}{cc} S_1 & 0 \\ 0 & S_0\end{array}\right],$$
equation \eqref{Diagonal_model_definition} can be rewritten as:
$$
\begin{array}{rcl}
\epsilon_n^{(i)} &=& \left(U_n^{(i)}\right)^\top \frac{\Sigma_{\bar v}^{-1/2} \bar v}{\sqrt{n}}\\
& = &\left(D_n^{(i)}\right)^{-1} \left(V_n^{(i)}\right)^{\top}\frac{\left(G^{(i)}\right)^\top}{\sqrt{n}}
 \Sigma_{\bar v}^{-1} \frac{\bar v}{\sqrt{n}}\\
 &  = & \left(D_n^{(i)}\right)^{-1} \left(V_n^{(i)}\right)^{\top} \frac{\left(G^{(i)}\right)^\top}{\sqrt{n}}
 \left[\begin{array}{cc} P &  P_\perp\end{array}\right] \left[\begin{array}{cc}(n S+\sigma^{2}I)^{-1} & 0 \\ 0 & \sigma^{-2}I\end{array}\right]\left[\begin{array}{c} P^\top \\  P_\perp^\top\end{array}\right]
 \frac{\bar v}{\sqrt{n}} \\
 & = & \left(D_n^{(i)}\right)^{-1} \left(V_n^{(i)}\right)^{\top} \frac{\left(G^{(i)}\right)^\top}{\sqrt{n}}
  \left[\begin{array}{cc} P & P_\perp\end{array}\right] \left[\begin{array}{cc}(n S+\sigma^{2}I)^{-1} & 0 \\ 0 & \sigma^{-2}I\end{array}\right]\left[\begin{array}{c} P^\top \\  P_\perp^\top\end{array}\right]
 \left[\sum_{j \in I_1}\frac{G^{(j)}}{\sqrt{n}}\theta^{(j)} + \frac{v}{\sqrt{n}}\right]\\
 & = & \underbrace{\left(D_n^{(i)}\right)^{-1} \left(V_n^{(i)}\right)^{\top} \frac{\left(G^{(i)}\right)^\top}{\sqrt{n}}
   P  (n S+\sigma^{2}I)^{-1}\left[\begin{array}{c} P_1^\top P_1 \\  P_0^\top P_1\end{array}\right]
 \sum_{j \in I_1}T^{(j)}\theta^{(j)} }_{m_{\epsilon_n}(\theta)}+\\
 & & +\underbrace{ \left(D_n^{(i)}\right)^{-1} \left(V_n^{(i)}\right)^{\top} \frac{\left(G^{(i)}\right)^\top}{\sqrt{n}}
  \left[\begin{array}{cc} P & P_\perp\end{array}\right] \left[\begin{array}{cc}(n S+\sigma^{2}I)^{-1} & 0 \\ 0 & \sigma^{-2}I\end{array}\right] \frac{v_{\bar p}}{\sqrt{n}}}_{v_{\epsilon_n}}
\end{array}
$$
where the last equation defines $m_{\epsilon_n}(\theta)$ and $ v_{\epsilon_n}$, the noise $$
v_{\bar P}:=\left[\begin{array}{c} P^\top \\  P_\perp^\top\end{array}\right] v
$$
is still a zero mean Gaussian noise with variance $\sigma^2 I$
and  $\frac{G^{(j)}}{\sqrt{n}} = P_1T^{(j)}$ provided $j\neq i$.
Note that $m_{\epsilon_n}$ does not depend on $v$ and that $\E_v v_{\epsilon_n} =0$. Therefore $m_{\epsilon_n}(\theta)$ is the mean (when only noise $v$ is averaged out) of $\epsilon_n$. 
 As far as the asymptotic behavior of $m_{\epsilon_n}(\theta)$ is concerned, it is convenient to preliminary observe that
 $$
  (n S+\sigma^{2}I)^{-1}\left[\begin{array}{c} P_1^\top \bar P_1 \\  P_0^\top \bar P_1\end{array}\right]=\left[\begin{array}{c}
  (n S_1 + \sigma^{2}I)^{-1} P_1^\top \bar P_1 \\ (n S_0 + \sigma^{2}I)^{-1} P_0^\top \bar P_1\end{array}\right]
  $$
   and that the second term on the right hand side can be rewritten as
  \begin{equation}\label{term-to-zero}
  (n S_0 + \sigma^{2}I)^{-1} P_0^\top \bar P_1 = \left[\begin{array}{c} (n [ S_0]_{1,1} +\sigma^{2}I)^{-1} P_{0,1}^\top \bar P_1 \\ 
  (n [ S_0]_{22} +\sigma^{2}I)^{-1} P_{0,2}^\top \bar P_1\\ \vdots \\(n [ S_0]_{m-k,m-k} +\sigma^{2}I)^{-1} P_{0,m-k}^\top \bar P_1 \end{array}\right]
    \end{equation}
%
  where $[ S_0]_{ii}$ is the $i-th$ diagonal element of $ S_0$ and $ P_{0,i}$ if the $i-th$ column of $ P_0$. Now, using equation \eqref{Preale} one obtains that
  $$
  \begin{array}{rcl}
  n[ S_{0}]_{ii} = P_{0,i}^\top  P n S   P^\top  P_{0,i} & = &
 P_{0,i}^\top \left( \bar P_1 n\bar S_1  \bar P_1^\top + n\Delta\right)  P_{0,i} \\
& \geq  &  P_{0,i}^\top\bar  P_1 n \bar S_1 \bar  P_1^\top   P_{0,i} \\
& \geq  & \sigma_{min}(n \bar S_1) P_{0,i}^\top \bar P_1 \bar P_1^\top   P_{0,i} \\
& = & \sigma_{min}(n\bar S_1)\| P_{0,i}^\top \bar P_1\|^2.
\end{array}
$$
With an argument similar to that used in \eqref{min_S}, also
\begin{equation}\label{min_S1}
\sigma_{min}(n \bar S_1) \geq c_{min} {\rm min}\{n\lambda_j^n,\; j\in I_1\} = c_{min}f_n
  \end{equation}
  holds true; denoting $\| P_{0,i}^\top \bar P_1\|= g_n$,  the generic term on the right hand side of  \eqref{term-to-zero} satisfies
\begin{equation}\label{bound-term-to-zero}
\begin{array}{rcl}
\|(n [ S_0]_{ii} +\sigma^{2}I)^{-1} P_{0,i}^\top \bar P_1 \| &\leq&
\frac{\| P_{0,i}^\top \bar P_1\|}{n \sigma_{min}(\bar S_1)\| P_{0,i}^\top \bar P_1\|^2 + \sigma^2} \\
&\leq&
k \;{\rm min}(g_n,(f_n g_n)^{-1}) \\
&=& \frac{k}{\sqrt{f_n}} {\rm min}(\sqrt{f_n} g_n,(\sqrt{f_n} g_n)^{-1}) \\
&\leq & \frac{k}{\sqrt{f_n}}
\end{array}\end{equation}
for some positive constant $k$.   Now, using
lemma \ref{lemma_diagonal_model},
  $D_n^{(i)}$ is bounded and bounded away from zero in probability, so that
  $\|D_n^{(i)}\| = O_P(1)$ and $\|\left(D_n^{(i)}\right)^{-1}\| = O_P(1)$.
In addition $V_n^{(i)}$ is an orthonormal matrix and $\|\frac{G^{(i)}}{\sqrt{n}}\| = O_P(1)$.
Last, using \eqref{min_S1} and \eqref{condition_G}, we have $\|(n S_1+\sigma^2)^{-1}\| = O_P(1/n)$. Combining these conditions with \eqref{OpT}  and \eqref{bound-term-to-zero} we obtain the first of \eqref{m_sigma_epsilon}. As far as the asymptotics  on $v_{\epsilon_n}$  are concerned, it suffices to observe that
$$
w_n^\top v_{\bar P}/\sqrt{n} = O_P(1/\sqrt{n})\;\; {\rm if}  : \; \|w_n\|=O_P(1).
$$
The variance (w.r.t. noise $v$) ${Var}_v\{{\epsilon_n}\} = \E_v\left[v_{\epsilon_n}v_{\epsilon_n}^\top\right]$ satisfies 
$$
{Var}_v\{{\epsilon_n}\} =\frac{\sigma^2}{n} \left(U_n^{(i)}\right)^\top \Sigma_{v}^{-1}  \left(U_n^{(i)}\right)
$$
so that, using the condition $\|\Sigma_{\bar v}^{-1}\| = \sigma^{-2}$ derived in Lemma \ref{lemma_diagonal_model}, and the fact that $U_n^{(i)}$ has orthonormal columns,  the condition $ {Var}_v\{{\epsilon_n}\}  = O_P\left(\frac{1}{n}\right)$ in \eqref{convergence_variance_conditioned} follows immediately. 

If in addition  \eqref{block_asy_orth} holds then \eqref{OpT} becomes
$$
\|T_j\| = O_P(1/\sqrt{n})\quad j=1,...,k; \;\;\; j\neq k$$
so that and extra $\sqrt{n}$ appears at the denominator in the expression of $m_\epsilon(\theta)$ yielding \eqref{m_epsilon_uncorrelated blocks}. This concludes the proof.
\end{proof}

Our next results will focus on the estimator~\eqref{lambda^n}. We will show 
that when the hypotheses of Lemma~\ref{lemma_diagonal_model} hold, estimator~\eqref{AlternativeGlambda} satisfies the 
key hypothesis of Lemma~\ref{lemma_structure_epsilon}. We first take a close look 
at the objective~\eqref{AlternativeGlambda}.

\begin{lemma} \label{StructureObjective}
Take objective~\eqref{AlternativeGlambda} divided by $n$:
\begin{equation}\label{AlternativeGlambdaExplicit}
\begin{aligned}
g_n(\lambda) 
= 
\log\sigma^2
+
&\underbrace{\frac{1}{2n} \log \det(\sigma^{-2}\Sigma_y(\lam))}_{S_1} 
+ 
\underbrace{\frac{1}{2n}\sum_{j\in I_1} \frac{\|\hat \theta^{(j)}(\lambda)\|^2}{k_j\lambda_j}}_{S_2}
+
\underbrace{\frac{1}{2n}\sum_{j\in I_0} \frac{\|\hat \theta^{(j)}(\lambda)\|^2}{k_j\lambda_j}}_{S_3}\\
&+
\underbrace{\frac{1}{n}\gamma\|\lambda\|_1}_{S_4}
+
\underbrace{\frac{1}{2n\sigma^2}\|y - \sum_j G^{j}\hat\theta^{(j)}(\lambda)\|^2}_{S_5}\;,
\end{aligned}
\end{equation}
where
$\hat\theta(\lambda) = \Lambda G^T\Sigma_y^{-1}y$ (see~\eqref{AlternativeGteta}), $k_j$ is the size of the $j$th block, 
and dependence on $n$ has been suppressed. 
For any minimizing sequence $\lambda^n$, we have the following results:
\begin{enumerate}
\item 
$\hat\theta_n  \rightarrow_p \bar\theta$.  
\item $S_1, S_2, S_3, S_4 \rightarrow_p 0$.
\item $S_5 \rightarrow_p \frac{1}{2}$. 
\item $n\lambda_j^n \rightarrow_p\infty$ for all $j\in I_1$. 
\end{enumerate}

\end{lemma}

\begin{proof}
First, note that $0\leq S_i$ for $i \in\{1,2,3,4\}$. Next,   
\begin{equation}\label{BiasVariance}
\begin{aligned}
S_5 &=  \frac{1}{2n\sigma^2}\|y - \sum_j G^{j}\bar\theta^{(j)}(\lambda) + \sum_j G^{j}\left(\bar\theta^{(j)}(\lambda) -\hat\theta^{(j)}(\lambda)\right)\|^2\\
&=  \frac{1}{2n\sigma^2}\|\nu+ \sum_j G^{j}\left(\bar\theta^{(j)}(\lambda) -\hat\theta^{(j)}(\lambda)\right)\|^2\\
&= \frac{1}{2n\sigma^2}\|\nu\|^2 + \frac{1}{2n\sigma^2} \nu^T\sum_j G^{j}\left(\bar\theta^{(j)}(\lambda) -\hat\theta^{(j)}(\lambda)\right)
+ 
\frac{1}{2n\sigma^2}\|\sum_j G^{j}\left(\bar\theta^{(j)}(\lambda) -\hat\theta^{(j)}(\lambda)\right)\|^2.
\end{aligned}
\end{equation}
The first term converges in probability to $\frac{1}{2}$. Since $\nu$ is independent of all $G^{j}$, the middle term converges in probability to $0$. 
The third term is the bias incurred unless $\hat\theta = \bar\theta$. 
These facts imply that, $\forall \epsilon>0$, 
\begin{equation}\label{ObjectiveSat_0}
\lim_{n\rightarrow\infty} P\left[S_5(\lambda(n)) > \frac{1}{2}-\epsilon\right] = 1\;.
\end{equation}
Next, consider the particular sequence $\bar\lambda_j^n = \frac{\|\bar\theta_j\|^2}{k_j} 	$. For this sequence, it is immediately clear 
that $S_i\rightarrow_p 0$ for $i \in \{2,3,4\}$. To show $S_1 \rightarrow_p 0$, note that $\sum \lambda_i G_iG_i^T \leq \max\{\lambda_i\} \sum G_iG_i^T$,
and that the nonzero eigenvalues of $GG^T$ are the same as those of $G^TG$. Therefore, we have 
\[
S_1 \leq \frac{1}{2n}\sum_{i=1}^m \log(1 + n\sigma^{-2}\max\{\lambda\}c_{max}) = O_P\left(\frac{\log(n)}{n}\right) \rightarrow_p 0\;.
\]
Finally $S_5 \rightarrow_p \frac{1}{2}$ by~\eqref{BiasVariance}, so in fact, $\forall \epsilon>0$, 
\begin{equation}\label{ObjectiveSat}
\lim_{n\rightarrow\infty} P\left[\left| g_n(\bar\lambda(n)) - \frac{1}{2} - \log(\sigma^2)\right|<\epsilon \right] = 1\;.
\end{equation}
Since \eqref{ObjectiveSat} holds for the deterministic sequence $\bar\lambda_n$,  
any minimizing sequence $\hat\lambda_n$ must satisfy, $\forall\epsilon>0$, 
\begin{equation}\label{ObjectiveSatbis}
\lim_{n\rightarrow\infty} P\left[ g_n(\hat\lambda(n)) < \frac{1}{2} + \log(\sigma^2)+\epsilon \right] = 1\;.
\end{equation}
which, together with \eqref{ObjectiveSat_0}, implies \eqref{ObjectiveSat}

Claims $1,2,3$ follow immediately. To prove claim 4, suppose that for a particular minimizing sequence $\check\lambda(n)$, 
we have $n\check\lambda_j^n \not \rightarrow_p \infty$ for $j\in I_1$. We can therefore find a subsequence where 
$n\check\lambda_j^n \leq K$, and since $S_2(\check\lambda(n)) \rightarrow_p 0$, we must have
$\|\check\theta^{(j)}(\check \lambda)\|\rightarrow_p 0.$
But then there is a nonzero bias term in~\eqref{BiasVariance}, since in particular $\bar\theta^{(j)}(\lambda) - \hat\theta^{(j)}(\lambda) = \bar\theta^{(j)}(\lambda) \neq 0$, 
which contradicts the fact that $\check\lambda(n)$ was a minimizing sequence. 
\end{proof}

Before we proceed, we review a useful characterization of convergence. While it can be stated for many types of convergence, we 
present it specifically for convergence in probability, since this is the version we will use. 
%
\begin{remark}
\label{SubConv}
$a^n $ converges in probability to $a$ (written $a^n \rightarrow_p a)$ if and only if every subsequence $a^{n(j)}$ of $a^n$ 
has a further subsequence $a^{n(j(k))}$ with $a^{n(j(k))} \rightarrow_p a$. 
\end{remark}
\begin{proof}
If $a^n \rightarrow_p a$, this means that for any $\epsilon > 0$, $\delta >0$ there exists some $n_{\epsilon, \delta}$ such that 
for all $n \geq n_{\epsilon, \delta}$, we have $P(|a^n - a| > \epsilon) \leq \delta$. Clearly, if $a^n \rightarrow_p a$, then $a^{n(j)} \rightarrow_p a$  for every 
subsequence $a^{n(j)}$ of $a^n$. We prove the other direction by contrapositive. 

Assume that $a^n \not\rightarrow_p a$. That means precisely that there exist some $\epsilon > 0, \delta > 0$ 
and a subsequence $a^{n(j)}$ so that $P(|a - a^{n(j)}|> \epsilon) \geq \delta$. 
Therefore the subsequence $a^{n(j)}$ cannot have further subsequences that converge to $a$ in probability, since every term of $a^{n(j)}$ 
stays $\epsilon$-far away from $a$ with positive probability $\delta$. 
\end{proof}
Remark~\ref{SubConv} plays a major role in the next lemma
from which Theorem \ref{MainConvergenceTheorem} immediately comes.

\begin{lemma}
\label{MainConvergenceTheoremAlong1}
Let $\lambda_1$ be arbitrary and consider the estimator 
(\ref{lambda^n}) along $\lambda_1$ that, in view of \eqref{Diagonal_model}, is given by:
\begin{equation}
\label{marginalLam}
\hat{\lambda}_{1}^n = \arg \min_{{\lambda} \in \R_+}
\frac{1}{2}\sum_{k=1}^{k_1}
\left[\frac{\eta_{k,n}^2 + v_{k,n}}{\lam +w_{k,n}}+\log(\lam+w_{k,n})\right] + \gamma \lambda\;,
\end{equation}
 where $w_{k,n}:=1/(n(d^{(1)}_{k,n})^2)$ and $v_{k,n} = 2\epsilon_{k,n}^{(1)}d_{k,n}^{(1)} + (\epsilon_{k,n}^{(1)})^2$. 
Suppose that the hypotheses of Lemma~\ref{lemma_diagonal_model} hold, so that $w_{k,n} \rightarrow 0$.   
Then by Lemma~\ref{StructureObjective} we know Lemma~\ref{lemma_structure_epsilon} applies, so  
that $v_{k,n} \rightarrow_p 0$. 
Let 
\[
\bar \lambda_1^\gamma := \frac{-k_1 + \sqrt{k_1^2 + 8\gamma\|\theta^{(1)}\|^2}}{4\gamma}\;,
\quad 
\bar \lambda_1 =  \frac{\|\theta^{(1)}\|^2}{k_1}\;.
\]
We have the following results:
\begin{enumerate}
\item $\bar\lambda_1^\gamma \leq \bar \lambda_1$ for all $\gamma > 0$, and $\lim_{\gamma \rightarrow 0^+}\bar\lambda_1^\gamma = \bar \lambda_1$\;.
\item If $\|\theta^{(1)}\| > 0$ and $\gamma > 0$, we have $\hat \lambda_{1}^n \longrightarrow_p \bar \lambda_1^\gamma $\;.
\item If $\|\theta^{(1)}\| > 0$ and $\gamma = 0$, we have $\hat \lambda_{1}^n \longrightarrow_p \bar \lambda_1 $\;.
\item if $\theta^{(1)}  = 0$, we have $\hat \lambda_{1}^\gamma \longrightarrow_p 0$ for any value $\gamma \geq 0$. 
\end{enumerate}

\end{lemma}
\begin{proof}
\begin{enumerate}
\item 
The reader can quickly check that $\frac{d}{d\gamma}\bar\lambda_1^\gamma< 0$, so $\bar \lambda_1^\gamma$ is decreasing in 
$\gamma$. The limit calculation follows immediately from L'Hopital's rule it is clear that $\lim_{\gamma\rightarrow 0^+} \lambda_1^\gamma= \bar\lambda_1$. 

\item We use the convergence characterization given in Remark~\ref{SubConv}. Pick any subsequence $\hat \lambda^{n(j)}_{1}$ of $\hat \lambda^n_{1}$.
 Since $\{V_{n(j)}\}$ is bounded, by Bolzano-Weierstrass 
it must have a convergent subsequence $V_{n(j(k))} \rightarrow V$, where $V$
satisfies $V^TV = I$ by continuity of the $2$-norm. 
The first order optimality conditions for $\hat \lambda_1^n > 0$  are given by 
\begin{equation}
\label{OptimalityMargin}
0 = f_1(\lambda, w, v,\eta) = \frac{1}{2}\sum_{k=1}^{k_1} \frac{-\eta_{k}^2-v_{k}}{(\lambda + w_{k})^2} + \frac{1}{\lambda + w_k} + \gamma\;,
\end{equation}
and we have $f_1(\lambda, 0, 0, V^T\theta^{(1)}) = 0$ if and only if $\lambda = \bar \lambda_1^\gamma$. 
Taking the derivative we find 
\[
\frac{d}{d\lambda} f_1(\lambda, 0,0, V^T\theta^{(1)}) = \frac{\|\theta^{(1)}\|^2}{\lambda^3} - \frac{k_1}{2\lambda^2}\;,
\]
which is nonzero at $\lambda_1^\gamma$ for any $\gamma$, since the only zero is at $2\frac{\|\theta^{(1)}\|^2}{k_1}= 2\bar \lambda_1 \geq 2\bar \lambda_1^\gamma$. 

Applying the Implicit Function Theorem to $f$ at
$\left(\lambda_1^\gamma,0,0,V^\top\btheta^{(1)}\right)$ yields the existence of neighborhoods $\cU$
of $(0,0,V^\top\btheta^{(1)})$ and $\cW$ of $\lambda_1^\gamma$ such that
\[
f(\phi(w,v,\eta),w,v,\eta)=0\qquad \forall\, (w,v,\eta)\in \cU\ .
\]
In particular, $\phi(0,0,V^\top\btheta^{(1)})=\lambda_1^\gamma$.
Since $(w_{n(j(k))},v_{n(j(k))}, \eta_{n(j(k))})\rightarrow_p (0,0,V^\top \btheta^{(1)})$, 
we have that for any $\delta > 0$ there exist some $k_{\delta}$ so that 
for all $n(j(k)) > n(j(k_{\delta}))$ we have $P((w_{n(j(k)},v_{n(j(k))}, \eta_{n(j(k))})\not\in\cU) \leq \delta$. 
For anything in $\cU$, by continuity of $\phi$ we have 
\[
\hat\lam^{n(j(k))}_1 = \phi(w_{n(j(k))}, v_{n(j(k))},\eta_{n(j(k))}) \rightarrow_p \phi(0,0,V^\top \btheta^{(1)}) =  \lambda_1^\gamma\;.
\]
These two facts imply that $\hat\lam^{n(j(k))}_1 \rightarrow_p  \lambda_1^\gamma$. 
We have shown that every subsequence $\hat\lam^{n(j)}_1$ has a further subsequence 
$\hat\lam^{n(j(k))}_1\rightarrow_p \lambda_1^\gamma$, 
and therefore  $\hat\lam^n_1\rightarrow_p \lambda_1^\gamma$ by Remark~\ref{SubConv}.

\item In this case, the only zero of~\eqref{OptimalityMargin} with $\gamma=0$ is found at $\bar\lambda_1$, 
and the derivative of the optimality conditions is nonzero at this estimate, by the computations already given. 
The result follows by the implicit function theorem and subsequence argument, just as in the previous case. 

\item Rewriting the derivative~\eqref{OptimalityMargin} 
\[
\frac{1}{2}\sum_{k=1}^{k_1} \frac{\lambda -v_k-\eta_{k}^2 + w_k}{(\lambda + w_{k})^2} + \gamma\;,
\]
we observe that for any positive lambda, the probability that the  derivative is positive tends to one.  Therefore 
the minimizer $\lambda_1^\gamma$ converges to $0$ in probability, regardless of the value of $\gamma$.

\end{enumerate}
\end{proof}

\bibliographystyle{plain}
\bibliography{biblio}

\end{document}